\documentclass{article}

\usepackage{fullpage}

\usepackage{microtype}
\usepackage{graphicx}
\usepackage{subfigure}
\usepackage{booktabs} %
\usepackage{array}
\usepackage{tikz}
\usepackage{forest}
\usepackage{wrapfig}
\usepackage{mdframed}

\usepackage{algorithm}
\usepackage[noend]{algpseudocode}
\algdef{SE}[DOWHILE]{Do}{doWhile}{\algorithmicdo}[1]{\algorithmicwhile\ #1}%

\usepackage{xcolor}         %

\definecolor{darkred}{RGB}{165,42,42}
\definecolor{darkgreen}{RGB}{0,130, 0}%
\definecolor{darkblue}{RGB}{0,0,150}

\usepackage[colorlinks=true, citecolor = darkgreen, linkcolor = darkred, urlcolor = darkblue]{hyperref}

\usepackage{amsmath}
\usepackage{amssymb}
\usepackage{mathtools}
\usepackage{amsthm}

\usepackage[utf8]{inputenc} %
\usepackage[T1]{fontenc}    %
\usepackage{url}            %
\usepackage{booktabs}       %
\usepackage{amsfonts, dsfont}       %
\usepackage{nicefrac}       %
\usepackage{microtype}      %
\usepackage{enumitem}
\usepackage{soul}

\usepackage[
backend=biber,
style=alphabetic,
maxalphanames=4,
useprefix=true,
sorting=anyt,
maxbibnames=99,
backref=true
]{biblatex}
\usepackage{amsmath}
\usepackage{amssymb}
\usepackage{mathtools}
\usepackage{amsthm, amsfonts}
\usepackage{mathrsfs}

\theoremstyle{plain}

\newtheorem{theorem}{Theorem}

\newtheorem{lemma}[theorem]{Lemma}

\newtheorem{definition}[theorem]{Definition}

\newtheoremstyle{lowspace}%
  {\topsep}%
  {-.1\baselineskip}%
  {\itshape}%
  {0pt}%
  {\bfseries}%
  {.}%
  { }%
  {\thmname{#1}\thmnumber{ #2}\textnormal{\thmnote{ (#3)}}}
\theoremstyle{lowspace}
\newtheorem{lemmalowspace}[theorem]{Lemma}

\newtheorem{theoremlowspace}[theorem]{Theorem}

\providecommand{\argmax}{\mathop\mathrm{arg\, max}} 

\newcommand{\ttheta}{\widetilde\theta}

\DeclareMathOperator*{\argmin}{arg\,min}

\usepackage[textsize=tiny]{todonotes}

\newcommand{\eff}{\mathbf{R}}
\newcommand{\saf}{\mathbf{S}}

\newcommand{\indi}{\mathds{1}}

\newcommand{\hist}{\mathfrak{H}}

\newcommand{\confset}{\mathcal{C}}

\newcommand{\lil}{\mathrm{LIL}}

\newcommand{\tO}{{\widetilde{O}}}
\newcommand{\tPhi}{\widetilde{\Phi}}
\newcommand{\lpt}{\mathsf{LP}}
\newcommand{\socp}{\mathsf{SOCP}}
\newcommand{\asafe}{a_{\mathsf{safe}}}
\newcommand{\gsafe}{\Gamma(\asafe)}

\newcommand{\con}{\mathsf{Con}}
\newcommand{\consistency}{\con_t(\delta)}
\newcommand{\acnorm}{\|a_t\|_{V_t^{-1}}}

\newcommand{\acnorms}{\|a_s\|_{V_s^{-1}}}

\newcommand{\optimism}{\mathsf{L}}
\newcommand{\Goptimism}{\mathsf{G}}

\newcommand{\colts}{\textsc{colts}}
\newcommand{\rcolts}{\textsc{r-colts}}
\newcommand{\scolts}{\textsc{s-colts}}
\newcommand{\ecolts}{\textsc{e-colts}}

\newcommand{\ts}{\textsc{ts}}
\newcommand{\onem}{\mathbf{1}_m}

\renewcommand{\succeq}{\succcurlyeq}

\newcommand{\bbst}{\bar{b}_{s\to t}}
\newcommand{\bbtt}{\bar{b}_{\tau(t) \to t}}

\newcommand{\ssst}{\sigma_{s \to t}}

\begin{document}

\title{Constrained Linear Thompson Sampling}
\author{Aditya Gangrade \\ Boston University \\\texttt{gangrade@bu.edu} \and  Venkatesh Saligrama\\Boston University\\ \texttt{srv@bu.edu}}
\date{\vspace{-\baselineskip}}

\maketitle
\begin{abstract}
We study safe linear bandits (SLBs), where an agent selects actions from a convex set to maximize an unknown linear objective subject to unknown linear constraints in each round. Existing methods for SLBs provide strong regret guarantees, but require solving expensive optimization problems (e.g., second-order cones, NP hard programs). To address this, we propose Constrained Linear Thompson Sampling (\colts), a sampling-based framework that selects actions by solving perturbed linear programs, which significantly reduces computational costs while matching the regret and risk of prior methods. We develop two main variants: 
\scolts, which ensures zero risk and ${\tO(\sqrt{d^3 T})}$ regret given a safe action, and \rcolts, which achieves ${\tO(\sqrt{d^3 T})}$ regret and risk with no instance information. In simulations, these methods match or outperform state of the art SLB approaches while substantially improving scalability. On the technical front, we introduce a novel coupled noise design that ensures frequent `local optimism' about the true optimum, and a scaling-based analysis to handle the per-round variability of constraints.

\end{abstract}

\section{Introduction}

Stochastic bandit problems are a fundamental model for optimising unknown objectives through repeated trials. While single-objective bandit theory is well-developed, real-world learners must also deal with \emph{unknown constraints} at every round of interaction. For instance, in \emph{dose-finding}
\cite{aziz2021multi}, \emph{micro-grid control}
\cite{feng2022stability}, and \emph{fair recommendation}
\cite{chohlas2024learning}, a learner must choose actions that maximise reward
while never crossing unknown toxicity, voltage, or exposure limits (see \S\ref{app:domains}).

The \emph{safe linear bandit} (SLB) problem models these scenarios in a linear programming (LP) setting: a learner selects actions $\{a_t\}$ from a convex domain $\mathcal{A}$ to optimize an unknown objective vector $\theta_* \in {\mathbb{R}^d}$ subject to unknown constraints of the form $\Phi_* a \le \alpha,$ where $\Phi_* \in {\mathbb{R}^{m \times d}}$. After each action, the learner observes noisy feedback of the objective ${\theta_*^\top a} + \mathrm{noise}$ and the constraints ${\Phi_* a} + \mathrm{noise}$, thus acquiring information to guide future actions. Performance in SLBs is measured via the
\begin{equation} \label{eqn:reg_and_risk_definitions}%
\textit{regret, } \eff_T := \sum_{t \le T} \bigl(\theta_*^\top(a_* - a_t)\bigr)_+,
\quad\textrm{and}\quad
\textit{risk, }  \saf_T :=\sum_{t \le T}  \Bigl(\max_i \bigl(\Phi_* a_t - \alpha\bigr)^i\Bigr)_+,
\end{equation}
where $a_*$ is the optimal action under the true (but unknown) constraints, and $(\cdot)_+ := \max(\cdot, 0)$. %
There are two main notions of safety in SLBs:\vspace{-.3\baselineskip}
\begin{itemize}[wide, nosep,topsep = 0pt, labelindent = 10pt]
\item \emph{Hard constraint enforcement}, which requires that with high probability, $\saf_T=0$ for all $T$. This is only achievable if the learner has prior access to a \emph{known safe action}~$\asafe$.
\item \emph{Soft constraint enforcement}, which requires $\saf_T = o(T)$ with high probability (whp). This is a weaker requirement, but does not need prior information.\vspace{-.3\baselineskip}
\end{itemize}%
A series of confidence‐based algorithms \cite[e.g.][]{gangrade2024safe,pacchiano2024contextual, amani2019linear, moradipari2021safe} has offered strong theoretical performance for SLBs, but in practice they often require nontrivial or NP‐hard constraints at each round (e.g., second‐order conic programs) and sometimes multiple such solves per step. This motivates a more computationally efficient design.

\textbf{Contributions.}
We introduce a \emph{sampling-based} approach, \emph{COnstrained Linear Thompson Sampling} (\colts), which adds carefully chosen noise to estimates of both the objective and constraint parameters, and selects actions according to this perturbed program. This allows us to maintain the same order of regret and risk bounds as prior methods, while substantially reducing the complexity of each round. However, just perturbing the program as above does not directly yield good actions, since the perturbed program may be infeasible, or its optimum may be unsafe. We therefore develop two augmentations of $\colts$, which address the SLB problem under distinct regimes:
\begin{itemize}[wide, nosep, topsep = -.1\baselineskip, itemsep = 0.1\baselineskip, labelindent = 10pt]
    \item \textbf{\scolts}  assumes a \emph{given safe action} $\asafe$. Actions are picked by first solving a perturbed LP (while ensuring that $\asafe$ is feasible), and then scaling its optimum towards $\asafe$ to ensure safety. This yields \emph{zero} risk, and regret $\mathcal{R}(\asafe) \cdot {\tO(\sqrt{d^3 T})}$ (see \S\ref{sec:defi}, or Table \ref{table:result_comparison} for definition of $\mathcal{R}(a)$).
    \item \textbf{\rcolts} requires only \emph{feasibility} of the true problem, and operates by sampling $O(\log T)$ perturbed programs, and setting $a_t$ to be the optimiser of the one with largest value. This resampling directly yields optimism, leading to instance-independent ${\tO(\sqrt{d^3 T})}$ regret and risk bounds. We additionally argue that under Slater's condition, and with extra exploration, a similar regret and risk guarantee follows without resampling, and so solving only one optimisation per round.%
\end{itemize}
Table~\ref{table:result_comparison} summarizes our results in comparison to prior work. Each variant attains regret and risk bounds matching those of prior methods, whilst selecting actions by only optimising over linear constraints (in addition to those due to $\mathcal{A}$). This yields the first efficient method for soft enforcement, and significantly speeds up hard constraint enforcement. Contextual extensions are discussed in \S\ref{sec:contextual}.%

\begin{table}[tb]
\caption{\footnotesize\textsc{Comparison of SLB Methods.} `Known $\asafe$' means that the method requires an action known a priori to be safe. $\Delta(a) := \theta_*^\top(a_* - a)$ is the reward gap of an action $a$, and $\Gamma(a) := \min_i (\alpha - \Phi_*a)^i_+$ is its safety margin. Slater's condition is that $\max_{a \in \mathcal{A}} \Gamma(a) > 0$, while feasibility assumes that $\mathcal{A} \cap \{\Phi_* a \le \alpha\} \neq \emptyset$. $\mathcal{R}(a) := 1 + (\nicefrac{\Delta(a)}{\Gamma(a)})$ if $\Gamma(a) > 0$, and $\infty$ otherwise. $\lpt$ is the computation needed to optimize a linear objective with $m$ linear constraints over $\mathcal{A}$ to constant approximation. $\socp$ is the same with $m$ second-order conic constraints. \textsc{opt-pess} refers to most frequentist hard enforcement methods discussed in \S\ref{sec:related_work}, which have similar costs and bounds; \textsc{safe-LTS} is due to \cite{moradipari2021safe}; \textsc{doss} and the lower bound are due to \cite{gangrade2024safe}.}\vspace{.5\baselineskip}
\label{table:result_comparison}
\centering
\setlength{\tabcolsep}{3pt}
{\begin{tabular}{llccc}
\toprule
Algorithm  & Assumptions                  & Regret & Risk & Compute at $t$  \\
\midrule
 \textsc{opt-pess} & Known $\asafe$ & $\mathcal{R}(\asafe) \cdot \tO(\sqrt{d^2 T})$ & $0$ & NP-hard\\
  Relaxed \textsc{opt-pess} & Known $\asafe$ & $\mathcal{R}(\asafe) \cdot \tO(\sqrt{d^3 T})$ & 0 & $d\cdot \socp\cdot\log(t)$
\\
  \textsc{safe-LTS} & Known $\asafe$ & $\mathcal{R}(\asafe) \cdot \tO(\sqrt{d^3 T})$ & $0$ & $\socp\cdot\log(t)$\\
\textbf{$\scolts$}  & Known $\asafe$ & $ \mathcal{R}(\asafe) \cdot \tO(\sqrt{d^3 T})$  & $0$ & $\lpt\cdot\log(t)$                 \\
 \midrule
\textsc{doss}  & Feasibility & $\tO(\sqrt{d^2 T})$  & $\tO(\sqrt{d^2 T})$ & NP-hard  \\
\textbf{$\rcolts$}  & Feasibility          & $\tO(\sqrt{d^3 T})$ & $\tO(\sqrt{d^3 T})$ & $\lpt \cdot \log^2(t)$ \\
\midrule
\textsc{lower-bound} & Feasibility & \multicolumn{3}{c}{ $\max(\eff_T, \saf_T)  = \Omega(\sqrt{T})$, \emph{no matter the instance}; } \\
\bottomrule
\end{tabular}} %
\end{table}

\textbf{Technical Innovations.} The random perturbations in our sampling-based approach cause two challenges that break existing analyses of linear $\ts$: (i) the feasible region fluctuates at each round; and (ii) the true optimum $a_*$ can become infeasible under perturbed constraints, complicating direct analysis. We address these via two key innovations:

\begin{enumerate}[label=\Alph*), wide, nosep, topsep=-.4\baselineskip, labelindent=10pt, itemsep=0.1\baselineskip]
    \item \emph{Coupled Noise Design.} Independent perturbations of objectives and constraints are difficult to analyze and yield undesirable exponential factors ($e^{\Omega(m)}$). We instead \emph{couple} the perturbations by adding a single random vector $\psi$ to the objective estimate and $-\psi$ to each row of the constraint estimate. This coupling ensures a high \emph{local optimism rate}: with constant probability, the perturbed program is feasible at the true optimum $a_*$, achieving regret bounds scaling only with $\log(m)$. Empirical studies (\S\ref{sec:simulations},\ref{appx:simulations}) confirm the advantages of coupled noise.
    
    \item \emph{Scaling and Resampling.} The fluctuating constraints disable both existing analysis frameworks for linear $\ts$: the `unsaturation' approach of \cite{agrawal2013thompson} and the `optimism' approach of \cite{abeille_lazaric}. To analyze $\scolts$, we adapt the unsaturation framework with a new scaling-based trick allowing comparisons across distinct feasible regions. For $\rcolts$, we instead use resampling to directly generate optimistic and feasible actions, bypassing these analytic barriers entirely.\vspace{-.5\baselineskip}
\end{enumerate}

\subsection{Related Work}\label{sec:related_work}

\textbf{Safe Bandits.} Safe bandits have been studied under two main notions of constraint enforcement: \emph{soft}~\cite{chen2022strategies, gangrade2024safe} and \emph{hard}~\cite{amani2019linear, moradipari2021safe, pacchiano2021stochastic, pacchiano2024contextual, hutchinson2023impact, hutchinson2024directional}. Soft enforcement achieves regret and risk bounds of ${\tO(\sqrt{d^2 T})}$, with improved instance-specific guarantees for polytopal domains. Hard enforcement achieves zero risk, and regret bounds of ${\tO(\mathcal{R}(\asafe)\sqrt{d^2T})}$ but given a safe action $\asafe$. Efficient variants of these methods instead achieve weaker regret bounds of ${\tO( \mathcal{R}(\asafe) \sqrt{d^3 T})}$.  In contrast to safe bandits, \emph{bandits with knapsacks}~\cite{badanidiyuru2013bandits, agrawal2016linear} control aggregate constraints, which is unsuitable for roundwise safety enforcement (see \S\ref{appx:related_work}).\vspace{0.3\baselineskip}\\
\textbf{Computational Complexity.} Existing efficient hard-enforcement methods rely on frequentists confidence sets for constraints, which induce $m$ expensive second-order conic (SOC) constraints during action selection~\cite{pacchiano2021stochastic, pacchiano2024contextual, amani2019linear, moradipari2021safe}. Most variants require solving $2d$ such problems per round, and further suffer from poor numerical conditioning. Our approach, $\scolts$, instead uses perturbations combined with scaling and resampling techniques, requiring only linear constraints per round while maintaining near-optimal guarantees. This scaling approach is related to \textsc{roful}~\cite{hutchinson2024directional} although this prior method utilises a the NP-hard method \textsc{doss} as a subroutine.\vspace{0.2\baselineskip}\\
Notably, no computationally efficient methods have previously been proposed for soft enforcement-the main point of comparison, \textsc{doss} needs $(2d)^{m+1}$ linear optimisations each round \cite{gangrade2024safe}. $rcolts$ resolves this gap by sampling $O(\log(t))$ perturbed programs each round. Under mild conditions (Slater's condition), one can further reduce to a single LP per round. See \S\ref{appx:related_work} for more details.\vspace{0.2\baselineskip}\\%
\textbf{Thompson Sampling ($\ts$).} Frequentist bounds for linear $\ts$ were first established by Agrawal \& Goyal~\cite{agrawal2013thompson} through an `unsaturation' approach, while Abeille \& Lazaric~\cite{abeille_lazaric} developed a related `global optimism' approach. Neither approach extends to SLBs due the per-round fluctuation of the perturbed constraints, and the ensuing variability of the `feasible regions' for each round (see \S\ref{appx:related_work} for more details). We overcome these challenges through our coupled noise design, ensuring frequent optimism, and a novel scaling trick to compare solutions across distinct feasible regions.\vspace{0.2\baselineskip}\\
The only existing sampling-based treatment of unknown constraints is due to Chen et al.~\cite{chen2022strategies} for multi-armed settings, who use posterior quantiles to enforce constraints. Although their method does not scale to continuous action sets, our resampling approach can be interpreted as an efficient, scalable analogue for simultaneously enforcing constraints and optimizing reward indices.%

\section{Problem Definition and Background}\label{sec:defi}%

\emph{Notation.~} For a vector $v$, $\|v\|$ denotes its $\ell_2$-norm. For a PSD matrix $M, \|v\|_M := \|M^{1/2} v\|$.  $\mathbb{S}^d$ is the unit sphere in $\mathbb{R}^d$. For a matrix $M$, $M^i$ is the $i$th row of $M$. $\onem$ is the all ones vector in $\mathbb{R}^m$. Also see \S\ref{app:glossary} for an extensive glossary of notation used in the paper.%

\textbf{Setup.} An instance of a SLB problem is defined by an objective $\theta_* \in \mathbb{R}^d,$ a constraint matrix $\Phi_* \in \mathbb{R}^{m \times d}$, constraint levels $\alpha \in \mathbb{R}^m$, a compact \emph{convex} domain $\mathcal{A} \subset \mathbb{R}^d,$ and $\delta \in (0,1)$. $\mathcal{A},\alpha,\delta$ are known to the learner, but $\theta_*$ and $\Phi_*$ are not. The program of interest is \( \max \theta_*^\top a \textrm{ s.t. } \Phi_* a \le \alpha, a \in \mathcal{A}, \) assumed to be feasible. $a_*$ denotes a(ny) maximiser of this program. The \emph{reward gap} of $a \in \mathcal{A}$ is $\Delta(a) := \theta_*^\top(a_* - a)$, and its \emph{safety margin} is $\Gamma(a) = \min_u(\alpha - \Phi_* a)^i_+$. For infeasible $a$, $\Gamma(a) = 0$, and $\Delta$ may be negative. We set $\mathcal{R}(a) = 1 + \nicefrac{\Delta(a)}{\Gamma(a)}$ if $\Gamma(a)>0$, and $\infty$ otherwise.%

\textbf{Play.} We index rounds by $t$. At each $t$, the learner picks $a_t \in \mathcal{A}$, and receives the feedback $R_t = \theta_*^\top a_t + w_t^R,$ and $S_t = \Phi_* a_t + w_t^S,$ where $w_t^R \in \mathbb{R}$ and $w_t^S\in \mathbb{R}^m$ are noise processes. $C_t$ denotes algorithmic randomness at round $t$. The historical filtration is $\hist_{t-1} := \sigma( \{(a_s, R_s, S_s, C_s)\}_{s < t}),$ and $\mathfrak{G}_t := \sigma( \hist_{t-1} \cup \{(a_t, C_t)\})$. The action $a_t$ must be adapted to $\sigma(\hist_{t-1} \cup \sigma(\{C_t\}))$.%

\textbf{The Soft Enforcement SLB problem} demands algorithms that ensure, with high probability, that both the metrics $\eff_T$ and $\saf_T$ (see (\ref{eqn:reg_and_risk_definitions}) grow sublinearly with $T.$\vspace{0.2\baselineskip}\\
\textbf{The Hard Enforcement SLB problem} demands algorithms that ensure, with high probability, that $\saf_T = 0$ and $\eff_T = o(T)$. This is enabled by a safe starting point $\asafe$ such that $\gsafe > 0.$%

\textbf{Standard Assumptions.} We assume the following standard conditions \cite[e.g.][]{abbasi2011improved} on the instance $(\theta_*, \Phi_*, \mathcal{A})$, and noise. All subsequent results only hold under these assumptions. \begin{itemize}[left = 0pt, wide, nosep, topsep = -.3\baselineskip, labelindent = 10pt]
    \item \emph{Boundedness}: $\|\theta_*\| \le 1,$ for each row $i$, $\|\Phi_*^i\| \le 1$, and $\mathcal{A} \subset \{ a : \|a\| \le 1\}$.
    \item \emph{SubGaussian noise}: $w_t := (w_t^R, (w_t^S)^\top)^\top$ is centred and $1$-subGaussian given $\mathfrak{G}_t$, i.e., \(\mathbb{E}[w_t|\mathfrak{G}_t] = 0,\) and \( \forall \lambda \in \mathbb{R}^{m+1}, \mathbb{E}[ \exp( \lambda^\top w_t)|\mathfrak{G}_t] \le \exp(\|\lambda^2\|/2).  \)
\end{itemize}
To simplify the form of our bounds, we also assume that $m/\delta = O(\mathrm{poly}(d))$ when stating theorems.%

\textbf{Background on Linear Regression.} The ($1$-)Regularised Least Squares estimates for $\theta_*,\Phi_*$ given the history $\hist_{t-1}$ are \[ %
\hat\theta_t = \argmin_{\hat\theta} \sum_{s < t} (\hat\theta^\top a_s - R_s)^2 + \|\hat\theta\|^2, \textit{ and } \hat\Phi_t = \argmin_{\hat\Phi} \sum_{s < t} \|\hat\Phi a_s - S_s\|^2 + \sum_i \|\hat\Phi^i\|^2. \] The standard \emph{confidence sets} \cite{abbasi2011improved} for $(\theta_*,\Phi_*)$ are  \[ %
\confset_t^\theta(\delta) = \{\ttheta : \|\ttheta - \hat\theta_t\|_{V_t} \le \omega_t(\delta)\}, \textrm{ and } \confset_t^\Phi(\delta) = \{\tPhi: \forall \textrm{ rows } i, \|\tPhi^i - \hat\Phi_t^i\|_{V_t} \le \omega_t(\delta)\}, \] where \( V_t := I + \sum_{s < t} a_sa_s^\top, \textit{ and } \omega_t(\delta) := 1 + \sqrt{\nicefrac12 \log((m+1)/\delta) + \nicefrac14 \log(\det V_t)}.\)
A key standard result states that these confidence sets are \emph{consistent} \cite{abbasi2011improved}.
\begin{lemma}\label{lemma:online_linear_regression}
    Let the \emph{consistency event at time $t$} be $\con_t(\delta) := \{ \theta_* \in \confset_t^\theta(\delta), \Phi_* \in \confset_t^\Phi\},$ and let $\con(\delta) := \bigcap_{t \ge 1} \con_t(\delta).$ Under the standard assumptions, for all $\delta \in (0,1),$ $\mathbb{P}(\con(\delta)) \ge 1-\delta$.%
\end{lemma}

\section{The Constrained Linear Thompson Sampling Approach}\label{sec:colts_approach}

We begin by describing the $\colts$ framework. In the frequentist viewpoint, $\ts$ is a randomised method for bandits that, at each $t$, perturbs an estimate of the unknown objective, in a manner sensitive to the historical information $\hist_{t-1}$, and then picks actions by optimising this perturbed objective. 

Naturally, then, we will perturb the estimates $\hat\theta_t, \hat\Phi_t$, for which we use a law $\mu$ on $\mathbb{R}^{1\times d} \times \mathbb{R}^{m \times d}$. For $(\eta, H) \sim \mu,$ independent of $\hist_{t-1},$ we define the perturbed parameters \begin{equation}\label{eqn:perturbed_param_definition} %
\ttheta(\eta, t)^\top := \hat\theta_t^\top + \omega_t(\delta) \eta V_t^{-1/2} \textrm{ and } \tPhi(H,t) := \hat\Phi_t + \omega_t H V_t^{-1/2}. \end{equation} Notice that these perturbations are aligned with $\hist_{t-1}$ only via the scaling by ${\omega_t(\delta)V_t^{-1/2}.}$ The underlying thesis of the $\colts$ approach is that for well-chosen $\mu$, the action \begin{equation} \label{eqn:generic_a_t_definition} %
a(\eta, H, t) = \argmax \{ \ttheta(\eta, t)^\top a : \tPhi(H,t) a \le \alpha, a \in \mathcal{A}\},\end{equation} if it exists, is a good choice to play, in that it is either underexplored, or nearly safe and optimal. Here we abuse notation, and treat $\argmax$ as a point function that (measurably) picks any one optimal solution. Two major issues arise with this view. Firstly, the set $\smash{\mathcal{A} \cap \{\tPhi(H,t) a \le \alpha\}}$ may be empty for certain $H,$ meaning $a(\eta,H,t)$ need not exist. Secondly, in hard enforcement, $a(\eta, H,t)$ need not actually be safe, and so cannot directly be used. Thus, the main questions are 1) what $\mu$ we should use, 2) how we should augment the $\colts$ principle to design effective algorithms, and 3) how we can analyse these algorithms to prove effectiveness. These questions occupy the rest of this paper.

Before proceeding, however, we observe that if $\eta$ or $H$ are very large, then they will `wash out' the `signal' in ${\hat\theta_t}$ and ${\hat\Phi_t}$, meaning that their size must be contained. We state this as a generic condition.
\begin{definition} \label{def:bconc}
    Let $B: (0,1] \to \mathbb{R}_{\ge 0}$ be a nondecreasing map. A law $\mu$ on $\mathbb{R}^{1\times d} \times \mathbb{R}^{m\times d}$ is said to satisfy $B$-concentration if \( \forall \xi \in (0,1], \mu\left(\left\{ \max( \|\eta\|, \max_{i \in [1:m]} \|H^i\|) \ge B(\xi) \right\}\right) \le \xi. \)%
\end{definition}
As an example, if marginally each $\eta, H^i$ were normal, then $B(\xi) = \sqrt{d \log((m+1)/\xi)}$. Henceforth, we will assume that $\mu$ satisfies $B$-concentration for some map $B,$ and define quantities in terms of this $B$. This condition has the following useful consequence (\S\ref{appx:analysis}).
\begin{lemma}\label{lemma:basic_noise_concentration_and_cauchy-schwarz}
    For $B: (0,1] \to \mathbb{R}_{\ge 0}$, and $t \in \mathbb{N},$ define $B_t = 1 + \max(1, B(\delta_t)),$ and $M_t(a) = {B_t \omega_t(\delta)\acnorm}$, where $\delta_t = \nicefrac{\delta}{t(t+1)}$. For all $t$, let $(\eta_t, H_t)$ be drawn from $\mu$ independently of $\hist_{t-1}$ at time $t$. If $\mu$ satisfies $B$-concentration, then with probability at least $1-2\delta,$ \[ %
    \forall t, a , \max \left( | (\theta_* - \ttheta(\eta_t,t))^\top a|, \max_i | (\tPhi(H_t,t)^i - \Phi_*^i) a |\right) \le M_t(a). \] Further, $\sum_{t \le T} M_t(a_t) \le B_T \omega_T \cdot O(\sqrt{d T}) \le B_T \tO(\sqrt{d^2 T}).$
\end{lemma}
\section{Hard Constraint Enforcement via Scaling-COLTS}\label{sec:scolts}

\begin{wrapfigure}[13]{r}{.47\linewidth}
\vspace{-1.7\baselineskip}
\begin{minipage}{\linewidth}
    \begin{algorithm}[H]
   \caption{Scaling-$\colts$ ($\scolts(\mu,\delta)$)}
   \label{alg:scolts}
\begin{algorithmic}[1]
   \State \textbf{Input}: $\asafe, \Gamma_0 \in [\gsafe/2, \gsafe]$.
   \For{$t = 1, 2, \dots$}
   \State Draw $(\eta_t, H_t) \sim \mu$
   \If{$M_t(\asafe) > \Gamma_0/3$ OR $a(\eta_t, H_t,t)$ does not exist} 
        \State $a_t\gets \asafe$.
    \Else
    \State $b_t \gets a(\eta_t,H_t,t)$
    \State Compute $a_t$ as in (\ref{eqn:scolts_at_defi}).
    \EndIf
    \State Play $a_t$, observe $R_t,S_t,$ update $\hist_t$.
   \EndFor
\end{algorithmic} 
\end{algorithm}
\end{minipage}
\end{wrapfigure}
We turn to the problem of hard constraint enforcement of minimising $\eff_T$ while ensuring that w.h.p., $\saf_T = 0$, using a safe action $\asafe$ such that $\Gamma(\asafe) > 0$. We will extend $\colts$ with a `scaling heuristic,' that was first proposed in the context of SLBs by Hutchinson et al.~\cite{hutchinson2024directional}, who used it to design a (inefficent) method \textsc{roful}.%

To begin, our method, $\scolts$, draws noise $(\eta_t, H_t)\sim \mu$, and computes the preliminary action $b_t := a(\eta_t,H_t,t)$, assuming for now that this exists. As argued in \S\ref{sec:colts_approach}, this action $b_t$ either has low-regret, or is informative. Of course, this $b_t$ need not be safe---we only know via Lemma~\ref{lemma:basic_noise_concentration_and_cauchy-schwarz} that $\Phi_* b_t \le \alpha + M_t(b_t)\onem$---and so cannot be used for hard enforcement. However, the action $\asafe$ is safe, with a large slack of at least $\gsafe$ in each constraint. Via linearity, and the convexity of $\mathcal{A},$ this means we can \emph{scale back} $b_t$ towards $\asafe$ to find a safe action, i.e., play $a_t$ of the form $(1-\rho_t)\asafe + \rho_t b_t$ for some $\rho_t \in [0,1]$. If $\rho_t$ is not too small, this maintains fidelity with respect to the informative direction $b_t$, while retaining safety. %

Note that if we knew that margin $\gsafe$ of $\asafe$, then since \[ %
\Phi_* (\rho b_t + (1-\rho)\asafe) \le \alpha + (\rho M_t(b_t) - (1-\rho)\gsafe) \onem, \] we could ensure that $1- \rho_t \le M_t(b_t)/\gsafe,$ which vanishes for small $M_t(b_t)$. Of course, we do not per se know the value of $\gsafe$. However, this can be estimated by repeatedly playing $\asafe$, and maintaining anytime bounds on its risk via a law of iterated logarithms. This results in an estimated  value $\Gamma_0$ such that ${\gsafe}/{2} \le \Gamma_0 \le \gsafe$ using $\tO(\gsafe^{-2})$ rounds. We leave a detailed account of this estimation to \S\ref{appx:scolts_sampling_asafe}, and henceforth just assume that we know such a value $\Gamma_0$.%

Define $\ttheta_t = \ttheta(\eta_t, t)$ and $\tPhi_t = \tPhi(H_t, t)$. A critical observation for $\scolts$ is that if $M_t(\asafe) \le \Gamma_0/3,$ then the action $b_t$ exists with high probability. Indeed, in this case, \[ %
\tPhi_t \asafe \le \Phi_*\asafe + M_t(\asafe) \onem \le \alpha - (\gsafe - \nicefrac{\Gamma_0}{3})\onem \le \alpha - \nicefrac{2\gsafe}{3}\onem.\] Thus, the constraints induced by $\tPhi_t$ are feasible (and so $b_t$ exists). To play a safe action, we set \begin{align}  %
    a_t = \mathfrak{a}_t(\rho_t), \textit{ where } \mathfrak{a}_t(&\rho) := (1- \rho) \asafe + \rho b_t, \textit{ and } \label{eqn:scolts_at_defi} \\
    \rho_t := \max\{ \rho \in [0,1]: \hat\Phi_t &\mathfrak{a}_t(\rho) + B_t^{-1}M_t( \mathfrak{a}_t(\rho)) \onem \le \alpha \},\notag  
\end{align}
where $B_t^{-1}M_t(a) = \omega_t \|a\|_{V_t^{-1}}$ is used since $\max_i \|\hat\Phi_t^i - \Phi_*^i\| \le \omega_t$ whp. Now, $M_t(\asafe) \le\Gamma_0/3$ ensures that $1-\rho_t \le 6M_t(b_t)/\gsafe,$ giving roughly the same fidelity as if we knew $\gsafe$. %

This leaves the law $\mu$ undetermined. In \S\ref{sec:unsaturation_and_lookback}, we first describe an \emph{unsaturation} condition on $\mu$ that induces low regret, and then in  \S\ref{sec:coupled_noise_design}, we provide a general construction of unsaturated laws using a local analysis at $a_*$. This operationalises the $\scolts$ design, with regret bounds described in \S\ref{sec:scolts_bounds}.

\subsection{Analysis of \texorpdfstring{$\scolts$}{Scaling COLTS}: Unsaturation and Looking-Back}\label{sec:unsaturation_and_lookback}

The scaling above direclty ensures the safety of $a_t$. We now present the main ideas behind a ${\tO(\sqrt{T})}$ regret bound for $\scolts$, leaving detailed proofs to \S\ref{appx:scolts_analysis}. Specifically, we describe how a \emph{look back} method operationalised with an \emph{unsaturation event} yields low-regret, using a \emph{scaling} strategy to handle shifting constraints. Finally, we finally contrast our analysis with prior studies.%

\textbf{Unsaturation.} Following \cite{agrawal2013thompson}, we say that an action $a$ is \emph{unsaturated} at time $t$ if $\Delta(a) \le M_t(a)$. Now, if $b_t$ is unsaturated, then (as we asserted in \S\ref{sec:colts_approach}) it is either informative (large $M_t(b_t)$) or low regret (small $M_t(b_t),$ and so small $\Delta(b_t)$). In fact, if for all $t$, $b_t$ was unsaturated and $b_t = a_t$, then we would already get a regret bound using Lemma~\ref{lemma:basic_noise_concentration_and_cauchy-schwarz}: $\sum_t \Delta(a_t) = \sum \Delta(b_t) \le \sum_t M_t(b_t) = \sum M_t(a_t) = \tO(\sqrt{T}).$ But, $b_t$ need not always be unsaturated (and usually is $\neq a_t$), and we must ensure that enough unsaturated $b_t$ occur, motivating the following definition.%
\begin{definition}\label{def:unsaturation} 
    Let $\mu$ be a $B$-concentrated law. Define the \emph{unsaturation event}  at time $t$ as \[ \setlength{\abovedisplayskip}{.4\baselineskip} \setlength{\belowdisplayskip}{.3\baselineskip} \mathsf{U}_t(\delta) := \{ (\eta, H) : a(\eta, H,t) \textrm{ exists, and } \Delta( a(\eta, H,t)) \le M_t( a(\eta, H ,t)).\] For $\chi \in (0,1],$ we say that $\mu$-satisfies $\chi$-unsaturation if for all $t$ such that $\delta/(t(t+1)) \le \chi/2,$ \[ \setlength{\abovedisplayskip}{.4\baselineskip} \setlength{\belowdisplayskip}{-\baselineskip} \mathbb{P}[ \mathsf{U}_t(\delta) |\hist_{t-1}] \indi_{\consistency} =  \mathbb{E}[ \mu(\mathsf{U}_t(\delta)) |\hist_{t-1}] \indi_{\consistency} \ge (\chi/2) \indi_{\consistency}. \]
\end{definition}
In words, at all $t$, given the past, $b_t$ is unsaturated with chance at least $\chi/2$.%

\textbf{Handing the scaling.} Our action $a_t$ equals $\rho_t b_t + (1-\rho_t)\asafe \neq b_t$. But, by linearity, \[ %
\Delta(a_t) = \rho_t \Delta(b_t) + (1-\rho_t) \Delta(\asafe).\] Thus, it suffices for us to control each of these terms. To bound these, we use the following result 
\begin{lemma}\label{lemma:rho_bounds_main_text}
    At any $t$ such that $M_t(\asafe) \le \Gamma_0/3,$ it holds that \[ (1-\rho_t) \le 6 \frac{M_t(a_t)}{\gsafe} \textit{ and } \rho_t M_t(b_t) \le 2 M_t(a_t). \]
\end{lemma}
This Lemma is directly implied by a slightly more detailed result, Lemma~\ref{lemma:rho_bounds} in \S\ref{appx:look_back_proof}. By the first conclusion above and Lemma~\ref{lemma:basic_noise_concentration_and_cauchy-schwarz}, we have \[ \sum (1-\rho_t)\Delta(\asafe) = \frac{\Delta(\asafe)}{\gsafe} \sum M_t(a_t) = \tO(\mathcal{R}(\asafe) \sqrt{d^3 T}).\] Thus, we only need to analyse $\sum \rho_t \Delta(b_t)$. Notice that if $b_t$ was always saturated, then via the control on $\rho_t M_t(b_t)$ above, we would also get regret control through the relation $\sum \rho_t \Delta(b_t) \le \sum \rho_t M_t(b_t) = O(\sum M_t(a_t))$.%

\textbf{Look Back Method.} In reality, $b_t$ is not always unsaturated. To handle this, we `look back' at the \emph{last time} $s < t$ that a $b_s$ was unsaturated. Concretely, define \[ %
\tau(t) := \inf \{ s < t: M_s(\asafe) \le \Gamma_0/3, \Delta(b_s) \le M_s(b_s)\}, \quad \inf\emptyset := 0. \] For the sake of notational brevity, we will often write $\tau$ instead of $\tau(t)$ below. Our basic approach is to bound $\Delta(b_t)$ in terms of $\Delta(b_{\tau})$, the idea being that since $b_\tau$ is unsaturated, $ \Delta(b_\tau) \le M_\tau(b_\tau),$ and as long as unsaturation is frequent, this $M_\tau(b_\tau)$ should not be much larger than $M_t(b_t),$ yielding regret control along the lines of the discussion above. %

\newcommand{\bbttLITE}{\bar{b}_{\tau \to t}}

A major issue in this program, however, is that ${b_{\tau}}$ may not be feasible for the perturbed constraints $\smash{\tPhi_t}$. This means that inequalities like $\smash{\ttheta_t^\top( b_{\tau} - b_t)} \le 0$ may be false, preventing us from using the basic fact that $b_t$ optimises $\smash{\ttheta_t}$ over the perturbed feasible set $\smash{\mathcal{A} \cap \{\tPhi_t a \le \alpha\}}$. To address this, we introduce a \emph{second, analysis-only, scaling step}. Concretely, for $s < t,$ let \[%
\bbst := \ssst b_{s} + (1-\ssst) \asafe, \textit{ where } \ssst = \frac{\Gamma_0/3}{\Gamma_0/3 + M_t(b_s) + M_t(b_t)}.\] Since $M_t(\asafe) \le \Gamma_0/3,$ using Lemma~\ref{lemma:basic_noise_concentration_and_cauchy-schwarz} twice yields \( \tPhi_s b_s \le \alpha \implies \tPhi_t \bbst \le \alpha.\) Now, setting $s = \tau(t)$ and using linearity, we write 
\[ %
\rho_t \Delta(b_t) = \rho_t \Delta(\bbttLITE) + \rho_t \ttheta_t^\top(\bbttLITE - b_t) + \rho_t (\theta_* - \ttheta_t)^\top (\bbttLITE - b_t).\]

\newcommand{\ssttLITE}{\sigma_{\tau \to t}}
Since $\bbttLITE$ is feasible, and $b_t$ optimal, for ${(\ttheta_t, \tPhi_t)},$ the second term is $\le 0$. Via Lemma~\ref{lemma:basic_noise_concentration_and_cauchy-schwarz}, Lemma~\ref{lemma:rho_bounds_main_text} and the triangle inequality, the final term is bounded as \[\rho_t (\theta_* - \ttheta_t)^\top (\bbttLITE - b_t) \le \rho_t M_t(\bbttLITE) + \rho_t M_t(b_t) \le \ssttLITE M_t(b_\tau) + (1-\ssttLITE) M_t(\asafe) +2 M_t(a_t), \] 
\begin{wrapfigure}[16]{r}{.35\linewidth}
\centering
\vspace{-1.2\baselineskip}
\includegraphics[width = \linewidth]{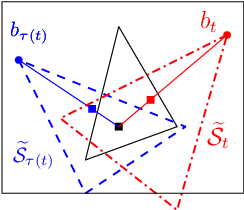}\vspace{-\baselineskip}
\caption{\footnotesize A schematic of the look-back analysis. $\smash{\tilde{\mathcal{S}}}$ are the perturbed feasible regions. $a_t$ (red box) is a mixture of $b_t$ and $\asafe$ (black box). $\smash{b_{\tau(t)}} \not\in \smash{\tilde{\mathcal{S}}_t}$, and we instead mix it with $\asafe$ to produce $\smash{\bbtt} \in {\tilde{\mathcal{S}}_t}$ (blue box).}
\end{wrapfigure}
Finally, we use the saturation of $b_\tau$ to observe that \begin{align*} %
\Delta(\bbttLITE) &= \ssttLITE \Delta(b_{\tau}) +  (1-\ssttLITE) \Delta(\asafe) \\ & \le \ssttLITE M_{\tau}(b_{\tau}) + (1-\ssttLITE) \Delta(\asafe). \end{align*} Thus, this calculation bounds $\Delta(a_t)$ by $M_t(a_t)$ and terms depending on $\ssttLITE,M_t(b_{\tau}),$ and $M_\tau(b_\tau)$. To handle these terms, we argue that, roughly speaking, $\ssttLITE$ behaves essentially like $\rho_\tau$. This, coupled with $M_\tau(\asafe) \le \Gamma_0/3$ lets us essentially reuse Lemma~\ref{lemma:rho_bounds_main_text} to show that $\ssttLITE M_\tau(b_\tau) \le 2M_\tau(a_\tau)$ and $(1-\ssttLITE) \gsafe \le 6 M_\tau(a_\tau).$ This yields the following key bound, proved in \S\ref{appx:look_back_proof}.%
\begin{lemma}\label{lemma:look_back_bound}
    Define $M_0(a) = \Delta(a_0) = B_0 = \omega_0(\delta) = 1$, and for $s < t,$ set $J_{s \to t} = 1 + \nicefrac{B_t \omega_t}{B_s \omega_s}.$ If $\mu$ is $B$-concentrated, then with probability at least $1-3\delta,$ for all $t$ with $M_t(\asafe) \le \Gamma_0/3,$ 
    \[ %
    \Delta(a_t) \le 6 \mathcal{R}(\asafe) \left( M_t(a_t) + J_{\tau(t) \to t} M_{\tau(t)}(a_{\tau(t)}) \right).\]
\end{lemma}

\textbf{Controlling Accumulation Error.} By Lemma~\ref{lemma:basic_noise_concentration_and_cauchy-schwarz}, $\sum M_t(a_t) = \tO(\sqrt{T})$. Thus, to bound regret, we are left with $\sum {J_{\tau(t) \to t} M_{\tau(t)}(a_{\tau(t)})}$. But, due to $\chi$-unsaturation, the time between two unsaturation events is typically only ${O(\chi^{-1})},$ which enables a martingale analysis in \S\ref{appx:better_width_control} to give\begin{lemmalowspace}\label{lemma:width_control} %
    If $\mu$ satisfies $\chi$-unsaturation, then with probability at least $1-\delta,$ for all $T \ge \sqrt{2/\chi},$ \[ \sum_{\substack{t \le T, M_t(\asafe)\le\Gamma_0/3}} J_{\tau(t) \to t} M_{\tau(t)}(a_{\tau(t)}) \le 10 \omega_T B_T \cdot \chi^{-1} \cdot ( \log(1/\delta) + \sum_{t\le T} \|a_t\|_{V_t^{-1}} ). \]
\end{lemmalowspace}

Finally, by Lemma~\ref{lemma:basic_noise_concentration_and_cauchy-schwarz}, $\sum \acnorm = \tO(\sqrt{T})$ as well, concluding the regret analysis.%

\textbf{Novelty Relative to Prior Work.} In unconstrained linear $\ts,$ Agrawal and Goyal \cite{agrawal2013thompson} compare $a_t$ to a \emph{minimum-norm unsaturated action} $\ddot{a}_t$ (see Lemma~4 in their arxiv version \cite{agrawal2012arxiv}), using $\smash{\ttheta_t^\top (\ddot{a}_t - a_t)} \le 0$ and $\Delta(\ddot{a}_t) \le M_t(\ddot{a}_t)$ to bound their $\Delta(a_t)$. Therein, $\chi$-unsaturation is essentially used to show $M_t(\ddot{a}_t) \le \chi^{-1}M_t(a_t).$ Our analysis sees three differences.\begin{itemize}[left = 0pt, wide, nosep, topsep = -.4\baselineskip, labelindent = 10pt, itemsep = 0.1\baselineskip]
    \item \emph{Scaled Actions}: $b_t \neq a_t,$ which is managed by expanding the linear blend and analysing $\rho_t$.
    \item \emph{Look-back Analysis}: instead of a reference action at time $t$, we look back to the last unsaturated $b_{\tau(t)}$. This, in our opinion, yields a more intuitive analysis, but is not an essential step.
    \item \emph{Fluctuating Constraints}: The reference action $b_{\tau(t)}$ may be infeasible for $\smash{\tPhi_t}$, breaking the comparison of values between it and $b_t$. This would also affect the equivalent of $\smash{\ddot{a}_t}$, and causes the prior analysis to break down. We handle this via a second analysis-only scaling step.\vspace{0.2\baselineskip}
\end{itemize}

The technical gap from \cite{hutchinson2024directional} is similar: $a_*$ may be infeasible for $\smash{\tPhi_t}$, which breaks their analysis.%

\subsection{The Coupled Noise Design}\label{sec:coupled_noise_design}

\S\ref{sec:unsaturation_and_lookback} shows that $\chi$-unsaturation yields control on the regret. To operationalise this, we need to design well-concentrated laws with good unsaturation. In single-objective $\ts$, unsaturation is enabled via anticoncentration of the $\eta$s, and a good balance is attained by, e.g., $\mathrm{Unif}(\sqrt{3d} \mathbb{S}^d)$ or $\mathscr{N}(0,I_d)$.%

A natural guess with unknown constraints is to sample both $\eta$ and each row of $H$ from such a law. However, the unsaturation rate under such a design is difficult to control well. The main issue arises from maintaining feasibility with respect to  all $m$ constraints under perturbation, since each such perturbation gets an independent shot at shaving away some unsaturated actions, suggesting that $\chi$ decays as $e^{-\Omega(m)}$ and indeed, experimentally, increase in $m$ may lead to at least a polynomial decay in the unsaturated rate with such independent noise (see \S\ref{appx:simulations_decoupled}). We sidestep this issue by \emph{coupling} the perturbations of the reward and constraints, as encapsulated below.
\begin{lemma}\label{lemma:coupled_noise_design}
    Let $\bar{B} \in \{ (0,1] \to \mathbb{R}_{\ge 0}\}$ be a map, and $p \in (0,1]$. Let $\nu$ be a law on $\mathbb{R}^{d \times 1}$ such that \[ %
    \forall u \in \mathbb{R}^d, \nu( \{ \zeta: \zeta^\top u \ge \|u\|) \ge p, \textrm{ and } \forall \xi \in (0,1], \nu(\{\zeta:  \|\zeta\| > \bar{B}(\xi)\} ) \le \xi. \] Let $\mu$ be the law of $\zeta \mapsto (\zeta^\top, -\mathbf{1}_m \zeta^\top)$ for $\zeta \sim \nu$. Then $\mu$ is $p$-unsaturated and $\bar{B}$-concentrated.%
\end{lemma}
Our proof of this lemma, executed in \S\ref{appx:coupled_noise_proof}, is based upon analysing the \emph{local optimism event} at $a_*$: \begin{equation} %
\label{eqn:local_optimism} \optimism_t(\delta) := \{ (\eta, H) : \ttheta(\eta,t)^\top a_* \ge \theta_*^\top a_*, \tPhi(H,t) a_* \le \alpha\}. \end{equation} Notice that $\optimism_t$ demands that the perturbation is such that $a_*$ remains feasible with respect to $\smash{\tPhi_t},$ and its value at $\smash{\ttheta}_t$ increases beyond $\theta_*^\top a_*$, in other words, the perturbed program is \emph{optimistic `at' the true optimum $a_*$}. Our proof first directly analyses $a_*$ under the perturbations to show that $\mathbb{P}[\optimism_t(\delta)|\hist_{t-1}] \indi_{\consistency} \ge p \indi_{\consistency}$, i.e., frequent local optimism. This enables an argument due to \cite{agrawal2013thompson}: since $a_*$ is unsaturated ($\Delta(a_*) = 0$), and, given the consistency event $\consistency$, the perturbed reward of any saturated action is dominated by that of $a_*$, it follows that $\optimism_t(\delta) \subset \mathsf{U}_t(\delta),$ yielding lower bounds on $\mu(\mathsf{U}_t(\delta))$ .%

We note that the conditions of Lemma~\ref{lemma:coupled_noise_design} are the same as those used for unconstrained linear $\ts$ in prior work \cite{agrawal2013thompson, abeille_lazaric}, and so this generic result extends this unconstrained guarantee to the constrained setting.   %
In our bounds, we will set $\mu$ to be the law induced by the coupled design with $\nu = \mathrm{Unif}(\sqrt{3d}\mathbb{S}^d)$, which is $0.28$-unsaturated, and $B$-concentrated for $B(\xi) = {\sqrt{3d}}$ (\S\ref{appx:simple_reference_laws}).

\subsection{Regret Bounds for \texorpdfstring{$\scolts$}{\textsc{s-colts}}}\label{sec:scolts_bounds}

With the pieces in place, we state and discuss our main result, which is formally proved in \S\ref{appx:scolts_analysis}.%

\begin{theoremlowspace}\label{thm:scolts_regret}
    Let $\mu$ be the law induced by $\mathrm{Unif}(\sqrt{3d} \mathbb{S}^d)$ under the coupled noise design. Then $\scolts(\mu,\delta/3)$ ensures that with probability at least $1-\delta,$ for all $T$, it holds that \[  \saf_T = 0 \quad \textrm{ and } \quad \eff_T =  \mathcal{R}(\asafe) \cdot \tO(  \sqrt{d^3 T + d^2 T \log(m/\delta)})  + \tO( {d^2}\Delta(\asafe){\gsafe}^{-2}). \]
\end{theoremlowspace}

\textbf{Comparison of Regret Bounds to Prior Results.} As noted in \S\ref{sec:related_work}, prior inefficient hard enforcement SLB methods attain regret ${\tO(\mathcal{R}(\asafe) \sqrt{d^2 T})},$ while efficient methods attain regret ${\tO(\mathcal{R}(\asafe) \sqrt{d^3 T})}$. We recover the latter bounds above. The loss of $\sqrt{d}$ relative to inefficient methods is expected since it appears in all known efficient linear bandit methods (without or without unknown constraints). The ${\Omega(\sqrt{T})}$ dependence is necessary (even with instance-specific information) \cite{gangrade2024safe} as is the additive $\Delta(\asafe)/\gsafe^2$ term \cite{pacchiano2021stochastic}. Thus, $\scolts$ recovers previously known guarantees via a sampling approach instead of an \textsc{oful}-type approach based on optimising over confidence sets. %

\textbf{Computational Aspects.} An advantage of $\scolts$ is that it only optimises over linear constraints (beyond those of $\mathcal{A}$), instead of SOC constraints of the form $\{ \forall i \in [1:m], {\hat\Phi_t^i} a + \omega_t(\delta) {\|a\|_{V_t^{-1}}} \le \alpha^i\}$ imposed by prior methods. While convex, these $m$ SOC constraints can have a palpable practical slowdown on the time needed for optimisation, especially as $m$ grows. As examples, over $\mathcal{A} = {[0,1/\sqrt{d}]^d},$ with the modest $d= m = 9$  we see a $>5\times$ speedup, and with $d = 2, m = 100$, a $18\times$ speedup, in \S\ref{sec:simulations}. In particular, when $\mathcal{A}$ is a polyhedron, $\scolts$ can be implemented with just linear programming.%

We explicitly note that $\scolts$ is efficient for convex $\mathcal{A}$. The dominating step is the computation of $b_t$, which can be carried out to an approximation of $1/t$ with no loss in Theorem~\ref{thm:scolts_regret}. With, say, interior point methods, this needs $O(\lpt \cdot \log(t))$ computation at round $t$, where $\lpt$ is the computation needed to optimise $\max\{\theta^\top a : \Phi a \le \alpha, a \in \mathcal{A}\}$ to constant error \cite{boyd2004convex}.%

\textbf{Practical Choice of Noise.} It has long been understood that while existing theoretical techniques for analysing linear \textsc{ts} need large noise (with $B(\xi) = {\Theta(\sqrt{d})}$), in practice much smaller noise (e.g., $\mathrm{Unif}(\mathbb{S}^d)$ with $B(\xi) = \Theta(1)$) typically retains a large enough rate of unsaturation, and significantly improve regret (although not in the worst-case \cite{hamidi2020worst}). Our practical recommendation is to indeed use such a small noise, which we find to be effective in simulations (\S\ref{sec:simulations}). We underscore that no matter the noise used, the risk guarantee for $\scolts$ is maintained. 

\section{Soft Constraint Enforcement with Resampling-COLTS}\label{sec:rcolts}

\begin{wrapfigure}[19]{r}{.47\linewidth}
\vspace{-2\baselineskip}
\begin{minipage}{\linewidth}
\begin{algorithm}[H]
   \caption{Resampling-$\colts$ ($\rcolts(\mu, r,\delta)$)}
   \label{alg:rcolts}
\begin{algorithmic}[1]
   \State \textbf{Input}: $\mu, \delta,$ `resampling order' $r \in \mathbb{N}$
   \State \textbf{Initialise}: $I_t \gets 1 + \lceil r \log\nicefrac{t(t+1)}{\delta}\rceil$%
   \For{$t = 1, 2, \dots$}
   \For{$i = 1,2,\dots, I_t$}
        \State Draw $(\eta_{i,t}, H_{i,t})\sim \mu$. 
        \If{$a(\eta_{i,t}, H_{i,t}, t)$ exists}
            \State $K(i,t) \gets \ttheta(\eta_{i,t},t)^\top a(\eta_{i,t}, H_{i,t}, t)$
        \Else
            \State $K(i,t)\gets -\infty$
        \EndIf
   \EndFor
   \If{$\max K(i,t) = -\infty$}
        \State $a_t \gets a_{t-1}.$
    \Else
        \State $i_{*,t} \gets \argmax_i K(i,t),$ 
        \State $a_t \gets a(\eta_{i_{*,t},t}, H_{i_{*,t},t}, t)$.
        \State $\ttheta_t \gets \ttheta(\eta_{i_{*,t}, t}, t)$. 
    \EndIf
    \State Play $a_t,$ observe $R_t,S_t,$ update $\hist_t$.
   \EndFor
\end{algorithmic} 
\end{algorithm}
\end{minipage}
\end{wrapfigure}
We now turn to the case where the learner does not a priori know an $\asafe$ with positive safety margin (thus disabling $\scolts$). In this case, it is impossible to ensure that $\saf_T = 0$, and we instead show instance-independent ${\tO(\sqrt{T})}$ bounds on $\saf_T.$ %

$\scolts$ uses forced exploration of $\asafe$ to ensure the feasibility of perturbed programs. However, the local optimism underlying our proof of Lemma~\ref{lemma:coupled_noise_design} gives a different way to achieve this. Indeed, the event $\optimism_t(\delta)$ of (\ref{eqn:local_optimism}) implies that $a_*$ is feasible, and so $a(\eta, H,t)$ exists. Thus, if $\mathbb{P}[\optimism_t(\delta)|\hist_{t-1}] \ge \pi,$ then we can just resample the noise $O(\log(t))$ times and end up with feasibility. In fact, even more is true: given $\optimism_t,$ we have ${\ttheta}^\top a_* \ge \theta_*^\top a_*$, and so the value of the perturbed program dominates optimum, yielding direct optimism. Thus, resampling $\pi^{-1} \Theta(\log(t))$ times w.h.p.~ensures not only feasibility, but also \emph{optimism} of the `best' perturbed optimum. The $\rcolts$ method is based on this observation.%

Concretely, we parametrise $\rcolts$ with $\mu,\delta$ and a resampling order $r$. At each time $t$, we sample $I_t = 1 + \lceil r \log \nicefrac{t(t+1)}{\delta}\rceil$ independent $(\eta,H)$ from $\mu$, optimise each perturbed program, and pick the optimiser of the one with largest value as $a_t$. If all are infeasible, we just set $a_t = a_{t-1}$ (picking $a_0$ arbitrarily). We let $\smash{\ttheta_t}$ denote the objective vector of this `winning' perturbed program: in the notation of Alg.~\ref{alg:rcolts}, $\smash{\ttheta_t = \ttheta(\eta_{i_{*,t}},t)}$. The main idea is captured in the following simple lemma.
\begin{lemma}\label{lemma:rcolts_main_lemma}
    Let $\pi \in (0,1],$ and suppose $\mu$ satisfies \(\indi_{\con_t(\delta)} \mathbb{E}[ \mu( \optimism_t(\delta))|\hist_{t-1}] \ge \pi \indi_{\con_t(\delta)}\) for every $t$. If $r\ge {\pi^{-1}}$, then with probability at least $1-2\delta,$ at all $t$, the actions $a_t$ and perturbed objective $\smash{\ttheta_t}$ selected by $\rcolts(\mu, r,\delta)$ are optimistic, i.e., they satisfy that $\theta_*^\top a_* \le \smash{\ttheta_t^\top} a_t$. %
\end{lemma}
The `local optimism condition' on $\mu$ above is reminiscent of the global optimism condition of Abeille \& Lazaric \cite{abeille_lazaric}, and indeed the same result holds under a global optimism assumption with unknown constraints. However, the analysis in this prior work does not extend to unknown constraints due to its reliance of convexity (\S\ref{sec:related_work}), and resampling bypasses this issue. See \S\ref{appx:conditions} for more details.%

Lemma~\ref{lemma:rcolts_main_lemma} enables the use of standard optimism based regret analyses \cite[e.g.][]{abbasi2011improved}. By operationalising the condition on $\mu$ via the coupled design in \S\ref{sec:coupled_noise_design}, we show
\begin{theoremlowspace}\label{thm:rcolts_bounds}
    If $\mu$ is the law induced by ${\mathrm{Unif}(\sqrt{3d}\mathbb{S}^d)}$ under the coupled design of Lemma~\ref{lemma:coupled_noise_design}, then with probability at least $1-\delta,$  $\rcolts(\mu, 4, \delta/2)$ ensures that for all T, \[ %
    \max(\saf_T, \eff_T) = \tO( \sqrt{d^3 T + d^2 T \log(m/\delta)}).\]%
\end{theoremlowspace}
\textbf{Instance-Independent Regret Bound.} The above result limits both regret and risk to $\tO(\sqrt{d^3 T})$, with no instance-specific terms, unlike $\mathcal{R}(\asafe)$ in $\scolts$. In particular, this bound holds even if $\max_a \Gamma(a) = 0,$ i.e., the problem is marginally feasible. This result is directly comparable to the ${\tO(\sqrt{d^2 T})}$ bound on both regret and risk under the \textsc{doss} method \cite{gangrade2024safe}, and loses a ${\sqrt{d}}$-factor relative to this, a loss that appears in all known efficient linear bandit methods.%

\textbf{Computational Costs.} $\rcolts$ with $\mu$ as above solves $\sim 4\log(t^2/\delta)$ optimisations of $\ttheta_t^\top a$ over $\{\smash{\tPhi_t} a \le \alpha\} \cap \mathcal{A}$. Again, Theorem~\ref{thm:rcolts_bounds} is resilient to approximate optimisation with error, say, $1/t,$ and so this takes ${O(\lpt \cdot \log^2 t)}$ computation per round, a factor of $\log(t)$ slower than $\scolts$, but still efficient in the practical regime of $\log(T/\delta) = O(\mathrm{poly}(d,m))$. The main point of comparison, however, is \textsc{doss}, which instead needs to solve $(2d)^{m+1}$ such programs, and so uses $(2d)^{m+1} \lpt \cdot \log(t)$ computation per round. $\rcolts$ is practically \emph{much faster} even for small domains with long horizons---for instance, even with $T = 1/\delta = 10^{10}, 4\log(t^2/\delta) \le (2d)^{m+1}$ for all $d \ge 4, m \ge 2$.%

\textbf{Relationship to Posterior Quantile Indices and Safe MABs.} The resampling approach executed in $\rcolts$ is closely related to the posterior-quantile approach of the \textsc{BayesUCB} method \cite{kaufmann2012bayesian}, which uses a quantile of the arm posteriors as a reward index instead of a frequentist upper confidence bound. Indeed, we can compute such a quantile in a randomised way by taking many samples from the posterior of each arm, and then picking the largest of the samples as the reward index. Most pertinently, this approach was proposed for safe multi-armed bandits \cite{chen2022strategies}, wherein this posterior quantile index is used to decide on the `plausible safety' of putative actions.  The same work further argued that the usual \emph{single-sample} \textsc{ts} cannot obtain sublinear regret in safe MABs. The $\rcolts$ approach can be viewed as an efficient extension of this principle to linear bandits with continuum actions, and differs by directly optimising the indices under each draw, and then picking the largest, instead of performing an untenable per-arm posterior quantile computation.%

\textbf{$\rcolts$ Without Resampling.} Given the lack of a safe action to play, one cannot direct establish the feasibility of the perturbed programs by contracting the confidence radius of a single action as in $\scolts$. However, if we introduce a small amount of `flat' exploration whenever $V_t$ is `small', then this ensures that any $a$ with $\Gamma(a) > 0$ will eventually be strictly feasible under perturbations. If such $a$ exists, we only need a single noise draw to attain feasibility, and can bootstrap the scaling analysis of $\scolts$ to show bounds. We term this method `exploratory-COLTS', or $\ecolts$, and specify and analyse it in \S\ref{appx:ecolts}. This results in the following soft-enforcement guarantee.%
\begin{theorem}\label{thm:ecolts_new_main}
    If $\mu$ is the law induced by $\mathrm{Unif}(\sqrt{3d} \mathbb{S}^d)$ under the coupled noise design, then the $\ecolts(\mu, \delta/3)$ method of Algorithm~\ref{alg:ecolts} ensures that with probability at least $1-\delta,$ for all $T$, \begin{align*}%
    \saf_T = \tO(\sqrt{d^3 T}) + \min_{a}\tO\Bigl(  \frac{d^3\|a\|^4}{\kappa^2 \Gamma(a)^4}\Bigr),\textit{ and } \eff_T =  \min_{a : \Gamma(a) > 0} \left\{ \mathcal{R}(a) \tO(\sqrt{d^3 T}) + \tO\Bigl(\frac{d^3 \|a\|^4}{\kappa^2 \Gamma(a)^4}\Bigr) \right\}, \end{align*} where $\kappa$ is a constant depending on the geometry of $\mathcal{A}$. %
\end{theorem}
Relative to $\rcolts,$ the above guarantees are instance-dependent, and are only nontrivial if $\max_a \Gamma(a) > 0,$ i.e., the Slater parameter of the optimisation problem induced by $\theta_*, \Phi_* , \mathcal{A}$ is nonzero. The advantage of $\ecolts$ lies in its reduced computation. Comparing to $\scolts$, the above loses the strong $\saf_T = 0$ safety, but improves regret by adapting to the best possible $\mathcal{R}(a)$. %

\section{An Informal Discussion of Contextual Safe Linear Bandits}\label{sec:contextual}

Rather than static bandit problems, most practical scenarios are contextual, wherein the learner observes some side information $x_t$ before choosing an action, and this side information affects the reward and constraint structure at time $t$. A common setting to model this \cite{pacchiano2024contextual, agrawal2013thompson} is to assume that there is a known feature map $\varphi : \mathcal{X} \times \mathcal{A} \to \mathbb{R}^d$ such that the reward and constraints at time $t$ are of the form \[ \theta_*^\top \varphi(x_t, a) \textrm{ and } \Phi_* \varphi(x_t, a) \le \alpha. \] Throughout, we assume the same feedback structure, i.e., noisy measurements of $\theta_*^\top \varphi(x_t, a_t) $ and $\Phi_* \varphi(x_t,a_t)$. Naturally, regret is compared to the optimal policy $\mathscr{A}_* : \mathcal{X} \to \mathcal{A}$, where \[ \mathscr{A}_*(x) = \argmax \theta_*^\top \varphi(x,a) : \Phi_*\varphi(x,a) \le \alpha, a \in \mathcal{A}.\] It should be noted that the Lemma~\ref{lemma:online_linear_regression} on consistency, and the elliptical potential lemma (Lemma~\ref{lemma:elliptical_potential_lemma}) continue to hold, with $V_t$ replaced by $I + \sum_{s \le t} \varphi(x_s, a_s)\varphi(x_s,a_s)^\top,$ and $a_t$ by $\varphi(x_t, a_t)$. Notationally, we extend $\Delta(a), \Gamma(a)$ to $\Delta(x, a) = \theta_*^\top (\varphi(x,\mathscr{A}_*(x)) - \varphi(x,a) )$ and $\Gamma(x,a) = \max_i(( \alpha - \Phi_* \varphi(x, a))^i)_+$.

A key observation is that our result on the frequency of the local optimism persists in this contextual setting. Under the hood, this essentially shows that at any $t$, and for any vector $\varphi,$ \[ \mathbb{P}\left\{ (\eta, H) : \ttheta^\top \varphi \ge \theta_*^\top \varphi, \tPhi \varphi \le \Phi_* \varphi \middle| \hist_{t-1} \right\} \indi_{\con_t(\delta)} \ge \pi \indi_{\con_t(\delta)},\] where $\pi \ge 0.28$ for the coupled noise driven by $\mathrm{Unif}(\sqrt{3d}\mathbb{S})^d$. Consequently, frequent local optimism follows in the contextual setting by using this result for $\varphi(x_t, \mathscr{A}_*(x_t))$ at time $t$. 

The above observation means that using the same coupled noise lets us extend the results of Theorem~\ref{thm:rcolts_bounds}  on the regret of $\rcolts$ to the contextual case with only cosmetic changes in the analysis. This holds no matter how the sequence $x_t$ is selected, as long as the noise remains conditionally centred and subGaussian given $a_t, x_t,$ the algorithmic randomness, and the history. Note, however, that the optimisation over $a$ may become harder due to the feature map $\varphi$, and efficiency requires further structural assumptions on $\varphi$.

\newcommand{\policysafe}{\mathscr{A}_{\mathrm{safe}}}
Focusing now on $\scolts$, let us first note that if we were given a safe action $\asafe$ that was safe no matter the context, i.e., such that \( \inf_x \Gamma(x, \asafe) \ge \Gamma_{\mathrm{safe}} > 0,\) and $\varphi$ were `nice' in terms of $a \in \mathcal{A}$,\footnote{We essentially need a way to efficiently select an action $a$ such that $\varphi(x_t, a) = \rho \varphi(x_t, b_t) + (1-\rho)\varphi(x_t,\asafe),$ so that safety can still be attained by mixing with $\asafe$.} then as long as we \emph{know} $\Gamma_{\mathrm{safe}}$ a priori, no real change is required, and the guarantees of Theorem~\ref{thm:scolts_regret} for $\scolts$ extend to the contextual setting,\footnote{upto replacing $\Delta(\asafe)$ by $1$} since we can again guarantee the frequent choice of unsaturated actions through our persistent local optimism property. We note that previous works on safe contextual bandits \cite{pacchiano2024contextual} assume exactly this existence of an `always very safe' action. Nevertheless, this structure is unrealistic: practically, safety should depend strongly on the context, and it is unlikely that a single action would always be safe, let alone have a large safety margin. 

A more natural assumption is that instead of a single safe action, we are given a safe policy $\policysafe :\mathcal{X}\to \mathcal{A}$. Here, again, if we know that $\inf_x \Gamma(x, \policysafe(x)) \ge \Gamma_{\mathrm{safe}} > 0,$ and we know the value of $\Gamma_{\mathrm{safe}}$, then the bound of Theorem~\ref{thm:scolts_regret} can be extended to show regret of $\tO( (\inf_x \Gamma(x,\policysafe(x)))^{-1} \sqrt{d^3 T})$. Without knowing this value, we need to be able to determine a good estimate of $\Gamma(x_t, \policysafe(x_t))$ in order to appropriately ensure feasiblity of perturbed programs, and to scale back the actions $b_t$. This can be a challenging task, especially if $x_t$ varies in an adversarial way, and structures enabling such estimation must be assumed.\footnote{For instance, if $x_t$ were drawn in some static randomised way, and $\Gamma$ were sufficiently simple, then we could learn $\Gamma(x,\policysafe(x))$ using regression techniques.} Finally, note that even given $\Gamma(x,\policysafe(x)),$ the aforementioned regret bound scales with $(\inf_x \Gamma(x, \policysafe(x)))^{-1}$, which is overly pessimistic. A different analysis is needed to clearly express how variation in this margin with $x$ affects the regret. %

This lacuna also affects the $\ecolts$ method of \S\ref{appx:ecolts}, but to a lesser extent. Sticking with `nice' feature maps, again, if there \emph{exists} an action that is always safe, i.e., if $\max_{a } \min_{x} \Gamma(x_t, a_t) > 0,$ then the guarantees of Theorem~\ref{thm:ecolts_new_main} extend with arbitrary context sequence. Without this guarantee, the main gap is the exploration policy being utilised, which must be adapted to attain a good coverage over $\{\varphi(x,a)\}$ even as $x_t$ varies. Given such a policy, however, the results of Theorem~\ref{thm:ecolts_new_main} again extend to the contextual case with arbitrary $x_t$. %

\section{Simulations}\label{sec:simulations}%

We give a brief summary of our simulations, leaving most details, and well as deeper investigation of our methods to \S\ref{appx:simulations}. In all cases, we utilise the coupled noise design, driven with the (uninflated) noise $\nu = \mathrm{Unif}(0.5 \cdot \mathbb{S}^d),$ in accordance with the discussion in \S\ref{sec:scolts}. The same noise is used for \textsc{safe-lts}. %

\begin{wraptable}[7]{r}{.44\linewidth}
\vspace{-1.5\baselineskip}
\centering
\caption{\footnotesize $\eff_T$ and $\saf_T$ at $T = 5 \cdot 10^4$ for $\rcolts$ with $1,2,3$ samples per round ($100$ trials).}\vspace{0.2\baselineskip}
\begin{tabular}{ccc} \label{table:rcolts_regret}
Samples & $\mathbf{R}_T$ & $\mathbf{S}_T$ \\ \hline\hline
$1$                 & $658 \pm 170$         & $2891 \pm 171$      \\
$2$                 & $397 \pm 116$         & $3126 \pm 137$      \\
$3$                 & $301 \pm 102$         & $3266 \pm 172$     
\end{tabular}
\end{wraptable}
\textbf{Resampling tradeoff in $\rcolts$.} For $d = 9$, we optimise ${\theta_* = \mathbf{1}_d/\sqrt{d}}$ over $\mathcal{A} = [0, 1/\sqrt{d}]^d,$ with a $9 \times 9$ constraint matrix (i.e., $m = 9$). In this case, the action $0$ is feasible, and so $\rcolts$ without any resampling is effective. Since $a = 0$ has a nontrivial safety margin, $\rcolts$, even without resampling, is effective for this problem. This is borne out in Table~\ref{table:rcolts_regret}, which shows regret and risk at the terminal time $T$. We see that resampling slightly worsens risk, but significantly improves regret (although with diminishing returns). Further, both regret and risk are far below the ${\sqrt{d^2 T}}$ scale expected from our bounds. We note that while a single iteration of $\rcolts$ takes $\sim 1\textrm{ms}$, since $(2d)^{m+1} > 10^{12},$ this would take \emph{years} for \textsc{doss}, and so we do not implement it. In any case, note that the computational advantage of $\rcolts$ is extremely strong.%

\textbf{Significant Computational Advantage and Regret Parity/Improvement of $\scolts$.} We compare $\scolts$ with the hard enforcement method \textsc{safe-lts} \cite{moradipari2021safe}, which has been shown to match the performance of alternate such methods, while being faster. Both methods are run on the $d = m = 9$ instance above, with ${\asafe = 0}$. As expected, both never play unsafe actions. Further (Fig.~\ref{fig:scolts_main}, left), $\scolts$ achieves an \emph{improvement} in regret relative to \textsc{safe-lts}, while reducing wall-clock time by a $5.1 \times$. To gain a deeper understanding of $\scolts$'s computational advantage, we investigate the same with growing $m \in \{1, 10, 20, \dots, 100\}$ constraints for a simple $d = 2$ setting (see \S\ref{appx:scolts_varying_m} for the setup). In this problem, the benefit is even starker (Fig.~\ref{fig:scolts_main}, right). For $m \ge 10,$ the regret of \textsc{safe-lts} is $2-4\times$ larger than that of $\scolts, $ i.e., the latter has much better regret ($m =1$ has wide confidence bands for the ratio, but mean $\sim 1.5$) Further, the computational costs of \textsc{safe-lts} relative to $\scolts$ grow roughly linearly, starting from $\approx 1.3\times$ for $m = 1$ to $>18\times$ at $m = 100$.%

\begin{figure}[H]
    \centering
        \includegraphics[width = 0.47\linewidth]{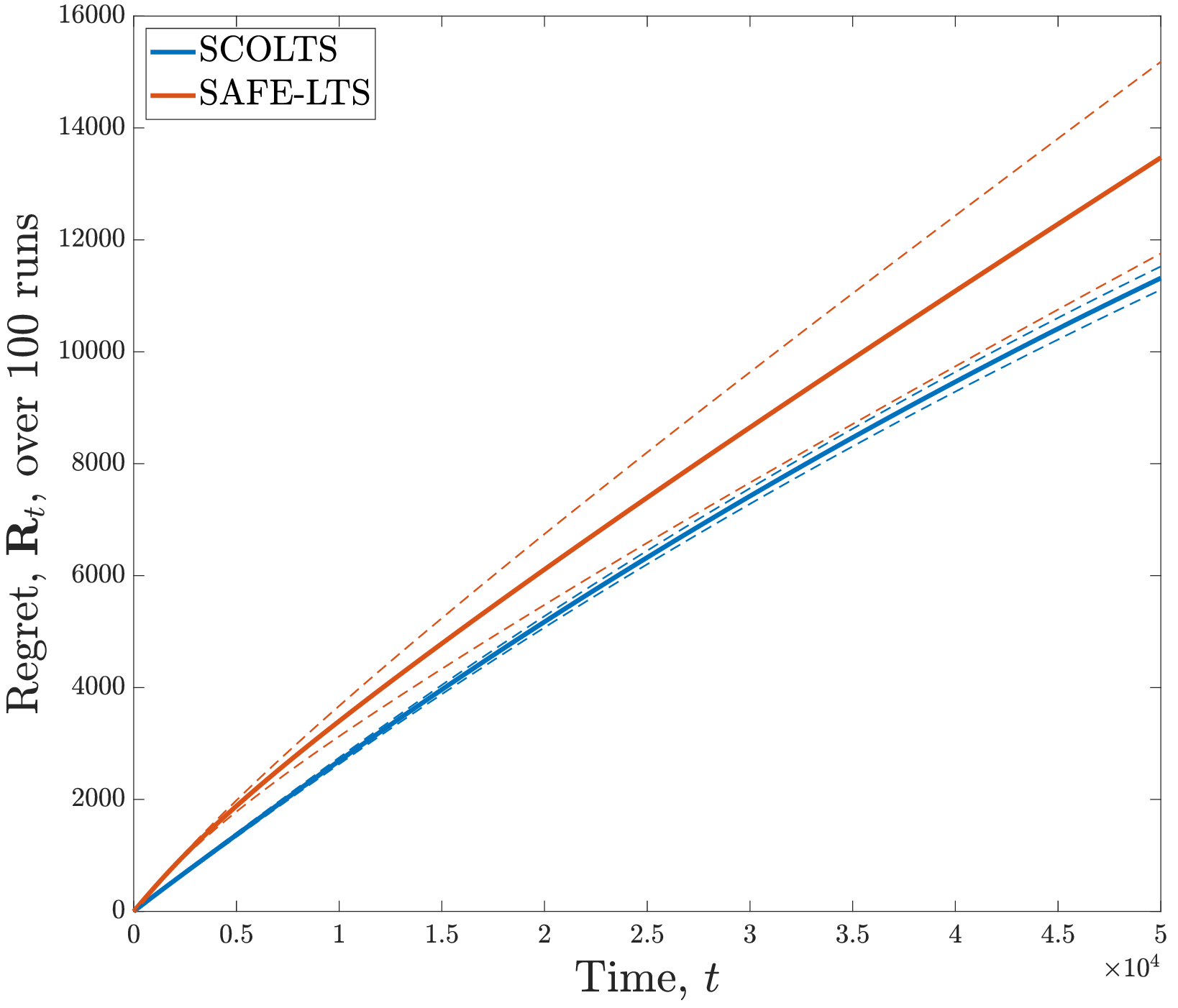}\label{fig:scolts_regret_main}~\hspace{0.05\linewidth}~\includegraphics[width = 0.47\linewidth]{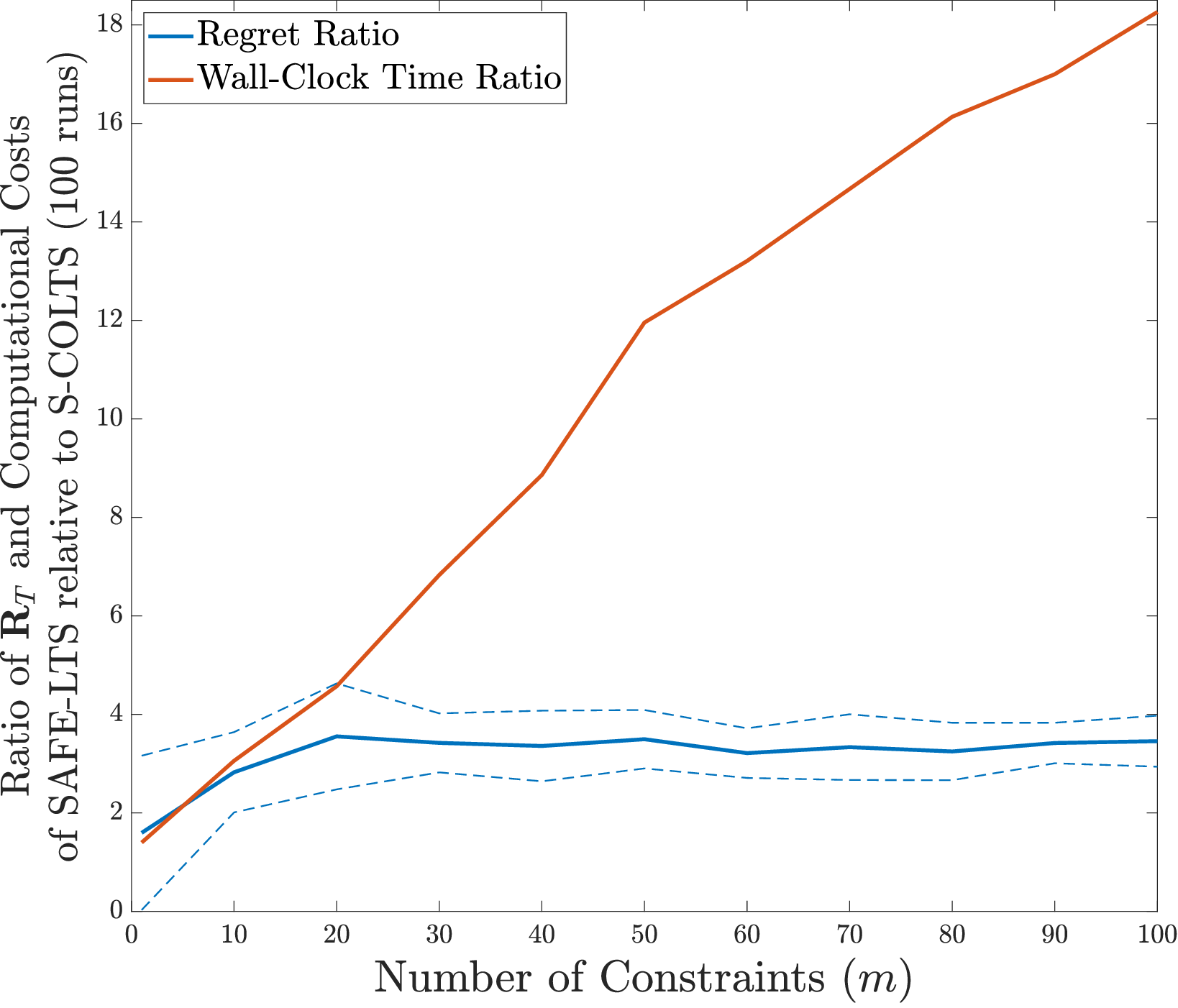}\label{fig:scolts_varying_m_main}\vspace{-.3\baselineskip}
        \caption{\footnotesize \textsc{Computational and Regret Comparisons of $\scolts$ and \textsc{safe-lts}}. \emph{Left}. Regret traces in the $d = 9$ instance (one-sigma error bars); $\scolts$ mildly improves regret, and is $5 \times$ faster. \emph{Right}. Relative performance as $m$ is varied in the $d = 2$ instance. The speedup of $\scolts$ grows linearly with $m$ from $1.3\times$ to $>18\times$. Further, for $m \ge 10,$ the regret of $\scolts$ is $2\textrm{-}3\times$ smaller than that of \textsc{safe-lts}}    \label{fig:scolts_main}
\end{figure}

\section{Discussion}

Our paper has described a novel and computationally efficient sampling-based approach to the SLB problem with the $\colts$ framework, which applies to both the hard and soft constraint enforcement settings. In particular, $\scolts$ uses a scaling approach to enable hard constraint enforcement while obtaining tangible computational gains over prior methods for this problem, and $\rcolts$ (and $\ecolts$) provide the first efficient soft constraint enforcement methods in continuous-action SLBs. The operation and analysis of these algorithms relies crucially on the coupled perturbation noise design to ensure frequent local optimism. 

Of course, a number of open problems remain. The chief amongst these is the resolution of the contextual bandit scenarios we discussed in \S\ref{sec:contextual}. We also mention two further intriguing lines of study specific to the $\colts$ framework. Firstly, we observe that the frequentist soft constraint enforcement method \textsc{doss} obtains much stronger regret bounds of $\tO(d^2 \log^2 T)$ when $\mathcal{A}$ is a polytope, and it is a natural open question to determine if $\rcolts$ can attain such strong regret guarantees. This direction is further supported by the strong regret observations in \S\ref{sec:simulations}. Secondly, as seen in \S\ref{appx:simulations_decoupled}, while we cannot analyse the decoupled noise design which samples perturbations independently, this appears to still admit sublinear regret and risk, proving which may require new insights about the global optimism properties of $\colts.$

\subsection*{Acknowledgements} The authors would like to thank Aldo Pacchiano for helpful discussions. This research was supported by the Army Research Office
Grant W911NF2110246, AFRL Grant FA8650-22-C1039, and the National Science Foundation grants CPS-2317079, CCF-2007350, and CCF-1955981.

\printbibliography

\appendix
\section{Glossary} \label{app:glossary} \raggedbottom
\begin{table}[H]
\centering
\setlength{\tabcolsep}{3pt}\vspace{-\baselineskip}
\resizebox*{!}{\dimexpr\textheight-2\baselineskip\relax}{
\begin{tabular}{clc}\hline
Symbol                           & Explanation                                                                                                               & Expression/Comments                                                                                                                                                                                                                                                                          \\ \hline\hline
$(\theta_*, \Phi_*)$                       & True objective/constraints                                                                                                     & $ \in \mathbb{R}^{d \times 1} \times \mathbb{R}^{m \times d}$                                                                                                                                                                                                                                                            \\
$\alpha$                         & Constraint level                                                                                                  &                                     $\in \mathbb{R}^{m \times 1}$                                                                                                                                                                                                                                                         \\
$\mathcal{A}$                    & Action domain                                                                                                             &                                                                                                                                                                                                                                                                                              \\
$a_*$                            & Optimal action for $(\theta_*, \Phi_*)$                                                                                   & $\argmax \{ \theta_*^\top a : a \in \mathcal{A}, \Phi_* a \le \alpha\}$                                                                         \\
$K(\theta, \Phi)$                & Value function                                                                                                            & \begin{tabular}[c]{@{}c@{}}$\sup \{ \theta^\top a : a \in \mathcal{A}, \Phi a \le \alpha \},$ \\ $-\infty$ if $\{\Phi a \le \alpha\} \cap \mathcal{A} = \varnothing$.\end{tabular}                                                                                                         \\
$\Delta(a)$                      & Reward gap                                                                                                                & $\theta_*^\top(a_* -a)$                                                                                                                                                                                                                                                                      \\
$\Gamma(a)$                      & Safety margin                                                                                                             & $ \min_i ((\alpha - \Phi_* a)^i)_+$                                                                                                                                                                                                                                                              \\
$\mathcal{R}(a)$                 & Gap-margin ratio                                                                                                       & $1 + (\Delta(a)/\Gamma(a))$                                                                                                                                                                                                                                                                  \\ \hline
\multicolumn{3}{c}{Estimation and Signal}                                                                                                                                                                                                                                                                                                                                                                                                                   \\ \hline
$\hist_{t-1}$                    & \begin{tabular}[c]{@{}c@{}}Historical filtration\end{tabular} & See \S\ref{sec:defi}                                                                                                                                                                                                                                                                                   \\
$\hat\theta_t, \hat\Phi_t$       & RLS-estimates of parameters                                                                                               & See \S\ref{sec:defi}                                                                                                                                                                                                                                                                               \\
$V_t$                            & Action second moment                                                                                                      & $I + \sum_{s < t} a_s a_s^\top$                                                                                                                                                                                                                                                              \\
$\omega_t(\delta)$               & Confidence radius                                                                                                         & See \S\ref{sec:defi}                                                                                                                                                                                                                                                                            \\
$\mathcal{C}_t^\theta,\confset_t^\Phi$   & Confidence sets for $\theta_*, \Phi_*$                                                                                               & \\
$\consistency$                 & Consistency event at time $t$                                                                                             & $\{ \theta_* \in \confset_t^\theta(\delta), \Phi_* \in \confset_t^\Phi(\delta)\}$                                                                                                                                                                                                            \\
$\con(\delta)$           & Overall consistency                                                                                                       & $\bigcap_{t \ge 1} \con_t(\delta)$                                                                                                                                                                                                                                                           \\ \hline
\multicolumn{3}{c}{$\colts$ in general}                                                                                                                                                                                                                                                                                                                                                                                                                 \\ \hline
$\mu$                            & Perturbation law                                                                                                          & Distribution on $\mathbb{R}^{1 \times d} \times \mathbb{R}^{m \times d}$                                                                                                                                                                                                                     \\
$(\eta, H)$                      & Perturbation noise                                                                                          & $\sim \mu$, independently of $\hist_{t-1}$                                                                                                                                                                                                                                             \\
$\ttheta(\eta, t)$               & Pertrubed objective                                                                                                       & $\hat\theta_t + \omega_t(\delta) \eta V_t^{-1/2}$                                                                                                                                                                                                                                            \\
$\tPhi(H,t)$                     & Perturbed constraint                                                                                                      & $\hat\Phi_t + \omega_t(\delta) H V_t^{-1/2}$.                                                                                                                                                                                                                                                \\
$B(\xi)$                         & Tail bound on $\|\eta\|, \max_i \|H^i\|$                                                                      &                                                                                                                                                                                                                                                                                              \\
$B_t$                            & Noise radius bound                                                                                                        & $ \max(1, B(\delta_t)),$ where $\delta_t = \nicefrac{\delta}{(t^2+t)}$.                                                                                                                                                                                                                                   \\
$M_t(a)$                         & Perturbation scale at $a$                                                                                                 & $ B_t \omega_t \smash{\|a\|_{V_t^{-1}}}$                                                  \\
$a(\eta, H,t)$                   & Perturbed optimum                                                                                  & See (\ref{eqn:generic_a_t_definition})                                                                                                                                                                                                                                                                                 \\
$\mathsf{U}_t(\delta)$           & Unsaturation event                                                                                                        & $\{ (\eta, H) : \Delta( a(\eta, H,t)) \le M_t(a(\eta, H,t) \}$                                                                                                                                                     \\
$\chi$                           & Unsaturation rate                                                                                                         &                                                                                                                                                                                                                                                                                              \\
$\optimism_t(\delta)$            & Local optimism event                                                                                                      & $\{ (\eta, H) :  \ttheta(\eta, t)^\top a_* \ge \theta_*^\top a_*, \tPhi(H, t) a_* \le \alpha \}$ \\
$\pi$                            & Local optimism rate                                                                                             & \\ \hline
\multicolumn{3}{c}{Coupled Noise Design}                                                                                                                                                                                                                                                                                                                                                                                                                    \\ \hline
$\nu$                            & Baseline perturbation law                                                                                                 & Supported on $\mathbb{R}^{d \times 1}$                                                                                                                                                                                                                                                       \\
$\zeta$                          & Generic draw from $\nu$                                                                                      & $\zeta \sim \nu$, independent of $\hist_{t-1}$                                                                                                                                                                                                                                               \\
$\bar{B}$                        & Tail bound for $\nu$                                                                                                      & $\nu ( \|\zeta\| > \bar{B}(\xi) ) \le \xi$                                                                                                                                                                                                                                                   \\
$p$                              & Anticoncentration parameter for $\nu$                                                                                     & $\inf_{u} \nu( \zeta^\top u > \|u\|) \ge p$                                                                                                                                                                                                                                                  \\
$(\zeta^\top, -\onem \zeta^\top)$ & Coupled noise induced by $\nu$                                                                                 & i.e., draw $\zeta,$ set $\eta = \zeta^\top$ and $H = -\onem \zeta^\top$.                                                                                                                                                                                                                      \\ \hline
\multicolumn{3}{c}{$\scolts$}                                                                                                                                                                                                                                                                                                                                                                                                                           \\ \hline
$\asafe$                         & A priori given safe action                                                                                                      & $\Gamma(\asafe) > 0$.                                                                                                                                                                                                                                                                        \\
$\Gamma_0$                       & \begin{tabular}[c]{@{}c@{}} Reference margin (see \S\ref{appx:scolts_sampling_asafe}) for estimation) \end{tabular}                                                     & $\Gamma_0 \ge \gsafe/2$ and $\Gamma_0 \le \gsafe$                                                                                                                                                                                                                                            \\
$(\eta_t, H_t)$                  & Perturbation noise at $t$                                                                                      &                                                                                                                                                                                                                                                                       \\
$\ttheta_t, \tPhi_t$             & Perturbed parameters at $t$                                                                                          & $\ttheta_t = \ttheta(\eta_t, t), \tPhi_t = \tPhi(H_t, t)$                                                                                                                                                                                                                                    \\
$b_t$                            & Preliminary action at time $t$ (if exists)                                                                                & $b_t = a(\eta_t, H_t, t)$                                                                                                                                                                                                                                                                    \\
$\mathfrak{a}(\rho)$             & $\rho$-mixture of $b_t$ and $\asafe$                                                                                      & $\mathfrak{a}(\rho) = \rho b_t + (1-\rho) \asafe$                                                                                                                                                                                                                                            \\
$\rho_t$                         & Largest $\rho$ with safe $\mathfrak{a}(\rho)$                                                           & See (\ref{eqn:scolts_at_defi}); $a_t = \mathfrak{a}(\rho_t)$. \\
$\tau(t)$                        & Look-back time                                                                                                            & Lemma~\ref{lemma:look_back_bound}   \\ \hline
\multicolumn{3}{c}{$\ecolts$}                                                                                                                                                                                                                                                                                                                                                                                                                           \\ \hline
$(\eta_t, H_t)$                  & Perturbation noise draws at time $t$                                                                                      & $(\eta_t, H_t) \sim t$                                                                                                                                                                                                                                                                       \\
$\kappa$                         & Goodness factor of exploratory policy                                                                                     & See \S \ref{appx:ecolts}                                                       \\
$u_t$                            & Number of exploration steps up to time $t$                                                                                & $u_t \approx B_t \omega_t\sqrt{dt}$ \\ \hline
\multicolumn{3}{c}{$\rcolts$}                                                                                                                                                                                                                                                                                                                                                                                                                           \\ \hline
$r$                              & Resampling parameter                                                                                                      &                                                                                                                                                                                                                                                                                              \\
$I_t$                            & Number of resamplings at time $t$                                                                                         & $I_t = \lceil r \log(1/\delta_t)\rceil + 1$.                                                                                                                                                                                                                                                 \\
$(\eta_{i,t}, H_{i,t})$          & $i$th draw of noise perturbation at time $t$                                                                              & $\sim \mu$ independently %
\\
$K(i,t)$                         & Value under perturbation                                                                                                  & $K(\ttheta(\eta_{i,t},t), \tPhi(H_{i,t}, t) )$                                                                                                                                                                                                                                               \\
$i_{*,t}$                        & Best index at time $t$                                                                                                    & $\argmax_i K(i,t)$                                                                                                                                                                                                                                                                           \\
$a_t$                            & Action picked                                                                                                             & $a_t = a( \eta_{i_{*,t}}, H_{i_{*,t}}, t)$                                                                                                                                                                                                                                                    \\
$\ttheta_t$                      & Objective for $i_{*,t}$                                                                                                   & $\ttheta_t = \ttheta( \eta_{i_{*,t}}, t)$.                                                                                                                                                                                                                                                   \\
                                 &                                                                                                                           &                                                                                                                                                                                                                                                                                             
\end{tabular}
}
\end{table} 
\clearpage

\section{Examples of Real-World Domains where the Safe Linear Bandit Problem Applies}
\label{app:domains}

\begin{table}[h]
\centering\footnotesize
\setlength{\tabcolsep}{6pt}
\caption{\textbf{Mapping real domains to the bandit linear programming.}  In all three cases the reward is linear in an unknown parameter vector $\theta_*,$ and the 
safety/fairness predicate is an \emph{unknown linear inequality}
$\Phi_* a \le \alpha$. Feedback noise in both rewards and constraints arises through environmental or individual fluctuations.}
\resizebox{\textwidth}{!}{\begin{tabular}{llll}%
\toprule
Domain ($\,$ref.$\,$) &
Action $a\!\in\!\mathcal A\!\subset\!\mathbb{R}^{d}$ &
Reward $\theta_*^\top a + \textrm{noise}$ &
Constraints \\ \midrule
\begin{tabular}[c]{@{}l@{}} \textbf{Dose-finding} \\ \cite{aziz2021multi} \end{tabular} 
 & \begin{tabular}[c]{@{}l@{}} One-hot vector for $d$ \\ discrete dose levels \end{tabular}
  & \begin{tabular}[c]{@{}l@{}} $\theta_*^i = $ patient-level efficacy \\ probability at dose $i$ \end{tabular} &
\begin{tabular}[c]{@{}l@{}} $\Phi_*^i$ = toxicity of dose $i;$ constraint\\ so that $P(\textrm{toxic}|\textrm{dose}) \le \alpha$ \end{tabular}
\\[5pt]
\begin{tabular}[c]{@{}l@{}} \textbf{Voltage-constrained} \\ \textbf{micro-grid}\\
\cite{feng2022stability} \end{tabular} &
\begin{tabular}[c]{@{}l@{}}Active/reactive power \\ set-point $[P,Q]^\top$ \\for each bus\end{tabular}
 & \begin{tabular}[c]{@{}l@{}}$\theta_*^i$ = locational marginal\\ price vector \end{tabular}
& \begin{tabular}[c]{@{}l@{}} $\Phi_*$ = linearised network power-flow\\ imposing nodal-voltage constraints\\ under variable demand \end{tabular}
 \\[5pt]
\begin{tabular}[c]{@{}l@{}} \textbf{Fair Reccommendation} \\ \textbf{in A/B testing} \\ \cite{chohlas2024learning} \end{tabular}
& \begin{tabular}[c]{@{}l@{}} Distribution over $d$\\ items or policies \end{tabular}
& \begin{tabular}[c]{@{}l@{}}$\theta_*^i$ = revenue of item $i$\end{tabular}
& \begin{tabular}[c]{@{}l@{}} $\Phi_*^i$ = encoding group attributes\\ and costs; constraints demand fair \\ exposure for each group \end{tabular}
\end{tabular}}
\end{table}

\section{Further Related Work}\label{appx:related_work}

\emph{Distinction of Safe Bandits From BwK.} BwK settings are concerned with aggregate cost metrics of the form $\mathbf{A}_T := \max_i (\sum \alpha - \Phi_* a_t)^i,$ without the $(\cdot)_+$ nonlinearity in $\saf_T$ \cite[e.g.][]{agrawal2014bandits, badanidiyuru2013bandits}. This simple change has a drastic effect, in that BwK algorithms can `bank' violation by playing very safe actions for some rounds, and then `spend' it to gain high reward, without any net penalty in $\mathbf{A}_T$. This is appropriate for modeling aggregate cost constraints (monetary/energy/et c.), but is evidently inappropriate to model safety constraints where feasibility violation in any round cannot be offset by acting safely in another round. Notice that such behaviour is precluded by the ramp nonlinearity in $\eff_T, \saf_T$: playing too-conservatively does not decrease $\saf_T$, while any violation of constraints is accumulated, and similarly, playing suboptimally causes $\eff_T$ to rise, but playing an over-aggressive action with negative $\Delta(a)$ does not reduce $\eff_T$. 

\emph{Pure Exploration in Safe Bandits.} While our paper focuses on controlling regret and risk, naturally the safe bandit problem can be studied in the pure-exploration sense. These are studied in both the `soft enforcement' sense, in which case methods can explore both within and outside the feasible region and return actions that are $\varepsilon$-safe and $\varepsilon$-optimal \cite[e.g.,][]{camilleri2022active, katz2019top}, and the `hard enforcement', wherein exploratory actions must be restricted to the feasible region \cite[e.g.,][]{sui2015safe, bottero2022information}.

\emph{More Details on Computational Costs of Prior Methods.} Most frequentist confidence-set based hard enforcement methods pick actions by solving the program \[ \max_{\theta \in \confset_t^\theta, a \in \mathcal{A}} \theta^\top a \textrm{ s.t. } \forall \Phi \in \confset_t^\Phi, \Phi a \le \alpha. \] Assuming, for simplicity, that $\asafe = 0,$ due to the structure of the confidence sets the above constraint translates to \[ \forall i \in [1:m], \hat\Phi_t^i a + \omega_t(\delta) \|a\|_{V_t^{-1}} \onem \le \alpha. \] Notice that this constitutes $m$ different second-order conic constraints. In fact, as discussed in \S\ref{appx:ecolts}, we expect $V_t^{-1}$ to have condition number scaling as $\Omega(t^{1/4})$, which adds further computational burdens to optimising under such constraints.

Of course, as written, the above program is nonconvex due to the objective $\theta^\top a$. This can be addressed via a standard `$\ell_1$-relaxation \cite{dani2008stochastic}, which reduces the problem to solving $2d$ optimisation problems with linear objectives and the above SOC constraints, while weakening regret to ${\tO(\mathcal{R}(\asafe) \sqrt{d^3 T})}$. This characterises the costs of most of these `optimistic-pessimistic' methods \cite[e.g.][]{pacchiano2021stochastic, pacchiano2024contextual, amani2019linear}. Afsharrad et al.~give a systematic and detailed account of these considerations \cite{afsharrad2024convex}. There are two exceptions. The \textsc{safe-lts} method of Moradipari et al.~\cite{moradipari2021safe} uses sampling to select the objective, but still imposes the same SOC constraints, thus needing only one optimisation each round. The \textsc{roful} method of Hutchinson et al.\cite{hutchinson2024directional} instead first picks an action according to the NP-hard method \textsc{doss}, and then scales it towards $\asafe$ as in  $\scolts$. Of course, note that $\scolts$ samples only one set of \emph{linear} constraints each round, and is efficient. There are also analytical differences between \textsc{roful} and $\scolts$, as discussed in \S\ref{sec:scolts}.

Turning to soft enforcement, as we mentioned in the main text, no efficient method is known. The main method herein for linear bandits is \textsc{doss} \cite{gangrade2024safe}, which instead picks actions by solving \[ \max_{\theta \in \confset_t^\theta, a \in \mathcal{A}} \theta^\top a \textrm{ s.t. } \exists \Phi \in \confset_t^\Phi : \Phi a \le \alpha.\] This $\exists$ operator renders this problem much more challenging, since now the constraint works out to the union of polytopes \[ \bigcup_{A \in \mathcal{C}_t^\Phi} \mathcal{A} \cap \{ \Phi a \le \alpha\}, \] which is highly nonconvex, and hard to condense or relax. Indeed, Gangrade et al. \cite{gangrade2024safe} propose using a similar $\ell_1$-relaxation as discussed above for both the objective and the constraints, but this now leads to $(2d)^{m+1}$-extreme points of the confidence sets (accounting for both $\theta$ and the $m$-rows of $\Phi$), leading to $(2d)^{m+1} \cdot \lpt \cdot \log(t)$ compute needed per round. In contrast, $\rcolts$ uses $\lpt \cdot \log^2(t)$ compute, and $\ecolts$ uses only $\lpt \cdot \log(t)$ compute.

\emph{More Details on the Failure of Prior Thompson Sampling Analyses.} \S\ref{sec:scolts} discusses the point where the prior unsaturation-based analysis of linear $\ts$ due to \cite{agrawal2013thompson} breaks down in the presence of unknown constraints in some detail. For the optimism-based analysis of \cite{abeille_lazaric}, we only briefly touch upon this in \S\ref{sec:rcolts}, and give a more detailed look in \S\ref{appx:conditions}. This section serves as a brief summary of the latter.

The analysis of Abeille and Lazaric relies on the convexity of the value function $J(\theta) := \max_{a \in \mathcal{A}} \theta^\top a$ to both analyse the roundwise regret ($\Delta(a_t)$) and to establish the frequency of a certain `global optimism' event (see \S\ref{appx:conditions}. With unknown constraints, the corresponding object of interest is the value function $K(\theta, \Phi) := \sup \{ \theta^\top a : a \in \mathcal{A}, \Phi a \le \alpha\}.$ This map is \emph{not} convex in $\Phi$, which causes both of these steps to break down. $\rcolts$ avoids this issue by resampling. It is also possible to give an analysis of $\scolts$ (and $\ecolts$) within the optimism framework, although this again utilises a scaling trick to bypass the same issue. Of course, we also establish optimism in a convexity-free way by analysing the local behaviour at $a_*$.

\emph{Finding a Feasible Point.} Notice that since there are plenty of polynomial time methods for hard enforcement in SLBs (even though the prior methods impose SOC constraints), in principle one can develop efficient soft-enforcement methods with regret scaling inversely in $\max_{a} \Gamma(a)$ by first discovering an action that has $\Gamma(a) \ge \textrm{const.} \cdot \max_{a} \Gamma(a),$ and then plugging this into a hard enforcement method. In this case, the exploration time would be random, but a constant, so the net risk would ostensibly be $O(1)$ as $T$ explodes, far below our $\sqrt{T}$ bounds, making the performance close to that of hard enforcement.\footnote{note that there is a cost, though: as stated before, the regret would scale inversely in the Slater gap, and until the safe point is discovered, would grow linearly.} Unfortunately, there is no \emph{efficient} method to discover such an action. Indeed, the closest method one can find in the literature is a feasibility \emph{test} due to Gangrade et al.~\cite{gangrade2024testing}, which can be extended to such an estimator, but this test uses \textsc{doss}-like optimisation to select actions, and needs to solve $(2d)^m$ optimisation problems each round. Our coupled noise design should have implications for this problem, but we do not pursue this direction further.

\section{Local Optimism, Global Optimism, and Unsaturation}\label{appx:conditions}

In \S\ref{sec:rcolts}, we (implicitly) defined a local-optimism condition on the perturbation law $\mu$ in the statement of Lemma~\ref{lemma:rcolts_main_lemma}, which is compared to a `global optimism' condition suggested by the prior work of Abeille \& Lazaric \cite{abeille_lazaric}. To further contextualise these, let us explicitly define them. 

\begin{definition}\label{defi:optimism_definition_appendix}
    Let $K(\theta, \Phi) := \sup \{ \theta^\top a : a \in \mathcal{A}, \Phi a \le \alpha\}$ denote the value function of optimising the objective $\theta$ under constraint matrix $\Phi$ over $\mathcal{A},$ with the convention that $\sup \emptyset = -\infty$. Recall that the \emph{local optimism event} at $a_*$ is \[ \optimism_t(\delta) := \{ (\eta, H) : \ttheta(\eta, t)^\top a_* \ge \theta_*^\top a_*, \tPhi(H,t) a_* \le \alpha\}, \] where $a_*$ is the constrained optimum for the true parameters $(\theta_*, \Phi_*)$. Further, define the \emph{global optimism event} \[ \Goptimism_t(\delta) := \{ (\eta, H) : K(\ttheta(\eta,t), \tPhi(H,t)) \ge \theta_*^\top a_* = K(\theta_*, \Phi_*) \}.\]

    For $\pi \in (0,1],$ we say that a law $\mu$ on $(\eta, H)$ satisfies $\pi$-local optimism if \[ \forall t, \mathbb{E}[ \mu(\optimism_t(\delta) ) |\hist_{t-1}] \indi_{\con_t(\delta)} \ge \pi \indi_{\con_t(\delta)},\] and similarly, that $\mu$ satisfies $\pi$-global optimism if \[ \forall t, \mathbb{E}[ \mu(\Goptimism_t(\delta) ) |\hist_{t-1}] \indi_{\con_t(\delta)} \ge \pi \indi_{\con_t(\delta)}.\] 
\end{definition}

Notice that $\Goptimism$ demands perturbations such that after optimising the perturbed parameters, the value of the resulting program is larger than $\theta_*^\top a_*,$ while $\optimism$ demands the stronger condition that $a_*$ is feasible, and its value increases. Evidently, $\optimism \subset \Goptimism,$ and so $\pi$-local optimism of $\mu$ implies $\pi$-global optimism. Naturally, the entirety of \S\ref{sec:rcolts} follows if we have a globally optimistic $\mu$ instead of locally optimistic $\mu$. We presented this section with $\optimism_t$ instead due to limited space in the main text.

As discussed in \S\ref{sec:coupled_noise_design}, we will also show, in \S\ref{appx:coupled_noise_proof}, $\optimism_t(\delta) \cap \con_t(\delta) \subset \mathsf{U}_t(\delta) \cap \con_t(\delta),$ i.e., when consistency holds, local optimism implies unsaturation. Thus, $\optimism_t$ links the global-optimism based framework of \cite{abeille_lazaric}, and the unsaturation based framework of \cite{agrawal2013thompson}. Nevertheless, technically, these are distinct events. 

Let us briefly note that the prior work \cite{agrawal2013thompson} essentially passes through the same strategy as us when establishing a good unsaturation rate, in that they argue that local-optimism holds frequently (although they do not consider unknown constraints, so their argument does not extend to our setting). On the other hand, \cite{abeille_lazaric} presents a convexity-based proof of frequent global optimism for linear $\ts$ without unknown constraints, while immediately breaks in our setting because $K(\theta, \Phi)$ is nonconvex in $\Phi$. We also reiterate that our coupled noise design of \S\ref{sec:coupled_noise_design} essentially takes the same conditions on perturbations used in these prior works, and extends them to produce the \emph{same} bounds on unsaturation or global-optimism rates by arguing that local-optimism holds. This means that these prior results do not capture the prevalence of these events beyond local optimism. Our simulations in \S\ref{appx:simulations} suggest that this leaves a significant amount of performance on the table, capturing which theoretically would require deeper understanding of $\mathsf{U}_t\setminus \optimism_t$ and $\Goptimism_t\setminus \optimism_t$.

\textbf{Role of These Conditions in Our Work.} To analyse $\scolts$ and $\ecolts,$ we used a look-back approach enabled by the unsaturation condition, while to analyse $\rcolts,$ we relied on a direct use of the optimism condition. It turns out that the unsaturation condition is not effective at capturing at least our strategy for analysing the resampling-based strategy $\rcolts$. The reason is that while the resampling will ensure that at least one of the optima of attaining the various $K(i,t)$ values will be unsaturated, we have no guarantee that the procedure we take of picking the $i_{*,t}$ that maximises $K(i,t)$ will choose an unsaturated action. On the other hand, the optimism condition \emph{can} be used to analyse $\scolts$ and $\ecolts$ directly (see \S\ref{appx:optimism_for_scolts}), but a direct execution of the previous optimism based approach \cite{abeille_lazaric} fails due to the lack of convexity of the map $K(\theta, \Phi)$. Instead, we have to directly analyse expressions of the form $\mathbb{E}[|K(\ttheta, \tPhi) - K(\ttheta', \tPhi')|\mid \hist_{t-1}],$ where $(\ttheta, \tPhi)$ and $(\ttheta', \tPhi')$ are iid draws of the perturbation at tie $t$. Under the assumption that there is an action with positive safety margin with small $M_t$, this can be executed via a similar scaling-based analysis, albeit at a loss of some factors in the regret bound (\S\ref{appx:optimism_for_scolts}). In our opinion the unsaturation based look-back analysis of $\Delta(a_t)$ is conceptually clearer, and we chose to present it in the main instead. 

Nevertheless, in terms of their explanatory power, neither condition dominates the other. Indeed, in simulations, we find both cases where unsaturation is frequent but global optimism is not, and cases where global optimism is frequent but unsaturation is not.\footnote{This is most pertinent for the setting where we drive the perturbations with independent noise, where in \S\ref{appx:simulations_decoupled} we observed that the unsaturation rate decayed with $m$, but the global optimism rate did not. Indeed, this is what prompted us to write the optimism-based analysis of \S\ref{appx:optimism_for_scolts}.} Of course, in our analysis, both of these are connected by local optimism as detailed above, which is rendered frequent through our coupled design. Nevertheless, the local optimism rate can be significantly smaller than the unsaturation and global optimism rates, particularly when the noise is shrunk far below the theoretically analysed setting of $\Theta(\sqrt{d})$-scale noise (see \S\ref{appx:simulations}). These observations again hint that developing a tight theory of linear $\ts$ (both with and without unknown constraints) requires a deeper understanding of the portion of these events that do not intersect with local optimism.

\section{Some Basic Tools For the Analysis}\label{appx:analysis}

We begin with some standard tools that are repeatedly utilised in the analysis. The first of these, termed the \emph{elliptical potential lemma} offers generic control on the accumulation of $\|a_t\|_{V_t^{-1}}$.

\begin{lemma}\label{lemma:elliptical_potential_lemma}
    \cite{abbasi2011improved, carpentier2020elliptical} For any sequence of actions $\{a_t\} \subset \{\|a\| \le 1\},$ and any $t$, \[ \sum_{s \le t} \acnorms^2 \le 2d \log(1 + t/d), \textit{ and } \sum_{s \le t} \acnorms \le \sqrt{2dt \log(1 + t/d)}.\] Further, for all $t, \delta, \omega_t(\delta) \le 1 + \sqrt{\nicefrac12 \log((m+1)/\delta) + \nicefrac d2 \log(1 + t/d)}.$
\end{lemma}

We further explicitly write the following instantiation of the Cauchy-Schwarz inequality pertinent to our setting. 
\begin{lemma}\label{lemma:cauchy_schwarz}
    For any positive definite matrix $V$. For pair of tuples $(\theta, \Phi)$ and $(\ttheta, \tPhi)$ lying in $\mathbb{R}^d \times \mathbb{R}^{m\times d}$ and any $a \in \mathbb{R}^d,$ it holds that \[ \max\left( |(\theta- \ttheta)^\top a|, \max_i |(\Phi^i - \tPhi^i) a| \right) \le \max( \|\ttheta-\theta\|_V, \max_i \|\tPhi^i - \Phi^i\|_V) \cdot \|a\|_{V^{-1}}. \]
    \begin{proof}
        Notice that $(\ttheta - \theta)^\top a = (\ttheta - \theta)^\top V^{1/2} V^{-1/2} a \le \| (V^{1/2} (\ttheta - \theta)\| \cdot \|V^{-1/2} a\|$. The claim follows by first repeating the same observation for each $(\Phi^i - \tPhi^i)$ (adjusting for the fact that these are row-vectors), and then recalling that (for column vectors) $\|a\|_M = \|M^{1/2} a\|$ by definition.
    \end{proof}
\end{lemma}

This immediately yields a proof of the concentration statement of Lemma~\ref{lemma:basic_noise_concentration_and_cauchy-schwarz}, which motivated the definition of $M_t(a)$. 

\begin{proof}[Proof of Lemma~\ref{lemma:basic_noise_concentration_and_cauchy-schwarz}]

    Notice that by a union bound \[ \mathbb{P}(\exists t : \max(\|\eta_t\|, \max_i \|H_t^i\|) > B(\delta_t)) \le \sum_{t} \delta_t = \delta. \] Now assume that $\max(\|\eta_t\|, \max_i \|H_t^i\|) \le B(\delta_t)$, and that the consistency event $\consistency$ holds. Then, via the triangle inequality, \[ \| \ttheta(\eta_t, t) - \theta_*\|_{V_t} \le \|\ttheta(\eta_t,t) - \hat\theta_t\|_{V_t} + \|\hat\theta_t - \theta_*\|_{V_t}. \] Of course, given $\consistency,$ the second term is smaller than $\omega_t(\delta)$. For the first, expanding the definition of $\ttheta(\cdot, \cdot)$, we find that \[ \|\ttheta(\eta_t,t) - \hat\theta_t \|_{V_t} =  \omega_t(\delta) \eta_t V_t^{-1/2} \|_{V_t} = \|\omega_t(\delta) \eta_t V_t^{-1/2} \cdot V_t^{1/2}\| \le \omega_t(\delta) \|\eta_t\|,\] and of course, $\|\eta_t\| \le B(\delta_t)$ by our assumption above. Thus, given the concentration assumption on $\|\eta_t\|$s and $\consistency,$ for any $t$, it holds that \[ \|\ttheta(\eta_t, t) - \theta_*\|_{V_t} \le (1+B(\delta_t)) \omega_t(\delta) \le B_t \omega_t(\delta). \] Of course, entirely the same applies to $\|\tPhi(H_t, t)^i - \Phi_*^i\|_{V_t}$, with $\eta$ replaced by $H_t^i$. The claim now follows by Lemma~\ref{lemma:cauchy_schwarz} and the fact that $\con(\delta) := \bigcap \consistency$ has chance at least $1-\delta$. \qedhere
    
\end{proof}

\section{Analysis of the Coupled Noise Design}\label{appx:coupled_noise_proof}

We will first execute the strategy described in \S\ref{sec:coupled_noise_design} to show that under the conditions of Lemma~\ref{lemma:coupled_noise_design}, local optimism is frequent. We will then use this to show the frequency of unsaturation.

\begin{lemma}\label{lemma:local_optimism_is_frequent}
    Let $p \in (0,1],$ and let $\nu$ be a law on $\mathbb{R}^{d \times 1}$ such that \[ \forall u \in \mathbb{R}^d, \nu(\{ \zeta : \zeta^\top u \ge \|u\|) \ge p. \] Let $\mu$ be the the pushforward of $\nu$ under the map $\zeta \mapsto (\zeta^\top , - \mathbf{1}_m \zeta^\top)$. Then, for all $t$, $\indi_{\consistency} \mathbb{E}[ \mu(\optimism_t(\delta))|\hist_{t-1}] \ge p \indi_{\consistency},$ where $\optimism_t(\delta)$ is the local optimism event (\ref{eqn:local_optimism}).
    \begin{proof}
        Observe that under a draw from $\mu,$ for all $t$, we have \begin{align*}  \ttheta^\top &:= (\ttheta(\eta,t))^\top = \hat\theta_t^\top + \omega_t(\delta) \zeta^\top V_t^{-1/2} \\ \tPhi &:= \tPhi(H, t) = \hat\Phi_t -   \mathbf{1}_m (\omega_t(\delta) \zeta^\top V_t^{-1/2}).  \end{align*} Further, recall that if the event $\consistency$ occurs, then, for all $a$,  \[ \hat\theta_t^\top a \ge \theta_*^\top a + \omega_t(\delta) \|V_t^{-1/2} a\|, \textrm{ and } \hat\Phi_t a \le \Phi_* a + \mathbf{1}_m (\omega_t(\delta) \|V_t^{-1/2} a\|, \] where we have the Cauchy-Schwarz inequality, and the fact that $\|a\|_{V_t^{-1}} = \|V_t^{-1/2} a\|$. Thus, assuming $\consistency,$ for any action $a$, we find that\begin{align*} \ttheta^\top a &\ge \theta_*^\top a + \omega_t(\delta) \left( \zeta^\top V_t^{-1/2} a - \|V_t^{-1/2} a\| \right),\\  \tPhi a &\ge \Phi_* a + \mathbf{1}_m \omega_t(\delta)  \left( \zeta^\top V_t^{-1/2} a - \|V_t^{-1/2} a\| \right).  \end{align*} Now, set $a = a_*,$ and suppose that $\zeta^\top V_t^{-1/2} a_* \ge \|V_t^{-1/2} a_*\|$. Then we can conclude that \[ \ttheta^\top a_* \ge \theta_*^\top a_* \textrm{ and } \tPhi a_* \le \Phi_* a_* \le \alpha,\] the final inequality holding since $a_*$ is of course feasible for the program it optimises. Of course, by definition, this means that the ensuing noise $\eta, H$ lie in the event $\optimism_t(\delta)$

        Now, it only remains to argue that $\zeta^\top V_t^{-1/2} a_* \ge \|\zeta^\top V_t^{-1/2} a_*\|$ happens with large chance given $\hist_{t-1}$. But notice that both $V_t^{-1/2}$ and (the constant) $a_*$ are $\hist_{t-1}$-measurable, and so are constant given it. It follows thus that \[ \mathbb{E}[ \nu(\{\zeta: \zeta^\top V_t^{-1/2} a_* > \|V_t^{-1/2} a_*\|\}) \mid \hist_{t-1}] \ge \inf_{u \in \mathbb{R}^d} \nu(\{ \zeta^\top u > \|u\|\}) \ge p. \qedhere\] 
    \end{proof}
\end{lemma}

To finish the proof of frequent unsaturation, we only need to determine that this local optimism induces unsaturation in the actions. 

\begin{proof}[Proof of Lemma~\ref{lemma:coupled_noise_design}]

    Fix a $t$, and assume consistency. Suppose that $\max(\|\eta_t\|, \max_i \|H_t^i\|) \le B(\delta_t)$. Note that given $\consistency$, this with chance at least $1-\delta_t$.  As a consequence, for any action $a \in \mathfrak{S}_t:= \{a : \Delta(a) > M_t(a)\},$ by following the proof of Lemma~\ref{lemma:basic_noise_concentration_and_cauchy-schwarz} we can conclude that \[\ttheta(\eta, t)^\top a \le \theta_*^\top a + M_t(a) = \theta_*^\top a_* - \Delta(a) + M_t(a) < \theta_*^\top a_*. \]
    Now, suppose that the drawn $\zeta$ induces local optimism. We claim that then all saturated actions are suboptimal. Indeed, by the above, each unsaturated action satisfies $\ttheta(\eta,t)^\top a < \theta_*^\top a_*$. But $\ttheta(\eta, t)^\top a_* \ge \theta_*^\top a_*,$ and further $\tPhi(H, t) a_* \le \alpha,$ means that there is an action that is feasible for the perturbed program with value strictly larger than that attained by any saturated action, i.e., any member of $\mathfrak{S}_t$. It thus follows that the optimum $a(\eta, H,t) \in \mathfrak{S}_t^c = \{ a : \Delta(a) \le M_t(a)\}.$ 

    Now, we know from Lemma~\ref{lemma:local_optimism_is_frequent} that given $\hist_{t-1}$, our assumptions of $\con_t(\delta)$ and the norm-control on $\|\eta_t\|, \max_i \|H_t^i\|$ imply that local optimism occurs with chance at least $p.$ Since these events occur with chance at least $1-\delta_t,$ this means that unsaturation occurs with chance at least $p - \delta_t$. Since definition~\ref{def:unsaturation} restricts attention to $t:\delta_t \le p/2,$ the statement follows.
\end{proof}

\subsection{Bounds for Simple Reference Laws}\label{appx:simple_reference_laws}

We argue that both the standard Gaussian, and the uniform law of the sphere of radius $\sqrt{3d}$ yield effective noise distributions for our coupled design. 

For the Gaussian, recall that if $Z \sim \mathscr{N}(0,I_d),$ then $\|Z\|^2$ is distributed as a $\chi^2$-random variable. A classical subexponential concentration argument \cite[e.g.][Lemma 1]{laurent2000adaptive} yields that for any $x$, \[ \mathbb{P}( \|Z\|^2 \ge d + 2\sqrt{dx} + 2x) \le e^{-x}. \] Note that $(d + 2\sqrt{dx} + 2x) \le (\sqrt{d} + \sqrt{2x})^2,$ and hence taking $x = \log(\nicefrac1\xi)$ in the above yields that $B(\xi) \le \sqrt{d} + \sqrt{2\log(1/\xi)}$. Further, due to the isotropicity of $Z, Z^\top u/\|u\| \overset{\mathrm{law}}{=} Z_1 \sim \mathscr{N}(0,1),$ and thus $\pi \ge 1- \Phi(1) \ge 0.158\dots$.

Further, notice that if $Z \sim \mathscr{N}(0,I_d),$ then $Y := \sqrt{3d} Z/\|Z\| \sim \mathrm{Unif}(\sqrt{3d} \cdot \mathbb{S}^d)$, and by isotropicity, for any $u, Y^\top u/\|u\| \overset{\mathrm{law}}{=} Y_1$. As a result, \begin{align*} \mathbb{P}( Y^\top u/\|u\| \ge 1) &= \mathbb{P}(Y_1 \ge 1) = \frac{1}{2} \mathbb{P}(Y_1^2 \ge 1) \\
&= \frac12 \mathbb{P}( (3d - 1) Z_1^2\ge \sum_{ i = 2}^d Z_i^2) \ge \frac12 \mathbb{P}(Z_1^2 \ge 1) \cdot \mathbb{P}( \sum_{i = 2}^d Z_i^2 \ge 3d - 1).\end{align*} But notice that $d-1 + 2\sqrt{ (d-1) \cdot d/3} + 2d/3 \le 3d -1,$ and thus, $\mathbb{P}(\sum_{i = 2}^d Z_i^2 \ge 3d - 1) \le \exp(-d/3)$. Invoking the bound on $\mathbb{P}(Z_1 \ge 1) = \frac12 \mathbb{P}(|Z_1| \ge 1)$ above, we conclude that $\pi \ge 0.15 \cdot (1-e^{-d/3}).$ Of course, $\|Y\| = \sqrt{3d}$ surely, giving the $B$ expression.

We note that while the above only shows a $0.15 (1-e^{-d/3})$ bound on the anticoncentration of the uniform law on $\sqrt{3d}\mathbb{S}^d,$ it is a simple matter of simulation to find that this is actually larger than $0.28$ for all $d$ - for small dimensions, the bound turns out to be very loose, while as $d$ diverges, this converges from above towards the chance that a standard Gaussian exceeds $1/\sqrt{3},$ which is $0.2818\dots$.  

\section{The Analysis of \textsc{s-colts}}\label{appx:scolts_analysis}

We move on to the analysis of $\scolts$. Before proceeding, we recall that in our presentation of $\scolts$ in Algorithm~\ref{alg:scolts}, we assumed access to a quantity $\Gamma_0 \in [\gsafe/2, \gsafe]$. We will first address how to obtain such a quantity by repeatedly playing $a_t = \asafe,$ and characterise how long this takes. For completenesss, the cost of this will be incorporated into our regret bound.

Beyond this, we need to characterise the subsequent time spent playing $\asafe$ due to $M_t(\asafe)$ being large, and to prove the look-back bound of Lemma~\ref{lemma:look_back_bound}, along with the characterisation of $\sum M_{\tau(t)}(a_{\tau(t)})$ offered in Lemma~\ref{lemma:width_control}. We will analyse these results in order, and finally show Theorem~\ref{thm:scolts_regret} using these results. 

\subsection{Identifying \texorpdfstring{$\Gamma_0$}{Estimated Margin} and Sampling Rate of \texorpdfstring{$\asafe$}{the Known Safe Action}}\label{appx:scolts_sampling_asafe}

We first discuss the determination of $\Gamma_0$. There are two main points to make: how to ensure a correct value of $\Gamma_0,$ and how many rounds of exploration this costs. To this end, we first recall the following nonasymptotic law of iterated logarithms \cite[e.g.][]{howard2021time}.

\begin{lemma}\label{lemma:lil}
    Let $\{ \mathfrak{F}_t\}$ be a filtration, and let $\{\xi_t\}$ be a process such that each $\xi_t$ is $\mathfrak{F}_t$-measurable, and is further conditionally centred and $1$-subGaussian given $\mathfrak{F}_{t-1}$. Then \[ \forall \delta \in (0,1], \mathbb{P}(\exists t : |Z_t| > \lil(t,\delta)) \le \delta,\] where $Z_t := \sum_{s \le t} \xi_t$, and \[\lil(t,\delta) := \sqrt{4t \log \frac{\max(1, \log(t))}{\delta}}.\] 
\end{lemma}

With this in hand, the determination of $\Gamma_0$ proceeds thus: we repeatedly play $\asafe,$ and maintain the running average $\mathrm{Av}_t = \sum_{s \le t} (\alpha - S_s)/t$. Further, we maintain the upper and lower bounds \[ u_t^i := \mathrm{Av}_t + \lil(t,\delta/m)/t, \ell_t^i := \mathrm{Av}_t - \lil(t,\delta)/t. \]  We stop at the first time when $\forall i, \ell_t^i \ge u_t^i/2$, and set $\Gamma_0 = \min_i \ell_t^i$. This stopping time is denoted $T_0$.

Let us first show that this procedure is correct, and bound the size of $T_0$.
\begin{lemma}
    Under the procedure specified above, it holds with probability at least $1-\delta$ that \[  \Gamma_0 \in  [\gsafe/2, \gsafe] \] and that \[ T_0 \le \frac{8}{\gsafe^2} \log(8/(\delta\gsafe^2))  \]
    \begin{proof}\label{lemma:initial_search_scolts}
        Notice that we can write \[ \mathrm{Av}_t = \alpha - \Phi_* \asafe + \sum_{s \le t} w_s^S/t.\] For succinctness, let us write $\Gamma = \alpha - \Phi_* \asafe$. Now, by our assumption on the noise $w_t^S,$ we observe that each coordinate of $w_t^{S}$ constitutes an adapted, centred, and $1$-subGaussian process.  Applying Lemma~\ref{lemma:lil} along with a union bound over the coordinates then tells us that with probability at least $1-\delta,$ \[ \forall t, |\mathrm{Av}_t -\Gamma| \le \lil(t,\delta/m)/t \cdot \onem. \] As a consequence, at all $t$, we have \[ u_t \ge \Gamma \ge \ell_t,\] where $u_t$ is the vector with $i$th coordinate $u_t^i$, and similarly for $\ell_t$. It follows thus that at the stopping time $T_0$, \[ \forall i, \ell_{T_0}^i \ge u_{T_0}^i/2 \implies \ell_{T_0} \ge \Gamma/2.   \] Of course, a fortiori, it follows that $\Gamma_0 = \min_i \ell_t^i \ge \min_i \Gamma^i/2 = \gsafe/2$. Further, of course, $\Gamma_0 \le \gsafe$ follows as well, since $\forall t \min_i \ell_t^i \le \min_i (\Gamma^i) = \gsafe$. 

        It only remains to control $T_0$. To this end, notice that for all $t$ \[ \ell_t = \mathrm{Av}_t - \lil(t,m/\delta)/t \cdot  \onem \ge \Gamma - 2\lil(t,m/\delta)/t  \cdot \onem,\] and similarly, \[ u_t \le \Gamma + 2\lil(t,m/\delta)/t \cdot \onem.\] Of course, then $\ell_t^i > u_t^i/2$ for all $t$ such that \[ \forall i,  \Gamma^i - \lil(t,m/\delta)/t \ge \Gamma^i/2 + \lil(t,m/\delta)/2t \iff \Gamma^i > 3\lil(t,m/\delta)/t. \] It follows thus that \[ T_0 \le \inf\{ t :  t\gsafe \ge 3\lil(t,m/\delta) \}.\] By a simple inversion, this can be bounded as \[ T_0 \le \inf\{t : t > 8/\gsafe^2 \log(1/\delta) \textrm{ and } t > 8/\gsafe^2 \log( 1 + \log(t))\}, \] which is bounded as \[ T_0 \le \frac{8}{\gsafe^2} \log(8/(\delta\gsafe^2)).\qedhere\]
    \end{proof}
\end{lemma}

\textbf{\emph{Number of Times $\asafe$ is sampled after $T_0$.}} Given the behaviour of $\Gamma_0$ above, we can further bound the number of times $\asafe$ is played after determining $\Gamma_0$. 

\begin{lemma}\label{lemma:explore_bounds_scolts}
    For any $\Gamma_0 > 0,$ and $T$, the number of times $\scolts$ plays $\asafe$ because $M_t(\asafe) > \Gamma_0/3$ is bounded as \( \frac{9\omega_T^2 B_T^2}{\Gamma_0^2} + 1.\)
    \begin{proof}
        Let $n_t$ denote the total number of times $\asafe$ has been played up to time $t$. Then, of course, $V_t \succeq  I + n_t \asafe\asafe^\top$. Now, recall that for symmetric positive definite matrices $A,B,$ it holds that $A\succeq B \iff B^{-1} \succeq A^{-1}$.\footnote{In more technical terms, inversion is monotone decreasing in the Loewner sense. A simple way to see this is to define $C = B^{-1/2} AB^{-1/2}$. Then $A\succeq B \implies C \succeq I$ (really iff), since for any $x$, $(B^{-1/2} x)^\top A (B^{-1/2}x) \ge (B^{-1/2} x)^\top B (B^{-1/2} x) \iff x^\top C x \ge x^\top x$. Using this for $y = C^{-1/2}x$ then gives $x^\top x = (C^{-1/2}x)^\top C(C^{-1/2}x) \ge (C^{-1/2}x)^\top (C^{-1/2}x) =  x^\top C^{-1}x$. Since $C^{-1} = B^{1/2} A^{-1}B^{1/2}$ (direct multiplication), the same trick yields $x^\top B^{-1} x = (B^{-1/2} x )^\top(B^{-1/2 }x) \ge x^\top B^{-1/2} (B^{1/2} A^{-1} B^{1/2}) B^{-1/2} x$, or in other words, $B^{-1} \succeq A^{-1}$. } Thus, we have \[ M_t(\asafe) \le \omega_t B_t  \sqrt{\asafe^\top (I + n_t \asafe\asafe^\top)^{-1}\asafe}. \]

        Now, by the Sherman-Morrisson formula, \[ \asafe (I+n_t \asafe\asafe^\top)^{-1} \asafe = \|\asafe\|^2 - \frac{\asafe^\top (n_t \asafe \asafe^\top) \asafe}{1 + n_t \|\asafe\|^2} = \frac{\|\asafe\|^2}{1 + n_t\|\asafe\|^2} \le \frac{1}{n_t}.\]

        It follows thus that \[ M_t(\asafe) \le \frac{\omega_t B_t}{\sqrt{n_t}}. \] Thus $M_t(\asafe) > \Gamma_0/3$ if and only if \[ n_t \le \frac{9\omega_t^2 B_t^2}{\Gamma_0^2}.\] Of course, each time this occurs, $n_t$ is increased by one. Consequently, the number of times $\asafe$ is played by time $t$ is at most \[ \frac{9 \omega_T^2 B_T^2}{\Gamma_0^2} + 1. \qedhere\]
        
        \end{proof}
\end{lemma}

Note that since $(\omega_T B_T)^2 = \Theta(d^2 + d\log(m/\delta))$ with our choice of the coupled noise driven by $\mathrm{Unif}(\sqrt{3d} \mathbb{S}^d),$ the bound above due to playing $\asafe$ due to too large an $M_t(\asafe)$ outstrips the bound on $T_0$ above as long as $\log(1/\gsafe) = o(d^2)$, as is to be expected.

\subsection{Proof of the Look-Back Bound}\label{appx:look_back_proof}

The main text provides a brief sketch of the approach. We will flesh out these details, as well as fill in the omitted aspects of the bound. To this end, we first state a result lower bounding $\rho_t$.

\begin{lemma}\label{lemma:rho_bounds}
    Assume that $\Gamma_0 \in [\gsafe/2, \gsafe]$, and that both $\con(\delta) = \bigcap \consistency$ and the event of Lemma~\ref{lemma:basic_noise_concentration_and_cauchy-schwarz} hold true. Then for all $t$ such that $M_t(\asafe) \le \Gamma_0/3,$ it holds that \[\rho_t \ge \frac{\gsafe}{\gsafe + 3M_t(b_t)} \] and \[ \rho_t \ge  \frac{2M_t(\asafe)}{2M_t(\asafe) + M_t(b_t)}.\] A fortiori, each of the following bounds is true: \begin{align*} (1-\rho_t)M_t(\asafe) &\le M_t(a_t), \\ \rho_tM_t(b_t) &\le 2M_t(a_t),  \textrm{ and } \\ (1-\rho_t)\gsafe &\le 6M_t(a_t). \end{align*}

    \begin{proof}
        Recall that $\rho_t$ is the largest $\rho$ in $[0,1]$ such that \[ \hat\Phi_t(\rho b_t + (1-\rho)\asafe) + \omega_t(\delta) \| \rho b_t + (1-\rho) \asafe\|_{V_t^{-1}} \onem \le \alpha. \] So, if we demonstrate a $\rho_0 \le 1$ that satisfies this inequality, then $\rho_t \ge \rho_0$.

    First note that under the assumption $M_t(\asafe) \le\Gamma_0/3,$ we know that \[  \tPhi_t \asafe \le \alpha - \gsafe \onem + \Gamma_0/3 \cdot \onem \le \alpha - 2\gsafe/3 \cdot \onem,\] and thus $b_t$ exists since the program defining it is feasible. Now, \begin{align*} \hat\Phi_t \asafe + \omega_t(\delta) \|\asafe\|_{V_t^{-1}}\onem &= \hat\Phi_t \asafe + \frac{M_t(\asafe)}{B_t} \onem \\ &\le \alpha - \gsafe \onem + \frac{2 M_t(\asafe)}{B_t} \onem \le \alpha - \frac{2\gsafe}{3}  \onem,\end{align*} using the consistency of the confidence sets (and the Cauchy-Schwarz inequality), along with the fact that $B_t = 1 + \max(1, B(\delta_t)) \ge 2$. Further, \[\hat\Phi_t b_t +\omega_(\delta)\|b_t\|_{V_t^{-1}} \le \tPhi_t b_t + \frac{B_t - 1}{B_t} M_t(b_t) \onem + \frac{1}{B_t} M_t(b_t) \onem  \le \alpha + M_t(b_t)\onem. \] Therefore,  \begin{align*}   &\phantom{\le} \hat\Phi_t( \rho b_t + (1-\rho)\asafe) + \omega_t(\delta)\|\rho b_t + (1-\rho)\asafe\|_{V_t^{-1}} \\
    &\le \rho\left(\hat\Phi_t b_t + \frac{M_t(b_t)}{B_t} \onem \right) + (1-\rho)\left( \hat\Phi_t \asafe + \frac{M_t(\asafe)}{B_t} \onem \right) \\
    &\le \alpha + \left( \rho M_t(b_t) - (1-\rho) \gsafe/3 \right)\onem. \end{align*} It is straightforward to find that the additive term above is nonpositive for $\rho_0 = \frac{\gsafe}{\gsafe + 3M_t(b_t)},$ and thus $\rho_t \ge \frac{\gsafe}{\gsafe + 3M_t(b_t)}$.   

    Further, since $M_t(\asafe) \le \gsafe/3,$ we also have \[ \alpha - 2\gsafe/3 \le \alpha - 2M_t(\asafe). \] Thus, we can also write \[ \hat\Phi_t \asafe + M_t(\asafe)/B_t \mathbf{1}_m \le \alpha - 2 M_t(\asafe)\onem, \] and carrying out the same procedure then shows that \[ \rho_t \ge \frac{2M_t(\asafe)}{2M_t(\asafe) + M_t(b_t)}.\]

    To draw the final conclusions, first observe that \[1 -\rho_t \le \frac{M_t(b_t)}{2M_t(\asafe) + M_t(b_t)} \implies 2(1-\rho_t)M_t(\asafe) \le \rho_t M_t(b_t) \le M_t(a_t) + (1-\rho_t)M_t(\asafe), \] where we used the fact that $\rho_t b_t = a_t - (1-\rho_t)\asafe,$ and that $M_t$ is a scaling of a norm. It follows that $(1-\rho_t)M_t(\asafe) \le M_t(a_t)$, and of course, that $\rho_t M_t(b_t) \le 2M_t(a_t)$. Further, by a similar calculation, \[ (1-\rho_t) \le \frac{3M_t(b_t)}{\gsafe + 3M_t(b_t)} \implies (1-\rho_t) \gsafe \le 3 \rho_t M_t(b_t) \le 6M_t(a_t).  \qedhere\]
    \end{proof}
\end{lemma}

\textbf{\emph{Proving the Look-Back Bound.}} The above control on $(1-\rho_t)$ is natural in light of terms of the form $(1-\rho_t) \Delta(\asafe)$ appearing in the bound as sketched in the main text. Let us now complete this argument.

\newcommand{\bbs}{\bar{b}_{s \to t}}
\newcommand{\sst}{\sigma_{s\to t}}
\newcommand{\sstau}{\sigma_{\tau \to t}}
\begin{proof}[Proof of Lemma \ref{lemma:look_back_bound}]
    We assume $\Gamma_0 \in [\gsafe/2, \gsafe],$ and that the event of Lemma~\ref{lemma:basic_noise_concentration_and_cauchy-schwarz} holds, as well as $\con(\delta)$. Together these occur with chance at least $1-3\delta$.      

    Now, we begin as in the main text, by observing that  \[ \Delta(a_t) = \Delta(\rho_t b_t + (1-\rho_t) \asafe) = \rho_t \Delta(b_t) + (1-\rho_t) \Delta(\asafe). \] Let $s < t$ be such that $M_s(\asafe) \le \Gamma_0/3$ as well. Then we further know that \[ \tPhi_s b_s \le \alpha \implies \tPhi_t b_s \le \alpha + (M_t(b_s) + M_s(b_s))\onem.\] As a consequence, for \[\sigma_{s \to t} := \frac{\gsafe}{\gsafe + 3(M_t(b_s) + M_s(b_s))}, \] we have  \[ \tPhi_t (\sigma_{s\to t} b_s + (1-\sigma_{s\to t})\asafe) \le \alpha + \left( \sigma_{s \to t} (M_t(b_s) + M_s(b_s)) - \frac{2(1-\sigma_{s\to t})\gsafe}{3}\right)\onem \le \alpha. \] Define $\bar{b}_{s\to t} = \sigma_{s \to t} b_s + (1-\sigma_{s \to t})\asafe$. By the above observation, $\bar{b}_{s\to t}$ is feasible for $\tPhi_t$, and therefore $\ttheta_t^\top \bar{b}_{s\to t} \le \ttheta_t^\top b_t$. To use this, we note that \begin{align*} \Delta(b_t) &= \Delta(\bbs) + \theta_*^\top(\bbs - b_t) = \Delta(\bbs) + \ttheta_t^\top(\bbs - b_t) + (\ttheta_t - \theta_*)^\top(\bbs - b_t) \\ &\le \Delta(\bbs) + \ttheta_t^\top(\bbs - b_t) + M_t(\bbs) + M_t(b_t),\end{align*} where we first use Lemma~\ref{lemma:basic_noise_concentration_and_cauchy-schwarz}, and then bound $M_t(\bbs-b_t)$ by using the fact that $M_t$ is a norm. The second term above is of course nonpositive, and so can be dropped while retaining the upper bound. Further, \[ \Delta(\bbs) = \sigma_{s \to t}\Delta(b_s) + (1-\sst) \Delta(\asafe). \] This leaves us with the bound \begin{align*} \Delta(a_t) &\le (1-\rho_t + \rho_t(1-\sst))\Delta(\asafe) \\&\qquad \qquad + \rho_t \left( \sst \Delta(b_s) + M_t(b_t) + \sst M_t(b_s) + (1-\sst)M_t(\asafe)\right),\end{align*} where we used the triangle inequality and the fact that $M_t$ is a scaling of a norm to write the final two terms.
    We will, of course, evaluate this at $s = \tau(t)$. In the subsequent, we will just write $\tau$ instead of $\tau(t)$ for the sake of reducing the density of notation. Using the fact that $\Delta(b_\tau) \le M_\tau(b_\tau),$ we set up the basic bound
    \begin{align*} \Delta(a_t) &\le (1-\rho_t + \rho_t(1-\sstau))\Delta(\asafe) \\ &\qquad\qquad + \rho_t M_t(b_t) +  \rho_t(\sstau (M_\tau(b_\tau) + M_t(b_\tau)) + (1-\sstau)M_t(\asafe)). \end{align*}

    Now, first observe that by Lemma~\ref{lemma:rho_bounds}, \[ (1-\rho_t) \Delta(\asafe) \le 6 \frac{\Delta(\asafe)}{\gsafe} M_t(a_t), \] and further \[ \rho_t M_t(b_t) \le 2M_t(a_t). \] We are left with terms scaling with $\sstau$ or $(1-\sstau)$. For this, we first observe that \[ M_t(b_\tau) = B_t\omega_t \|b_\tau\|_{V_t^{-1}} \le \frac{B_t \omega_t}{B_\tau \omega_\tau} \cdot B_\tau \omega_\tau \|b_\tau\|_{V_\tau^{-1}} = \frac{B_t \omega_t}{B_\tau \omega_\tau} \cdot M_\tau(b_\tau),  \] where we use the fact that $V_t$ is nondecreasing (in the positive definite ordering). Let us abbreviate $J_{\tau \to t} := 1 + \nicefrac{B_t \omega_t}{(B_\tau \omega_\tau)}$. Upon observing that $\rho_t \le 1,$ to finish the argument, we only need to control \[ (1-\sigma_{\tau \to t}) (\Delta(\asafe) + M_t(\asafe)) +  J_{\tau \to t} \sigma_{\tau \to t} M_\tau(b_\tau). \] Now, notice that since $M_\tau(\asafe) \le \Gamma_0/3,$ \[ \sstau = \frac{\gsafe}{\gsafe + 3(M_t(b_\tau) + M_\tau(b_\tau))}\le \frac{\gsafe}{\gsafe + 3M_\tau(b_\tau)} \le \rho_\tau \le \frac{2M_\tau(a_\tau)}{M_\tau(b_\tau)},\] where we invoke Lemma~\ref{lemma:rho_bounds} for the final two inequalities. Thus, we find that \[ \sstau J_{\tau \to t} M_\tau(b_\tau) \le J_{\tau \to t} \cdot \rho_\tau M_\tau(b_\tau) \le 2 J_{\tau \to t}M_\tau(a_\tau). \] This leaves us with the term \( (1-\sstau) (\Delta(\asafe) + M_t(\asafe)). \) To bound this, observe that  \begin{align*} (1-\sstau) &= \frac{3(M_t(b_\tau) + M_\tau(b_\tau))}{\gsafe + 3 (M_t(b_\tau) + M_\tau(b_\tau))} \\
    \implies (1-\sstau)\gsafe &= 3 \sstau (M_t(b_\tau) + M_\tau(b_\tau)) \le  3 \sstau J_{\tau \to t}M_\tau(b_\tau).\end{align*} Recall from the discussion above that $\sstau  M_\tau(b_\tau) \le \rho_\tau M_\tau(b_\tau) \le 2 M_\tau(a_\tau)$. Using this, and the fact that $M_t(\asafe) \le \gsafe/3$ then yields \[ (1-\sstau) (\Delta(\asafe) + M_t(\asafe)) \le 6 J_{\tau \to t} \frac{\Delta(\asafe)}{\gsafe} M_\tau(a_\tau) + 2 J_{\tau \to t} M_\tau(a_\tau). \]
    Putting everything together, then, we conclude that \[\Delta(a_t) \le 6\frac{\Delta(\asafe)}{\gsafe} \left( M_t(a_t) + J_{\tau \to t} M_\tau(a_\tau) \right) + 2 M_t(a_t) + 4J_{\tau \to t} M_\tau(a_\tau), \] which of course implies the bound we set out to show.
\end{proof}

\subsection{Controlling Accumulation in the Look-Back Bound}\label{appx:better_width_control}

We proceed to control the accumulation of the look-back terms.

\begin{proof}[Proof of Lemma~\ref{lemma:width_control}]
      Since $B_t$ and $\omega_t$ are nondecreasing, for any $s \le t \le T,$ we have \[ (1 + (B_t\omega_t(\delta)/B_s\omega_s(\delta)) ) M_s(a_s) = (B_s\omega_s(\delta) + B_t\omega_t(\delta))\acnorms \le 2B_T\omega_T(\delta) \acnorms.\] Let $\mathcal{T}_T = \{ t \le T : M_t(\asafe) \le \Gamma_0/3\}$, and $\mathcal{U}_T = \{s \in \mathcal{T}_T : \Delta(b_s) \le M_t(b_s)\}$. Then notice that  \[ \sum_{t \in \mathcal{T}_T}  \|a_{\tau(t)}\|_{V_{\tau(t)}^{-1}} = \sum_{s \in \mathcal{U}_T} L_s \acnorms,\] where $L_s = |\{t \in \mathcal{T}_T : \tau(t) = s\}|$ is the number of times $s$ serves as $\tau(t)$ for some $t$. But this is the same as the time (restricted to $\mathcal{T}_T$) between $s$ and the \emph{next} member of $\mathcal{U}_T$, i.e., the length of the `run' of the method playing saturated actions (plus one). 
      
      At this point, a weaker bound of the form $\frac{2}{\chi} \log(T^2/\delta) \sum_{s \in \mathcal{U}_T} \acnorms$ is straightforward: each round has at least a chance $\chi/2$ of picking a saturated $b_t,$ and so the chance that the $k$th such run has length greater than $\frac{2}{\chi} \log(k(k+1)/\delta)$ is at most $\delta/k(k+1).$ Since there are at most $T$ runs up to time $T$, union bounding over this gives $\max_{\mathcal{U}_T} L_s \le 1+ 2\log(T(T+1)/\delta)/\chi$.

      The rest of this proof is devoted to give a more refined martingale analysis that saves upon the multiplicative $\log(T)$ term above. We encapsulate this as an auxiliary Lemma below.

      \begin{lemma}\label{lemma:auxiliary_lookback_width_lemma}
          In the setting of Lemma~\ref{lemma:width_control}, it holds that with probability at least $1-\delta,$ \[ \sum_{s \in \mathcal{U}_T}  L_s \acnorms \le \frac{5}{\chi}(\sum_{s \in \mathcal{U}_T} \acnorms + \log(1/\delta))  \]
      \end{lemma}
      This result is shown below. Assuming this result, the original claim follows immediately, since due to the nonnegativity of $\|\cdot\|_{\cdot}, \sum_{s \in \mathcal{U}_T} \|a_s\|_{V_s^{-1}} \le \sum_{t \le T} \|a_t\|_{V_t^{-1}}.$
\end{proof}

To finish the argument, we move on to showing the auxiliary lemma described above. 

\begin{proof}[Proof of Lemma~\ref{lemma:auxiliary_lookback_width_lemma}]
We work with the reduction to $\sum_{s \in \mathcal{U}_T}\acnorms$ established above. Let us denote $\zeta_i = \inf\{ t > \zeta_{i-1} : M_t(\asafe) \le \Gamma_0/3, \Delta(b_t) \le M_t(b_t) \}$ as the times that an unsaturated action is picked, with $\zeta_0 := 0$---for $i: \zeta_i \le T,$ these are precisely the elements of $\mathcal{U}_T$. Notice that this $\{\zeta_i\}$ is a sequence of stopping times adapted to the history $\{\hist_{t}\}.$ Let us further denote $L_i = (\zeta_{i+1} - \zeta_{i}),$ for $i \ge 0$ (this corresponds to $L_s$, where $s = \zeta_i$). The object we need to control is \[ \sum_{i : \zeta_i \le T} L_i X_i,\] where $X_i  = \| a_{\zeta_i}\|_{V_{\zeta_i}^{-1}} \in [0,1]$, the lower bound being since $X_i$ is a norm, and the upper bound since $V_{\zeta_i} \succeq I$. For notational convenience, we always set $X_0 = 1$. Now, to control this, let us first pass to the associated sigma algebrae of the $\zeta_i$ past, denoted as \[ \mathfrak{G}_{i} := \zeta( \hist_{\zeta_{i}}).\] Notice that since $\zeta_i$ is nondecreasing, we know that $\{\mathfrak{G}_i\}$ forms a filtration. Of course, by definition, $X_i$ are adapted to $\mathfrak{G}_i$, while $L_i$ are adapted to $\mathfrak{G}_{i+1}$. We further know that $L_i$ is the time (including $\zeta_i$) between $\zeta_i$ and $\zeta_{i+1}$. But then for each $t > \zeta_i,$ $P(\zeta_{i+1} = t| \zeta_{i+1} > t-1, \hist_{t-1}) \ge \chi/2$. As a result, these $L_i$s are conditionally stochatically domainted by a geometric random variable, i.e., \[ \mathbb{P}( L_i > 1+k|\mathfrak{G}_i) \le (1-\chi/2)^k. \] 

This in turn implies that for any $\lambda$ small enough,  \[ \mathbb{E}[ e^{\lambda (L_i - 1) X_i} |\mathfrak{G}_{i}] \le \frac{\chi/2}{1 - (1-\chi/2) e^{\lambda X_i}}. \] In the subsequent, we will need to select a $\lambda$ that is independent of all of these $L_i, X_i$. To ensure that the calculation makes sense, we ensure that $(1-\chi/2) e^{\lambda} \le 1 $ (which suffices since $0 \le X_i \le 1$). Let us define $F_i(\lambda) := - \log( (1- (1-\chi/2) e^{\lambda X_i})/(\chi/2))$. Then by the above calculation, we find that the process $\{M_i\}$ with $M_0 := 1$ and  \[ M_i := \exp\left( \lambda \sum (L_i - 1)X_i - \sum F_i(\lambda) \right)  \] is a nonnegative supermartingale with respect to the filtration $\{\mathfrak{F}_i\}$ with $\mathfrak{F}_i = \mathfrak{G}_{i+1}\}$ and $\mathfrak{F}_0$ defined to be the trivial sigma algebra. Thus, by Ville's inequality, \( \mathbb{P}(\exists i : M_i > 1/\delta) \le \delta. \) Taking logarithms, we find that with probability at least $1-\delta,$ it holds that \[ \forall n, \sum_{i \le n} L_i X_i \le \sum_{i \le n }X_i  + \frac{\log(1/\delta)}{\lambda} +  \sum_{i \le n} \frac{F_i(\lambda)}{\lambda}, \] as long as $0  < \lambda < - \log (1-\chi/2)$. All we need now is a convenient bound on $F_i(\lambda)$ and a judicious choice of $\lambda$. To this end, we observe the following simple result.

\begin{lemma}
    For any constant $u \in (0,1),$ consider the map \( f(x) := -\log\frac{1- u e^x}{1-u} \) over the domain $[0, - \log(u)).$  Then for all $x \in [0, - \frac12 \log(u)],$ we have \[ f(x) \le \frac{\sqrt{u}}{1-\sqrt{u}}x.\]
    \begin{proof}
        Observe that \begin{align*}
            f'(x) &= \frac{ue^x}{1-ue^x} = \frac{e^{f}}{1-u} (1- e^{-f}(1-u)) = \frac{e^f}{1-u} - 1 \ge 0
        \end{align*}
         The inequalities above arise since $e^{f} = \frac{1-u}{1-ue^x} > 1-u$ using the fact that $e^x \ge 1$. By taking another derivative, we may see that $f'$ itself is an increasing function. Now, suppose $g(x)$ satisfies \[ g(0) = f(0) = 0, \textit{ and } \forall x, g'(x) = f'(-\nicefrac12 \log(u)) = \frac{\sqrt{u}}{1-\sqrt{u}}.\] Then since $g'(x) \ge f'(x)$ for all $x \in [0,-\frac12 \log(u)],$ by the fundamental theorem of calculus it follows that for all $x \le -\frac12 \log u, f(x) = \int_0^x f' \le \int_0^x g' = g(x).$ \qedhere
    \end{proof}
\end{lemma}
Now, of course, $F_i(\lambda) = f(\lambda X_i), $ with $u = 1-\chi/2$. Then setting $\lambda = -\frac12 \log(1-\chi/2),$ we have \[ \forall n, \sum_{i \le n} L_i X_i \le \sum X_i + \frac{\log(1/\delta)}{-\log(1-\chi/2)/2} + \sum_{i \le n} \frac{\sqrt{1-\chi/2}}{1-\sqrt{1-\chi/2}} X_i.\]

To get the form needed, we observe that \[\frac{\sqrt{1-v}}{1-\sqrt{1-v}} \le \frac{2}{v} \iff (2+v)^2(1-v) \le 4 \iff -v^3 - 3v^2 \le 0,\] and of course $-\log(1-v)/2 \ge v/2$. Plugging in $v = \chi/2 > 0,$ we end up at \[\forall n, \sum_{i \le n} L_i X_i \le \left( 1 + \frac{4}{\chi}\right)\sum_{i \le n} X_i + \frac{4}{\chi} \log(1/\delta). \] 
Note that no explicit $n$-dependent term appears in the above. This makes sense: we essentially have the $X_i$s acting as `time steps', and so $\sum L_i X_i$ should behave as $(1 + 2/\chi)\sum X_i + O(\sqrt{\sum X_i \log(1/\delta)/\chi} + \log(1/\delta) ),$ via a Bernstein-type computation. In our case, the square root terms do not meaningfully help the solution,\footnote{since there will always be an additive $\log(1/\delta)$ and $\frac{1}{\chi} \sum X_i$ term} and so we just pick a convenient $\lambda$ instead.\footnote{and in the process, avoid the subtleties of the dependence of $\lambda$ on the $X_i$s if we optimised it} Now, going back to our original object of study, we have  $L_i = L_{\zeta_i}, X_i = {\|a_{\zeta_i}\|_{V_{\zeta_i}^{-1}}},$ and these $\zeta_i$s are precisely the members of $\mathcal{U}_T,$ so we conclude that \[ \forall T, \sum_{s \in \mathcal{U}_T} L_s X_s \le \frac{5}{\chi} \left(\sum_{s \in \mathcal{U}_t} X_s + \log(1/\delta)\right). \qedhere\]
\end{proof}

\subsection{Regret and Risk Bounds for \scolts}

With the above pieces in place, we move on to showing the final bounds on the behaviour of $\scolts$.

\begin{proof}[Proof of Theorem~\ref{thm:scolts_regret}]

    We first argue the safety properties. Firstly, in the exploration phase, as well as to explore, we repeatedly play $\asafe$. But this is, by definition, safe, and so accrues no safety cost. When not playing $\asafe$, the selected action $a_t$ at time $t$ satisfies \[ \hat\Phi_t a_t + \omega_t(\delta) \|a_t\|_{V_t^{-1}} \onem \le \alpha.\] But, given the consistency event $\consistency,$ \[\forall a, \Phi_* a \le \hat\Phi_t a + \omega_t(\delta) \|a\|_{V_t^{-1]}}, \] and so $\Phi_* a_t \le \alpha$. Since $\con(\delta) := \bigcap \consistency$ holds with chance at least $1-\delta$, it follows that $a_t$ is safe at every $t$, and a fortiori, $\saf_T = 0$ for every $T$. 

    Let us turn to the regret analysis. Fix any $T$. We break the regret analysis into four pieces: the regret accrued over the initial exploration, that accrued after this phase, but when $M_t(\asafe) > \Gamma_0/3,$ and over the time $\mathcal{T}_T := \{ t \ge T_0: M_t(\asafe) \le \Gamma_0/3\},$ and finally the regret incurred up to the time $\inf\{t : \delta_t > \chi/2\}.$ 
    
    The last of these is the most trivial to handle: the number of such rounds is bounded as $\sqrt{2 \delta/\chi},$ and the regret in any round is at most $2$.
    
    For the first case, Lemma~\ref{lemma:initial_search_scolts} ensures that with probability at least $1-\delta,$ this phase has length at most \[ \frac{8}{\gsafe^2} \log(8/(\delta\gsafe^2)), \] and further, the output $\Gamma_0$ is at least $\gsafe/2$ at the end. Using this to instantiate Lemma~\ref{lemma:explore_bounds_scolts}, we further find that the number of times $\asafe$ is selected beyond this initial exploration is in total bounded as \[ 1+ \frac{36\omega_T^2 B_T^2}{\gsafe^2}.\] Together these contribute at most \[ \Delta(\asafe) \cdot \frac{ 44\omega_T^2 B_T^2}{\gsafe^2}\log(8/\delta \gsafe^2) \] to the regret. 

    This leaves us with the times at which $M_t(\asafe) \le \Gamma_0/3,$ for which we apply Lemma~\ref{lemma:look_back_bound}, along with the control of Lemma~\ref{lemma:width_control} to find that the net regret accrued thus is bounded as \[ O\left( \left( 1  + \frac{\Delta(\asafe)}{\Gamma(\asafe)} \right) B_T \omega_T(\delta) \cdot \frac{5}{\chi} \left( \sum_{t \le T} \acnorm 
 + \log(1/\delta) \right) \right).\] To complete the book-keeping, the probabilistic events required for this are the consistency of the confidence sets, that for all $t , \max( \|\eta_t\|, \max_i \|H_t^i\|)$ is bounded by $B(\delta_t)$, and of course the bound on the times between unsaturated $b_t$ being constructed from Lemma~\ref{lemma:width_control}. Together, these occur with chance at least $1-3\delta,$ and putting the same together with the stopping time bound, we conclude that with chance at least $1-4\delta,$ $\scolts(\mu,\delta)$ satisfies the regret bound \[ \eff_T \le \left( 1 + \frac{\Delta(\asafe)}{\gsafe}\right) \tO\left( \frac{\omega_t(\delta) B_T}{\chi} \sum_{t \le T} \acnorm \right) +  \frac{\Delta(\asafe)}{\gsafe} \cdot \tO\left(\frac{\omega_T^2 B_T^2}{\gsafe} \right) + \sqrt{\frac{8\delta}{\chi}}. \] Now, invoking Lemma~\ref{lemma:elliptical_potential_lemma}, we can bound $\omega_T(\delta) = \tO(\sqrt{d} + \log(m/\delta)),$ and $\sum \acnorm = \tO(\sqrt{d T})$. Finally, for the law $\mu$ induced via the coupled noise design by $\mathrm{Unif}(\sqrt{3d} \mathbb{S}^d),$ we further know that $B_T = O(\sqrt{d})$ and $\chi \ge 0.28$. Of course, for this noise, $B_t = \sqrt{3d}$ with certainty, which boosts the probability above to $1-3\delta$. The claim thus follows for $\scolts(\mu,\delta/3)$. \qedhere   
\end{proof}

\subsection{An Optimism-Based Analysis of \scolts}\label{appx:optimism_for_scolts}

We analyse $\scolts$ under the assumption that $\mu$ satisfies $B$-concentration and $\pi$-global optimism (Definition~\ref{defi:optimism_definition_appendix}). We shall be somewhat informal in executing this. 

\textbf{Setting Up.} We first note that regret accrued over rounds in which \( M_t(b_t) > \gsafe/3\) and $M_t(\asafe) \le \Gamma_0/3$ is small.  Indeed, \begin{align*} \sum_{t \in \mathcal{T}_T} \indi\{  M_t(b_t) > \gsafe/3\} &\le \frac{9}{\gsafe^2} \sum_{t \in \mathcal{T}_T} M_t(b_t)^2 \\  &\le \frac{16}{\gsafe^2} \sum_{t \in \mathcal{T}_T} \frac{M_t(a_t)^2}{\rho_t^2} = \tO\left( \frac{d^5}{\gsafe^4} \right),\end{align*} where $\mathcal{T}_T = \{t : M_t(\asafe) \le \Gamma_0/3\},$ and we used the bound on $M_t(b_t)$ from Lemma~\ref{lemma:rho_bounds}, along with the fact that since $b_t \in \mathcal{A}, M_t(b_t) \le B_t\omega_t = \tO(d),$ which in turn implies that $\rho_t \ge \gsafe/\tilde{\Omega}(d)$. Naturally, this additive term is much weaker than that seen in Theorem~\ref{thm:scolts_regret}. Nevertheless, the optimism-based framework does recover a similar main term. In particular we will show a regret bound of $\tO( \gsafe^{-1}\sqrt{d^3 T})$ 

The point of the above condition is that (using Lemma~\ref{lemma:rho_bounds}), if $M_t(b_t) \le \gsafe/3,$ then $\rho_t \ge \frac12$. We will repeatedly use this fact in the subsequent.

Now, we begin similarly to the previous analysis by using \[ \Delta(a_t) = (1-\rho_t) \Delta(\asafe) + \rho_t \Delta(b_t) \le (1-\rho_t)\Delta(\asafe) + \Delta(b_t). \] The first term is well-controlled, as detailed in the proof of Lemma~\ref{lemma:look_back_bound}. So, we only need to worry about $\sum \Delta(b_t)$. Notice that for this it suffices to control $\sum \mathbb{E}[  \Delta(b_t)|\hist_{t-1}]$. Indeed, $\Delta(b_t) \le 1$ (and if it is $\le 0,$ we can just drop it from the sum, i.e., we could study $(\Delta(b_t))_+$ instead with no change in the argument), so the difference $\sum_{t \le T} \Delta(b_t) - \mathbb{E}[\Delta(b_t)|\hist_{t-1}]$ is a martingale with increments lying in $[-1,1],$ and the LIL (Lemma~\ref{lemma:lil}) ensures that for all $T$ simultaneously, the difference between these is $O(\sqrt{T \log (\log(T)/\delta)})$ with chance at least $1-\delta$.

From the above, then, we can restrict attention to $t$ such that $M_t(\asafe) \le \Gamma_0/3, \rho_t \ge \frac12.$ Finally, recalling the notation $K(\theta, \Phi) = \max\{\theta^\top a : a \in \mathcal{A}, \Phi a \le \alpha\}$ from Definition~\ref{defi:optimism_definition_appendix}, we observe that \begin{align*} \Delta(b_t) &= \theta_*^\top a_* - \ttheta_t^\top b_t + (\ttheta_t -\theta_*)^\top b_t \\ &\le K(\theta_*, \Phi_*) - K(\ttheta_t, \tPhi_t) + M_t(b_t) \\
&\le K(\theta_*, \Phi_*) - K(\ttheta_t, \tPhi_t) + 4M_t(a_t),\end{align*} where we used Lemma~\ref{lemma:basic_noise_concentration_and_cauchy-schwarz}, and Lemma~\ref{lemma:rho_bounds} along with the fact that $\rho_t \ge 1/2.$ Now note that the final term above is summable to $\tO(\sqrt{d^3 T})$. Thus, it equivalently suffices to analyse the behaviour of $\mathbb{E}_{t-1}[ K(\theta_*,\Phi_*) - K(\ttheta_t, \tPhi_t)|\hist_{t-1}]$. In order to do so, we begin with a `symmetrisation' lemma.  

\newcommand{\btheta}{\bar{\theta}}
\newcommand{\bPhi}{\bar{\Phi}}
\begin{lemma}
    Let $(\ttheta_t, \tPhi_t)$ and $(\btheta_t, \bPhi_t)$ denote two independent copies of parameter perturbations at time $t$. Let $\mathbb{E}_{t-1}[ \cdot] := \mathbb{E}[\cdot \mid \hist_{t-1}]$.  If $\mu$ satisfies $\pi$-global optimism, then \[ \indi_{\consistency}\mathbb{E}_{t-1}[ ( K(\theta_*, \Phi_*) - K(\ttheta_t, \tPhi_t)] \le \indi_{\consistency} \cdot \frac{1}{\pi} \mathbb{E}_{t-1}[ |K(\ttheta_t, \tPhi_t) - K(\btheta_t, \bPhi_t)|].\]
    \begin{proof}
        Let $\bar{\Goptimism} := \{ K(\btheta_t, \bPhi_t) \ge K(\theta_*, \Phi_*)\}$. Since $K(\theta_*,\Phi_*)$ is a constant, and since $(\ttheta_t, \tPhi_t)$ are independent of $(\btheta_t, \bPhi_t)$ given $\hist_{t-1},$ we conclude that \[ \mathbb{E}_{t-1}[ K(\theta_*, \Phi_*) - K(\ttheta_t, \tPhi_t)] = \mathbb{E}_{t-1}[ K(\theta_*, \Phi_*) - K(\ttheta_t, \tPhi_t) \mid \bar{\Goptimism} ]. \] But given $\bar{\Goptimism},$ $K(\theta_*, \Phi_*) \le K(\btheta_t, \bPhi_t),$ and so  \begin{align*} \mathbb{E}_{t-1}[ K(\theta_*, \Phi_*) - K(\ttheta_t, \tPhi_t)] &\le  \mathbb{E}_{t-1}[ K(\btheta_t, \bPhi_t) - K(\ttheta_t, \tPhi_t) \mid \bar{\Goptimism} ] \\ &\le \mathbb{E}_{t-1}[ |K(\btheta_t, \bPhi_t) - K(\ttheta_t, \tPhi_t)| \mid \bar{\Goptimism} ]. \end{align*} Finally, for any nonnegative random variable $X,$ and any event $\mathsf{E},$ it holds that \[ \mathbb{E}_{t-1}[X|\mathsf{E}] \mathbb{E}_{t-1}[ \indi_{\mathsf{E}}] = \mathbb{E}_{t-1}[X \indi_{\mathsf{E}} ] \le \mathbb{E}_{t-1}[X].\] The claim follows upon taking $X = |K(\btheta_t, \bPhi_t) - K(\ttheta_t, \tPhi_t)|, \mathsf{E} = \bar{\Goptimism},$ and recognising that due to $\pi$-optimism, $\bar{\Goptimism}$ satisfies $\mathbb{E}_{t-1}[\indi_{\bar{\Goptimism}}]\indi_{\con_t} \ge \pi \indi_{\con_t}$.  
    \end{proof}
\end{lemma}

\newcommand{\tb}{\tilde{b}}
\newcommand{\tlambda}{\tilde{\lambda}}
\newcommand{\bb}{\bar{b}}
\newcommand{\blambda}{\bar{\lambda}}
\newcommand{\bst}{\bar{\sigma}_t}

The main question now becomes controlling how far the deviations in $K$ can go. We control this using a similar scaling trick as in the proof of Lemma~\ref{lemma:look_back_bound}. 

For the sake of clarity, we will denote the optimiser of $K(\ttheta_t, \tPhi_t)$ as $\tb_t$ (instead of just $b_t$ as in the rest of the text), and similarly that of $K(\btheta_t, \bPhi_t)$ as $\bb_t$. Our goal is to control (the conditional mean of) \[ |\btheta_t^\top \bb_t^\top - \ttheta_t^\top \tb_t|. \] Naturally, the core issue remains that $\bb_t$ and $\tb_t$ are optima in distinct feasible sets, and so it is hard to, e.g., compare $\ttheta_t^\top \tb_t$ and $\ttheta_t^\top \bb_t$. To this end, we observe that \[ \bPhi_t \bb_t \le \alpha \implies \Phi_* b_t \le \alpha + M_t(\bb_t) \onem \implies \tPhi_t \bb_t \le \alpha + 2M_t(b_t)\onem, \] as long as consistency and the boundedness of the noise norms holds (which occurs with high probability). Using this and the fact that $\tPhi_t \asafe \le \alpha - 2\gsafe/3 \onem,$ we find that \[ \tPhi_t (\bst \bb_t + (1-\bst)\asafe) \le \alpha, \textrm{ where } \bst = \frac{\gsafe}{\gsafe + 3M_t(\bb_t)}.\] Thus, we may write \[ \btheta_t^\top \bb_t - \ttheta_t^\top \tb_t = (1-\bst) \btheta^\top \bb_t + \bst (\btheta_t - \ttheta_t^\top)\bb_t + \ttheta_t^\top (\bst \bb_t - \tb_t). \] Above, the third term is nonpositive, while the second term may be bounded by $2 \bst M_t(\bb_t)$, which can further be bounded by $8 M_t(\bar{a}_t)$ upon recalling that $\rho_t(\bb_t) \ge \frac12$ and the bound on $\rho_t M_t(b_t)$ in Lemma~\ref{lemma:rho_bounds}. This leaves the first term. It is tempting to bound this directly via $\btheta_t^\top \bb_t \le \|\btheta_t\| \|\bb_t\|$, but notice that the former can be as large as $B_t \sim \sqrt{d}$. Instead, we can use the related bound \[  (1-\bst)(\btheta^\top \bb_t) \le  (1-\bst) M_t(\bb_t) + (1-\bst)\theta_*^\top \bb_t.\] 
Now notice that $(1-\bst) \le 1,$  and $M_t(\bb_t) \le 4M_t(\bar{a}_t)$ controls the first term. Similarly, $\theta_*^\top \bb_t \le 1$ (both have norm bounded by $1$), so the second term is bounded by $1-\bst \le \frac{3M_t(\bb_t)}{\gsafe} \le 12 \frac{M_t(\bar{a}_t)}{\gsafe}$. Putting these together, we conclude that \[ (1-\bst)(\btheta_t^\top \bb_t) \le 4 M_t(\bar{a}_t) + \frac{12 M_t(\bar{a}_t)}{\gsafe},\] which in turn yields the bound \[ K(\btheta_t, \bPhi_t) - K(\ttheta_t, \tPhi_t) \le 12M_t(\bar{a}_t) + \frac{12M_t(\bar{a}_t)}{\gsafe} \le \frac{24 M_t(\bar{a}_t)}{\gsafe}. \]

Of course, switching the roles of $(\btheta_t, \bPhi_t)$ and $(\ttheta_t, \tPhi_t),$ we have an analogous bound on $K(\ttheta_t, \tPhi_t) - K(\btheta_t, \bPhi_t)$. Putting these together, we conclude that \[  |K(\btheta_t,\bPhi_t) - K(\ttheta_t, \tPhi_t)| \le  \frac{24(M_t(\bar{a}_t) + M_t(\tilde{a}_t))}{\gsafe}.\] Finally, notice that $\bar{a}_t$, $\tilde{a}_t$, and the actually selected action $a_t$ all have the same distribution given $\hist_{t-1}$. We can thus conclude that \[\mathbb{E}_{t-1}[ |K(\btheta_t, \bPhi_t) - K(\ttheta_t, \tPhi_t)| ] \le 48\mathbb{E}_{t-1} \left[ \frac{M_t(a_t)}{\gsafe} \right].\]

With this in hand, the issue returns to one of concentration. We know that $\sum M_t(a_t)$ is $\tO(\sqrt{d^3 T}),$ and each $M_t(a_t)$ is bounded as $O(d)$ and so $\sum M_t(a_t) - \mathbb{E}_{t-1}[M_t(a_t)]$ enjoys concentration at the scale $d \lil (T,\delta) = \tO(\sqrt{d^2 T}) = o(\sqrt{d^3 T})$. Thus, passing back to the the unconstrained sums, we end up with a bound of the form \[ \eff_T = \tO\left( \gsafe^{-1} \sqrt{d^3 T}\right) + \tO(d^5\gsafe^{-4}). \] The main loss in the main term above is that instead of a $\Delta(\asafe)/\gsafe$, we just have a $\gsafe^{-1}$ term in the bound. This can be lossy, e.g., when $\asafe$ is very close to $a_*$, but in the regime $\Delta(\asafe) = \Omega(1),$ it recovers essentially the same guarantees as Theorem~\ref{thm:scolts_regret}, albeit with a weaker additive term.  

\section{The Analysis of Soft Constraint Enforcement Methods.}

\subsection{The Analysis of \rcolts}\label{appx:rcolts_analysis}

Let us first show the optimism result for $\rcolts$

\begin{proof}[Proof of Lemma~\ref{lemma:rcolts_main_lemma}]
    Fix any $t$, and assume $\consistency$. For each $i \in [1:I_t],$ we know that $K(i,t) := K(\ttheta(i,t), \tPhi(i,t)) \ge \ttheta(i,t)^\top a_* \ge \theta_*^\top a_*$ whenever the event $\optimism$ occurs, and thus this inequality holds with chance at least $\pi$ in every round. Since the draws are all independent given $\hist_{t-1}$, the chance that $\max K(i,t) < \theta_*^\top a_*$ is at most $(1-\pi)^{I_t} \le \exp( - \log(1/\delta_t) r \cdot \pi ) \le \delta_t = \nicefrac{\delta}{t(t+1)}.$ Thus, if we assume that $\con(\delta) := \bigcap \consistency$ holds true, the chance that at any $t$, $K(i_t,t) < \theta_*^\top a_*$ is at most $\sum \delta_t = \delta$. By Lemma~\ref{lemma:online_linear_regression}, $\con(\delta)$ holds with chance at least $1-\delta,$ and we are done.
\end{proof}
Of course, the above proof, and thus the statement of this Lemma, holds verbatim if we replace $\optimism_t$ by $\Goptimism_t$ (Definition~\ref{defi:optimism_definition_appendix}). 

With the optimism result of Lemma~\ref{lemma:rcolts_main_lemma}, the argument underlying Theorem~\ref{thm:rcolts_bounds} is extremely standard.

\begin{proof}[Proof of Theorem~\ref{thm:rcolts_bounds}]
    Assume consistency, and that at every $t$, $\ttheta_t^\top a_t \ge \theta_*^\top a_*$. Since we sample at most $2 + r \log(1/\delta_t)$ programs in round $t$, we further know that with probability at least $1-\delta,$ \[ \forall t, \max_{i} \left( \max \|\eta(i,t)\|, \max_j \|H_t^j(i,t)\| )\right)\| \le \beta_t := B(\delta_t/(2+r\log(1/\delta_t)).\] Assume that this too occurs, and define $\tilde{M}_t(a) = \omega_t(\delta) (1 + \beta_t) \|a\|_{V_t^{-1}}$. Then, using consistency, \[\theta_*^\top a_t \ge \ttheta_t^\top a_t - \tilde{M}_t(a_t), \Phi_* a_t \le \tPhi_t a_t + \tilde{M}_t(a_t) \onem \le \alpha + \tilde{M}_t(a_t) \onem.\] So, the safety risk is bounded as \[ \saf_T \le \sum_t \tilde{M}_t(a_t) \le \omega_T(\delta) (1 + \beta_T) \sum_{t \le T} \acnorm.\] Further, \[ \theta_*^\top a_* - \theta_*^\top a_t \le \theta_*^\top a_* - \ttheta_t^\top a_t + \tilde{M}_t(a_t),\] which implies that \[ \eff_T \le \omega_T(\delta) (1 + \beta_T \sum_{t \le T} \acnorm \] as well. Now Lemma~\ref{lemma:elliptical_potential_lemma} controls $\omega_T \sum_{t \le T} \acnorm$ to $\tO(\sqrt{d^2 T})$, and for our selected noise, the coupled design driven by $\mathrm{Unif}(\sqrt{3d} \mathbb{S}^d),$ we have $B(\cdot) = \sqrt{3d}$ independently of $t,r,\delta,$ and thus $\beta_T = \sqrt{3d}$.     
    The events needed to show the above were the consistency, the concentration of the sampled noise to $\beta_t$ at each time $t$, and the optimism event of Lemma~\ref{lemma:rcolts_main_lemma}. Again, the second happens with certainty for us, and so the above bounds hold at all $T$ with chance at least $1-2\delta$. Consequently, the result was stated for $\rcolts(\mu,r,\delta/2).$
\end{proof}

\subsection{The Exploratory-COLTS Method}\label{appx:ecolts}

As discussed in \S\ref{sec:rcolts}, the Exploratory COLTS, or $\ecolts$ method, augments $\colts$ with a low-rate of flat exploration, and exploits the resulting (eventual) perturbed feasibility of actions with nontrivial safety margin to bootstrap the scaling-based analysis of $\scolts$ to a soft-enforcement result without resampling.

The main distinction lies, of course, in the fact that in the soft enforcement setting, we do not have access to a given safe action $\asafe.$ To motivate the method, let us consider how $\scolts$ uses the knowledge of $\asafe$. This occurs in three ways: to ensure the existence of $a(\eta_t, H_t, t),$ to compute the action $a_t$ from this, and to enable the look-back analysis of Lemma~\ref{lemma:look_back_bound}. The second use is easy to address: we will simply play $a_t = a(\eta_t,H_t,t)$ if it exists. The key observation is that rather than explicit knowledge of any one particular safe action, as long as \emph{some action} $a$ exists such that $M_t(a) \le \Gamma(a)/3,$ the entirely of the first and third uses can be recovered, and so the machinery of \S\ref{sec:scolts} can be enabled. 

\begin{wrapfigure}[14]{r}{.5\linewidth}
\vspace{-2\baselineskip}
\begin{minipage}{\linewidth}
\begin{algorithm}[H]
   \caption{Exploratory-$\colts$ ($\ecolts(\mu, \delta)$)}
   \label{alg:ecolts}
\begin{algorithmic}[1]
   \State \textbf{Input}: $\mu, \delta,$ exploration policy.
   \State \textbf{Initialise}: $u_0 \gets 0, B_t \gets 1+B(\delta_t)$
   \For{$t = 1, 2, \dots$}
   \State Draw $(\eta_t, H_t) \sim \mu$. 
   \If{$u_{t-1} \le B_t \omega_t(\delta)\sqrt{dt}$ OR $a(\eta_t,H_t,t)$ does not exist}
   \State Pick $a_t$ via exploration policy.
   \State $u_t \gets u_{t-1} + 1$.
   \Else
   \State $a_t \gets a(\eta_t, H_t, t), u_t \gets u_{t-1}.$
   \EndIf
   \State Play $a_t,$ observe $R_t,S_t,$ update $\hist_t$. 
   \EndFor
\end{algorithmic} 
\end{algorithm}
\end{minipage}
\end{wrapfigure}
\textbf{Forced Exploration.} We enable the \emph{eventual} existence of such actions by introducing a small rate of \emph{forced exploration} in our method $\ecolts$. Concretely, we demand a `$\kappa$-good' exploration policy over $\mathcal{A}$, i.e., one such that after $N$ exploratory actions $e_1, \cdots, e_N,$ we are assured that $\sum e_i e_i^\top \succeq \kappa \lfloor \nicefrac Nd\rfloor I_d,$ where $\kappa > 0$ is a constant. This can, e.g., be done by playing the elements of a barycentric spanner of $\mathcal{A}$ in round-robin \cite{awerbuch2008online, dani2008stochastic}. The resulting $\kappa$ is a geometric property of $\mathcal{A},$ and we note that $\kappa$ only enters the analysis, not the algorithm. 

Let us call a time step $t$ where the exploratory policy is executed an `E-step'. In $\ecolts$, we ensure tha tat any $t$, at least $B_t \omega_t {\sqrt{dt}}$ such E-steps have been performed, and if not, we force an E-step. Note that we expect that the majority of the learning process occurs at steps other than E-steps, since this is where the informative action $a(\eta_t, H_t, t)$ is played. Consequently, we will call such steps `L-steps'.

By our requirement of enough E-steps, at any L-step $t$, the sample second moment matrix $V_t$ satsifies $V_t \succeq \kappa B_t \omega_t {\sqrt{t/d}} I_d,$ and so, \[ \forall a, M_t(a) \le \psi(t) := \left(\frac{d B_t^2 \omega_t^2}{\kappa^2 t}\right)^{1/4} \cdot \|a\|.\] This means that at such $t$, any $a$ with $\Gamma(a) > 2\psi(t)/3$ satisfies $M_t(a) \le \Gamma(a)/3,$ and so $a(\eta_t, H_t, t)$ exists, and we may use the analysis of \S\ref{sec:scolts} for such $a$.

 \textbf{Regret Bound.} The above insight is the main driver of the result of Theorem~\ref{thm:ecolts_new_main}, which we show in \S\ref{appx:ecolts_analysis} to follow. Recall that this states that under the $\ecolts$ strategy, executed with a $\mu$ constructed through the coupled noise design with base measure $\mathrm{Unif}(\sqrt{3d} \mathbb{S}^d)$, the risk and regret satisfy, w.h.p., the bounds \begin{align*} \saf_T &= \tO(\sqrt{d^3 T}) + \min_{a}\tO\Bigl(  \frac{d^3\|a\|^4}{\kappa^2 \Gamma(a)^4}\Bigr),\textit{ and } \\ \eff_T &=  \min_{a : \Gamma(a) > 0} \left\{ \mathcal{R}(a) \tO(\sqrt{d^3 T}) + \tO\Bigl(\frac{d^3 \|a\|^4}{\kappa^2 \Gamma(a)^4}\Bigr) \right\}, \end{align*} where $\kappa$ is precisely the `goodness-factor' of the exploratory policy. Let us briefly discuss this result.

\textbf{Risk bound.} Unlike $\scolts$, $\ecolts$ suffers nontrivial risk, which is unavoidable due to the lack of knowledge of $\asafe$ \cite{pacchiano2021stochastic}. The ${\tO(\sqrt{d^3 T})}$ risk above above is comparable to the ${\tO{\sqrt{d^2 T})}}$ risk of the prior soft enforcement method \textsc{doss} \cite{gangrade2024safe}, with a ${\sqrt{d}}$ loss again attributable to efficiency. Note that compared to $\rcolts,$ the risk bound is essentially the same, but now incurs an extra additive term scaling, essentially, with $(\max_a \Gamma(a))^{-4}$. Thus, a nontrivial risk bound is only shown if this maximum is strictly positive, i.e., under Slater's condition. Nevertheless, the term is additive, and scales with $T$ only logarithmically (through a dependence on $\omega_t(\delta)$, and so in typical scenarios is not expected to dominate as $T$ diverges, although the fourth-power dependence on this quantity would increase the `burn-in' time of this result.

\textbf{Regret bound.} As discussed in \S\ref{sec:rcolts}, the main term of the regret bound above improves over that of $\scolts$, since it \emph{minimises} over $\mathcal{R}(a)$, rather than working with the arbitrary $\mathcal{R}(\asafe)$. Note that finding the minimiser of $\mathcal{R}$ may be challenging, but $\ecolts$ nevertheless adapts to this. However, the additive lower-order term is larger than in $\scolts$ due to the `flat' exploration of $\ecolts$, and its practical effect is unclear. In simple simulations, we do observe a significant regret improvement (\S\ref{appx:simulations}). We note that the $\kappa$-good exploration condition only affects the lower order term in $\eff_T$, although again the fourth order dependence on $\Gamma(a)$ is nontrivial. Of course, relative to $\ecolts$, the result suffers from an instance-dependence, and again, unless Slater's condition is satisfied, it is ineffective. 

\textbf{Practical Role of Forced Exploration.} $\ecolts$ uses forced exploration to ensure that $V_t$ is large, which leads to both feasibility of the perturbed program, and the scaling-based analysis. In practice, however, one expects that low-regret algorithms satisfy ${\max_{a} \|a\|_{V_t^{-1}} \lesssim t^{-1/4}\|a\|}$ directly, the idea being that actions with larger ${V_t^{-1}}$-norm represent underexplored directions that would naturally be selected (recent work has made strides towards actually proving such a result, although it does not quite get there \cite{banerjee2023exploration}). Thus we believe that this forced exploration can practically be omitted except when the perturbed program is infeasible. Indeed, in simulations, we find that this strategy already has good regret (\S\ref{appx:simulations}).

\subsubsection{The Analysis of \ecolts}\label{appx:ecolts_analysis}

We will essentially reuse our analysis of $\scolts$, with slight variations. 
\begin{proof}[Proof of Theorem~\ref{thm:ecolts_new_main}]

    We will first discuss the bound on the regret. Throughout, we assume consistency, and the noise concentration event of Lemma~\ref{lemma:basic_noise_concentration_and_cauchy-schwarz}. We will further just write $\omega_t$ instead of $\omega_t(\delta)$. Recall the terminology that every $t$ in which we pick an action according to the exploratory policy is called an `E-step', and every other step an `L-step'. Here E and L stand for exploration and learning respectively, the idea being that the former constitute the basic exploration required to enable feasibility under perturbations, and so the main learning process occurs in L-steps.
    
    Note that the number of E-steps up to time $t$ is explicitly delineated to be at most $\lceil B_t \omega_t \sqrt{dt}\rceil$. Using the $\kappa$-good assumption, then, we find that at every L-step, \[ V_t \succeq \kappa B_t\omega_t \sqrt{t/d} I \iff  (\kappa B_t \omega_t \sqrt{t/d})^{-1} I \succeq V_t^{-1}.\]  

    Now, fix any action $a_0$ with $\Gamma(a_0) > 0.$ Then notice that at any L-step, \[ \|a_0\|_{V_t^{-1}}^2 \le \frac{\sqrt{d}\|a_0\|^2}{\kappa B_t \omega_t \sqrt{t}} \implies M_t(a_0)^2 \le \frac{B_t\omega_t \|a_0\|^2}{\kappa} \cdot \sqrt{d/t}. \] Thus, for all  \[ t \ge t_0(a_0) := \inf \left\{ t :  \frac{3^4 d \|a_0\|^4 B_t^2 \omega_t^2}{\kappa^2 \Gamma(a_0)^4} \le t\right\}\] that are L-steps, we know that as long as the noises $\eta_t, H_t$ satisfy the bound of Lemma~\ref{lemma:basic_noise_concentration_and_cauchy-schwarz}, $\tPhi_t a_0 \le \alpha - 2\Gamma(a_0)/3 \onem.$ Note that since $\omega_t^2 \le d \log(t) + \log(m/\delta),$ and since under our choice of coupled noise, $B_t = \sqrt{3d}$ for all $t$, we can conclude that \begin{align*} t_0(a_0) &\le  \frac{Cd^3 \|a_0\|^4}{\kappa^2 \Gamma(a_0)^4} \log\frac{Cd^3 \|a_0\|^4}{\kappa^2 \Gamma(a_0)^4} + \frac{C d^2 \|a_0\|^4 \log(m/\delta)}{\kappa^2 \Gamma(a_0)^4} \log\frac{Cd^2 \|a_0\|^4\log(m/\delta)}{\kappa^2 \Gamma(a_0)^4} \\ &= \tO\left(\frac{d^3 \|a_0\|^4}{\kappa^2 \Gamma(a_0)^4} \right),  \end{align*} where $C$ is some large enough constant ($C = 4 \cdot 81$ suffices). This implies that at all $t > t_0(a_0)$ at which the number of E-steps, $u_t,$ is large enough, the perturbed program is feasible, and $a_t$ exists. Thus, after this time, no extraneous E-steps are accrued due to infeasibility of the perturbed program.

    \newcommand{\bat}{\bar{a}_{\tau \to t}}
    At this point we apply the proof of Lemma~\ref{lemma:look_back_bound}, with $\rho_t = 1$. Let \[\tau = \tau(t) = \sup\{ s \le t : \Delta(a_s) \le M_t(a_s), M_t(a_0) \le \Gamma(a_0)/3\}. \] Now, $a_\tau$ need not be feasible for $\tPhi_t$, but we know that $\tPhi_\tau a_\tau \le \alpha \implies \tPhi_t a_\tau \le \alpha + M_t(a_\tau) + M_\tau(a_\tau)$. So for \[ \sstau := \frac{\Gamma(a_0)}{\Gamma(a_0) + 3(M_t(a_\tau) + M_\tau(a_\tau))},\] we know that \[ \tPhi_t( \sstau a_{\tau} + (1-\sstau)a_0 ) \le \alpha.\] Let $\bat := \sstau a_{\tau} + (1-\sstau)a_0$. Then we can write \begin{align*}
        \Delta(a_t) &= \Delta(\bat) + \theta_*^\top (\bat - a_t  ) \\
        &\le \Delta(\bat) + \ttheta_t^\top(\bat - a_t) + M_t(a_t) + M_t(\bat)\\
        &\le \sstau \Delta(a_\tau) + (1-\sstau)\Delta(a_0) + M_t(a_t) + \sstau M_t(a_\tau) + (1-\sstau) M_t(a_0)\\
        &\le (1-\sstau)\Delta(a_0) + M_t(a_t) + \sstau (M_t(a_\tau) + M_\tau(a_\tau)) + (1-\sstau) M_t(a_0),
    \end{align*}
    where in the end we used the fact that $\Delta(a_\tau) \le M_\tau(a_\tau)$. Now, \[ 1-\sstau \le \frac{3 (M_t(a_\tau) + M_\tau(a_\tau)) }{\Gamma_0},\] and of course $M_t(a_0) \le \Gamma_0$. We end up with a bound of the form \[ \Delta(a_t) \le C\left( 1 + \frac{\Delta(a_0)}{\Gamma(a_0)} \right) ( M_t(a_t) + M_\tau(a_\tau) + M_t(a_{\tau})),  \] which is essentially the same as that of Lemma~\ref{lemma:look_back_bound}. Given this, we can immediately invoke Lemma~\ref{lemma:width_control} (appropritaely modifying by $\asafe \to a_0$ and $\Gamma_0 \to \Gamma(a_0)$). We end up with the control that \[ \sum_{t \le T, M_t(a_0) \le \Gamma_0/3} \Delta(a_t) = \tO\left( \left( 1 + \frac{\Delta(a_0)}{\Gamma(a_0)} \right) \cdot \frac{B_T \omega_T}{\chi} \cdot \sum_{t \le T} \acnorm\right).\] For our choice of noise (being the coupled design executed with $\nu = \mathrm{Unif}(\sqrt{3d} \mathbb{S}^d),$ we have $\chi = \Omega(1), B = O(\sqrt{d}),$ and so this can be bounded as \[ \sum_{t \le T, M_t(a_0) \le \Gamma_0/3} \Delta(a_t) = \tO\left( \mathcal{R}(a_0) d^3 T\right).\]

    The above holds true for all $t > t_0(a_0)$ that were not E-steps. Before $t_0(a_0),$ we may bound the per-round regret by $2$. Finally, we are left with the E-steps after the time $t_0(a_0)$. Since, as argued above, no extraneous E-steps due to the infeasiblity of perturbed programs occur, we can then, for $T \ge t_0(a_0)$, simply bound the total number of E-steps by $1 + B_T\omega_T \sqrt{dT},$ and accrue roudwise regret of at most $2$ in these steps. With our chosen noise, $B_t = O(\sqrt{d})$, this cost is $\tO(\sqrt{d^3 T}).$ Summing these three contributions, and invoking the bound on $t_0(a_0)$ finishes the argument upon recognizing that $a_0$ is arbitrary, and so we may minimise over it.

    Turning now to the risk, first observe that for any $t > T_0 := \min_a t_0(a),$ there exists at least one action such that $M_t(a) \le \Gamma(a)/3,$ and so the perturbed program is always feasible, i.e., $a(\eta_t, H_t, t)$ exists. Now, consider subsequent times. Observe that in L-steps, since $\tPhi_t a_t \le \alpha,$ we know by Lemma~\ref{lemma:basic_noise_concentration_and_cauchy-schwarz} that \( \Phi_* a_t \le \alpha + M_t(a_t)\onem, \) assuming consistency and the concentration of $\max(\|\eta_t\|, \max_i \|H_t^i\|)$. Thus, in L-steps, the risk accrued at any time is at most $M_t(a_t)$. On the other hand, in E-steps, the risk accumulated can be bounded by just $1$ (using the boundedness of $\Phi_*$ and $\mathcal{A}$, and so we only need to work out the total number of these. But after time $T_0$ such an E-step only occurs to make sure that the net number of E-steps is at least $B_t\omega_t\sqrt{dt},$ and so the total number of such steps is at most $B_T \omega_T \sqrt{dT}$. 

    Putting these together, we conclude that the net risk accrued is bounded as   \begin{align*} \saf_T &\le T_0 +  \sum_{\substack{ T_0 \le  t \le T\\ t \textrm{ is an E-step}}} 1 + \sum_{\substack{t \le T,\\ t \textrm{ is an L-step}}} M_t(a_t) = B_T \cdot\tO(\sqrt{d^2 T} + \min_a \tO\left( \frac{dB_t^2 \omega_t^2 \|a\|^4}{\kappa^2\Gamma(a)^4}\right). \end{align*}
    Invoking Lemma~\ref{lemma:elliptical_potential_lemma}, as well as the fact that $B_T = B_t = \sqrt{3d}$ for our noise design, the claim follows.

    Finally, let us account for the probabilistic conditions needed: we need the concentration event of Lemma~\ref{lemma:width_control} to hold for the regret bound, and the consistency and noise-boundedness events for both. Of course, the second is not actually needed, since our noise is bounded always. Together, then,  these occur with chance at least $1-2\delta$ under our noise design. Of course, then, passing to $\ecolts(\mu,\delta/2)$ yields the claimed result.
\end{proof}

\section{Simulation Study}\label{appx:simulations}

We conduct simulation studies to investigate the behaviour of $\ecolts/\rcolts,$ and of $\scolts.$ We first study the soft and hard constraint enforcement problems with our coupled noise design. After this, we investigate the behaviour of $\colts$ methods using independent (or decoupled) noise in \S\ref{appx:simulations_decoupled}. All experiments were executed on a consumer-grade laptop computer running a Ryzen-5 chip, in the MATLAB environment, and the total time of all experiments ran to about 8 hours.

\subsection{Soft Constraint Enforcement}

We begin with studying the behaviour of the soft constraint enforcement strategies $\ecolts$ and $\rcolts$. Throughout, we treat $\ecolts$ as $\rcolts(\mu,0,\delta)$, with no exploration.

\textbf{Setting.} We set $\Phi_*$ to be a certain $9\times 9$ directed adjacency matrix, $A$, obtained from \url{https://sparse.tamu.edu/vanHeukelum/cage4}, which is a $\approx 60\%$ populated matrix with $d = m = 9.$ The rows of $\Phi_*$ were normalised to have norm $1$. We set $\theta_* = \mathbf{1}_d/\sqrt{d}, \mathcal{A} = [0,1/\sqrt{d}]^d$, and enforce the unknown constraints $\Phi_* a \le 0.8 \cdot \nicefrac{\mathbf{1}}{\sqrt{d}}$. We note that the action $0$ is always safe, no matter the $\tPhi_t$. This choice is intentional, and lets us avoid the inconvenient fixed exploration in $\ecolts$ and $\scolts$. Throughout, we set $\delta = 0.1$. 

As stated above, for the bulk of this section, we will implement $\ecolts$ without forced exploration. Indeed, this is not required since $0$ is always feasible, as discussed above. This can equivalently be interpreted as $\rcolts$ with the resampling parameter $r = 0$. 

\textbf{Effect of Noise Rate.} As previously noted, in linear $\ts$, small perturbation noise---of the scale $1$ rather than $\Theta(\sqrt{d})$---retains sufficient rates of global optimism and unsaturation to enable good regret behaviour. Note that such a small noise directly reduces $B_T,$ and thus we would expect it to improve our regret behaviour by a factor of about $\sqrt{d}$. In order to exploit this, we begin by conducting pilot experiments with our coupled noise design to determine a reasonable noise scale for us to use.

Concretely, we drive our coupled noise design with the laws $\nu_\gamma = \mathrm{Unif}(\gamma \cdot \mathbb{S}^d)$, and run $\ecolts$ without exploration for $10^3$ steps $100$ times. In each run, we simply record whether (i) global optimism; (ii) local optimism; and (iii) unsaturation held, and estimate their rates simply as the fraction of time over the run that this property was true. We construct these rate estimates for $\gamma \in {[\sqrt{3d}^{-3} ,\sqrt{3d}]},$ specifically evaluating the same for $41$ values of $\gamma$ chosen over an exponential grid (i.e., so that $\log(\gamma)$ has a constant step). Figure~\ref{fig:rcolts_pilot_rates_with_gamma} shows the resulting estimates.

\begin{wrapfigure}[18]{r}{.45\linewidth}
\vspace{-2.2\baselineskip}
    \centering
    \includegraphics[width=.9\linewidth]{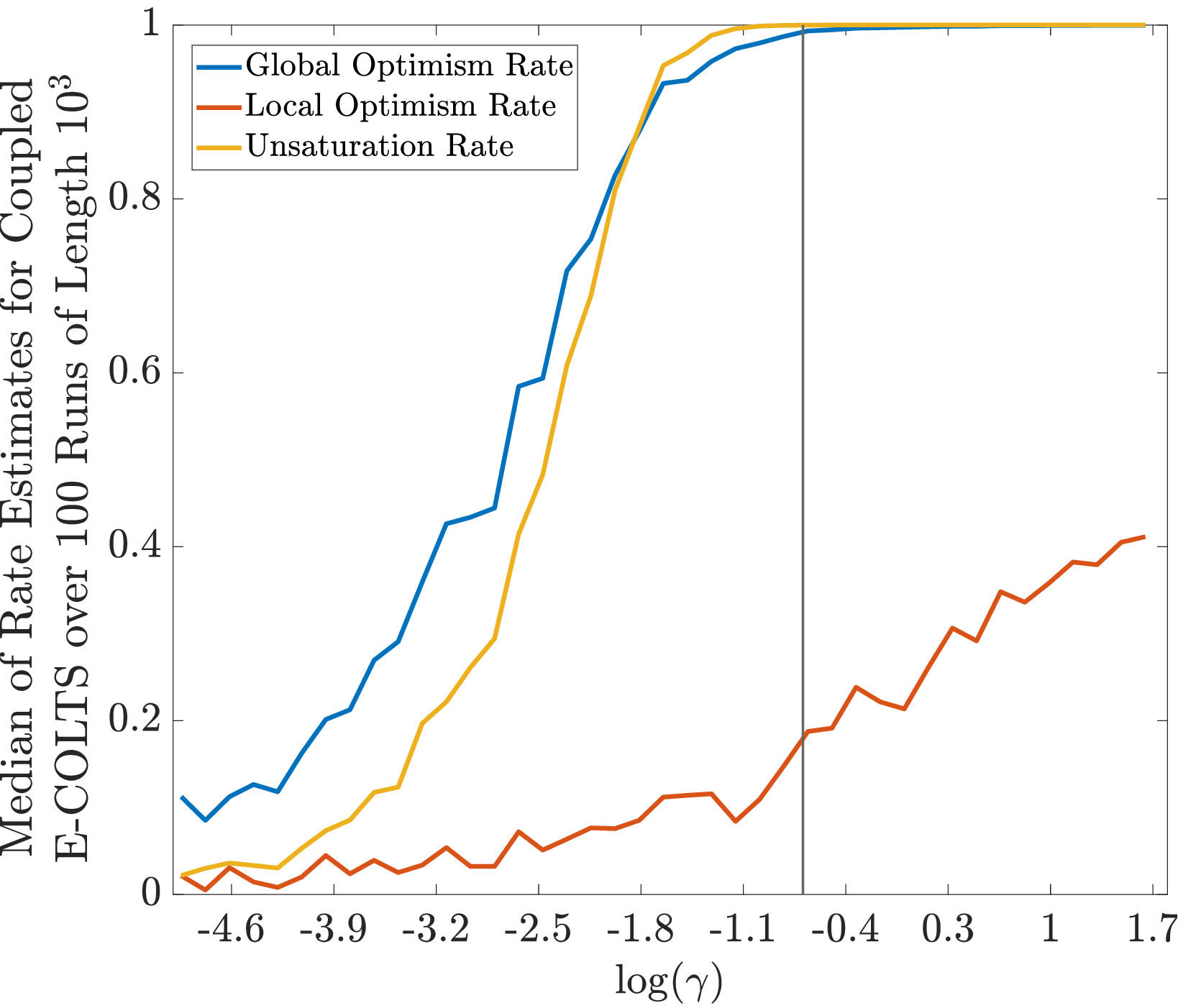}\vspace{-\baselineskip}
    \caption{\footnotesize Behaviour of Global and Local Optimism Rates, and Unsaturation Rate. The black vertical line lies at $\gamma = 0.5,$ the value selected for subsequent experimentation. The largest studied value is at $\sqrt{3d},$ which has logarithm about $1.65$ Observe that the global optimism and unsaturation rates are significant, and in particular $\approx 1$ for $\gamma = 0.5,$ far below $\sqrt{3d} \approx 5.2$.}%
    \label{fig:rcolts_pilot_rates_with_gamma}
\end{wrapfigure}
The main observation is that global optimism and unsaturation rates are already $\approx 1$ for $\log(\gamma) \approx -1$, suggesting good performance with this noise. Note that while such performance with small noise has been previously observed for linear $\ts$ without unknown constraints, we are unaware of an explicit observation of these rates as above. Of course, proving these properties at such small $\gamma$ is an open question, and we also note that our estimates above are not quite correct, since they integrate the events across time, while their rates could vary with $t$. In any case, the main upshot for this is that in our subsequent experiments, \emph{we work with $\gamma = 0.5$ instead of $\sqrt{3d} \approx 5.2$}. 

\textbf{The Behaviour of $\ecolts$ and $\rcolts$.} We now study $\rcolts$ and $\ecolts$ over the long horizon $T = 5\cdot 10^4$. We execute $\rcolts$ with zero resamplings (i.e., $\ecolts$ with no exploration), and then one and finally two resamplings in each round, all driven by the coupled perturbation noise with $\nu_{0.5}$.

\emph{On \textsc{doss}.} We note that \textsc{doss} is not implemented. $\ecolts$ runs in $\sim 10^{-3}$s per round on our machine. (Relaxed)-\textsc{doss} is totally impractical: $(2d)^{m+1} > 10^{12},$ and so it needs $>10^9$s, i.e., years, per round!

\emph{Observations.} Figure~\ref{fig:rcolts_regret_longrun} shows the observed regret and risk traces over $100$ runs. The observed regret behaviour is very strong: even without resampling, the terminal median regret of $\sim 600$ is closer to $\sqrt{T\log T} \approx 750$ than to $\sqrt{d^2 T\log(T)} \approx 6600.$ The risk behaviour is more significant, but still half this scale. The observation of $\eff_T$ suggests that a stronger regret bound may hold for $\ecolts$ and $\rcolts$, which is in line with the stronger instance-specific regret behaviour of the optimism-based method \textsc{doss} \cite{gangrade2024safe}. Proving this is an interesting open problem.

These simulations thus bear out the strong performance of $\ecolts/\rcolts$ with $r = 0$. Further, as we add resampling, risk degrades mildly, but the regret improves significantly, although the returns diminish with more resampling. This suggests that practically, a few resamplings in $\rcolts$ are enough to extract most of the advantage. Interestingly, resampling has a palpable effect even though the optimism rate is nearly one!

\begin{figure}[t]
    \centering
    \includegraphics[width=0.4\linewidth]{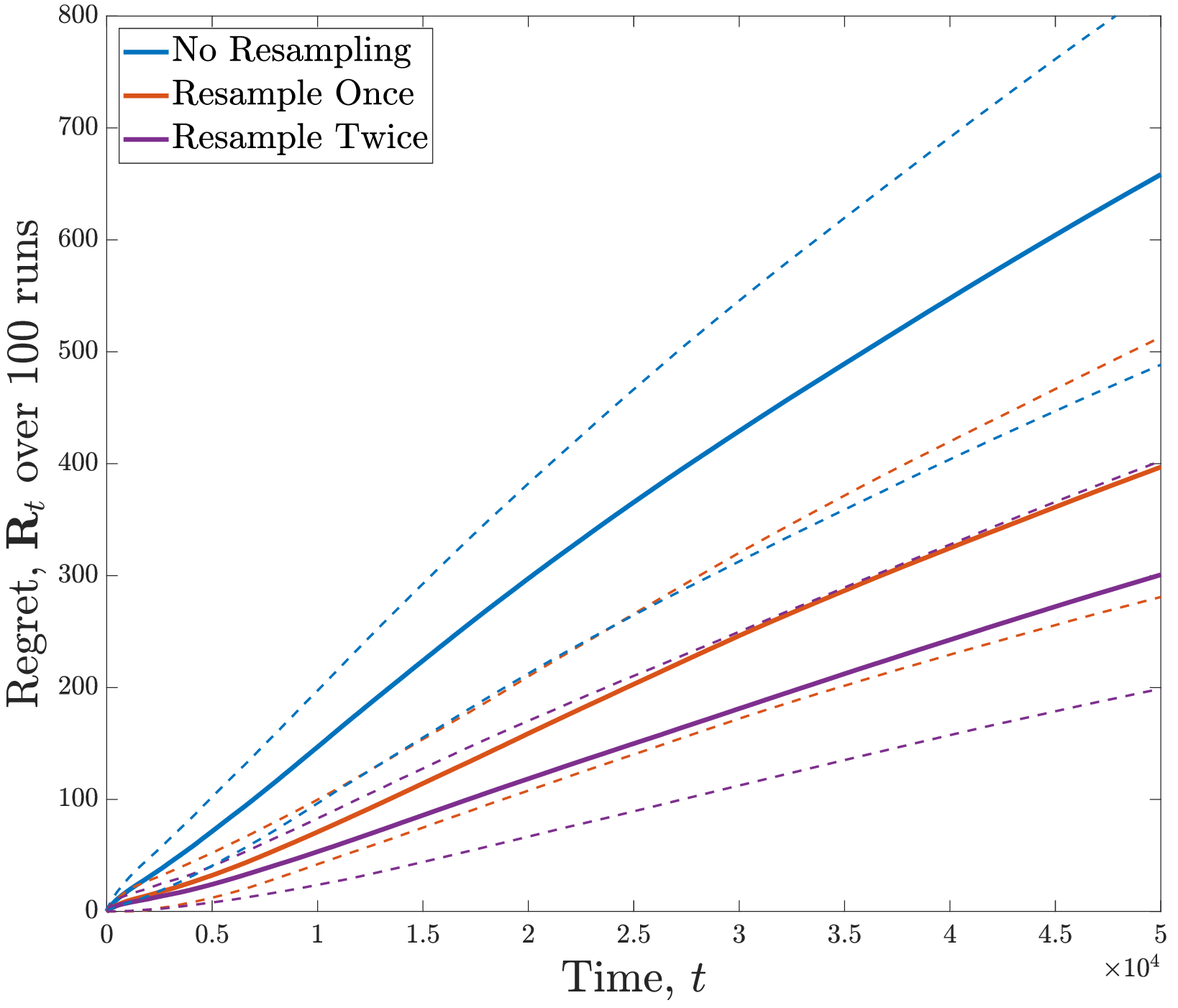}~\hspace{0.05\linewidth}~\includegraphics[width = 0.4\linewidth]{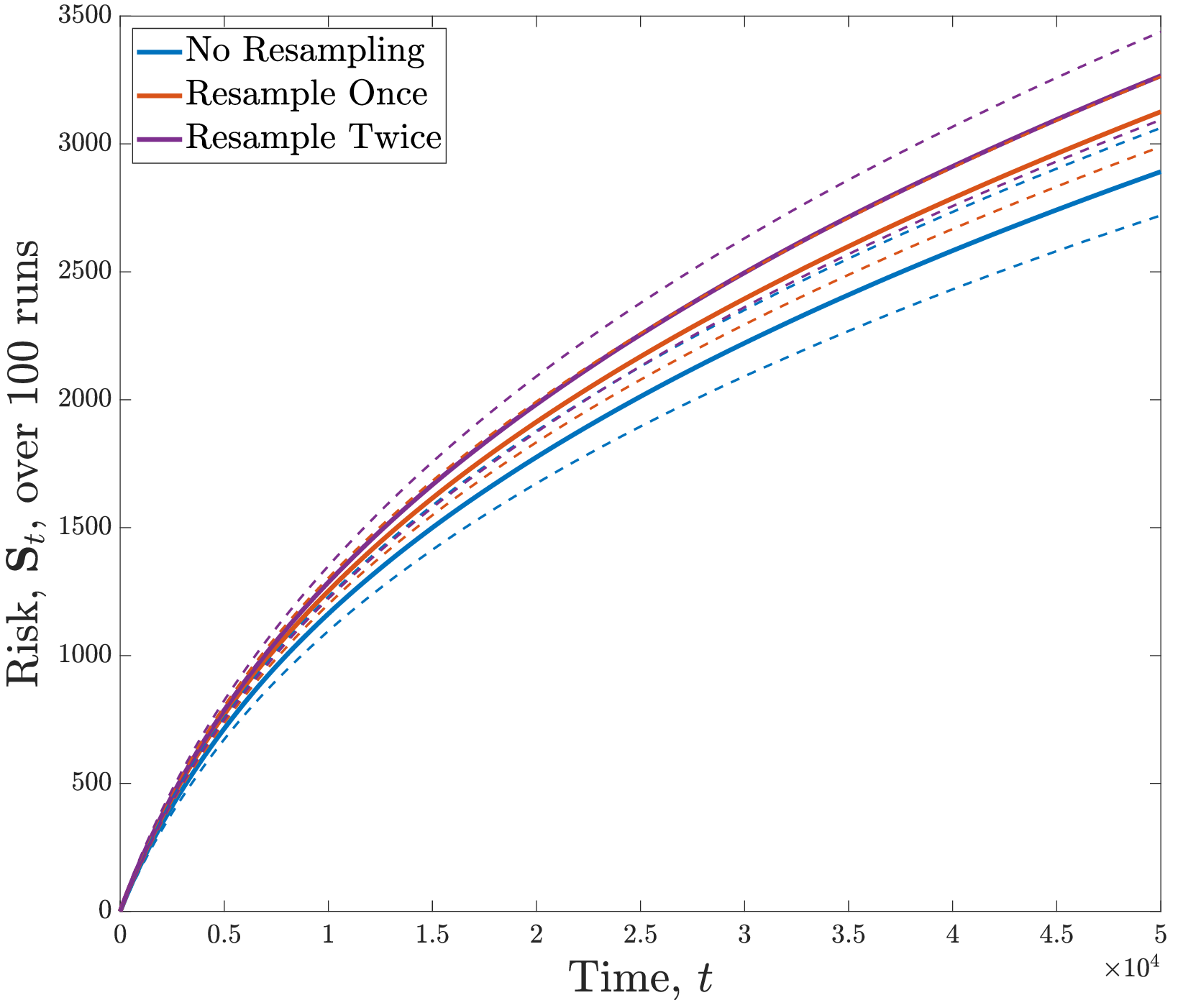}
    \includegraphics[width=0.4\linewidth]{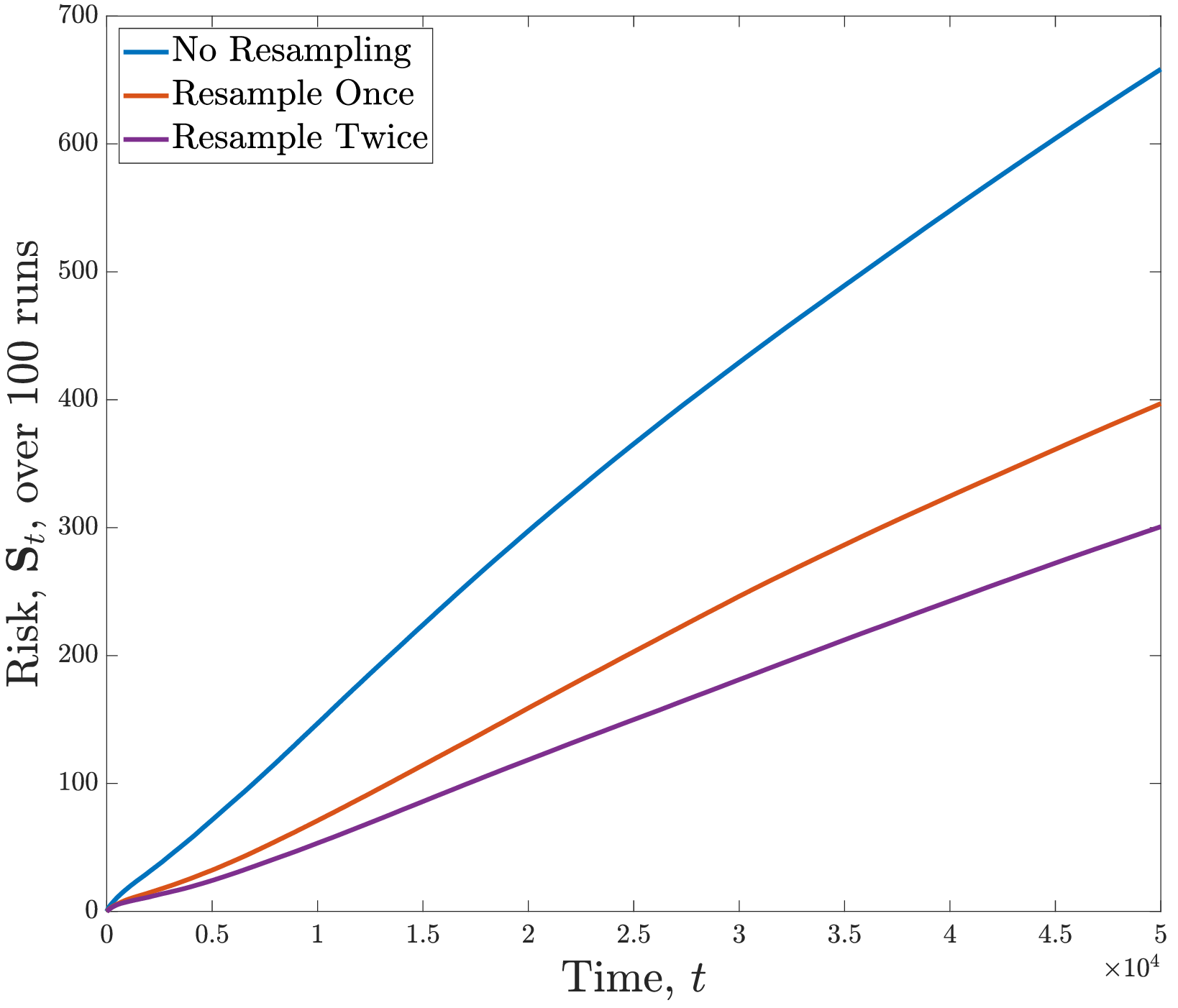}~\hspace{0.05\linewidth}~\includegraphics[width = 0.4\linewidth]{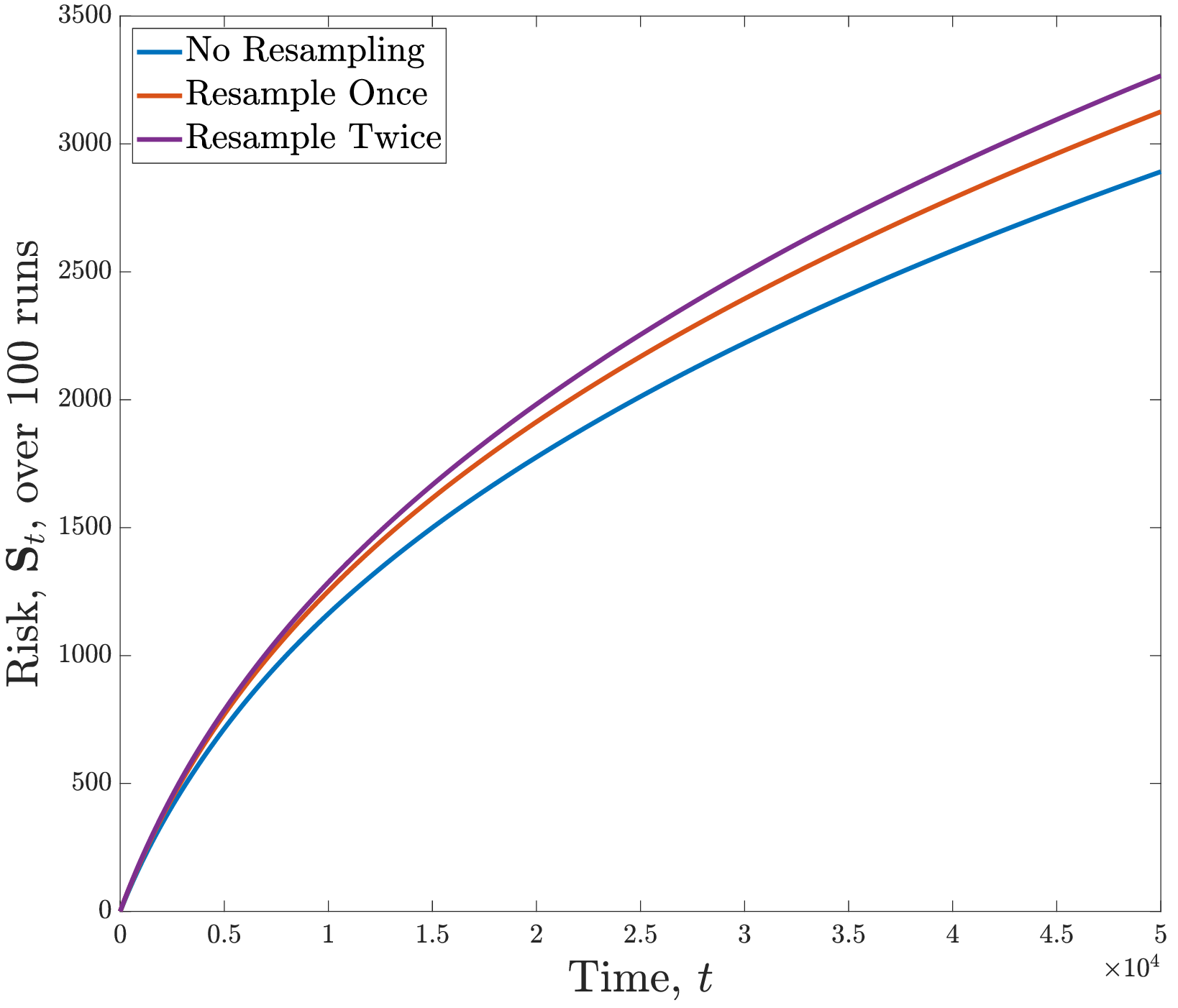}\\\vspace{-.5\baselineskip}
    \caption{\footnotesize Regret (left) and Risk (right) of $\rcolts$ with zero, one, and two resamplings per round. Top includes one-sigma error bars, and for clarity, the bottom figures omit them. Note that the regret behaviour is an order of magnitude smaller than the scale ${\sqrt{d^2 T}} \approx 6600,$ while the risk behaviour is about a factor of half of this. We further observe that resampling improves regret signficantly, while only hurting the risks slightly, although this effect appears to decelerate as resampling is increased.}
    \label{fig:rcolts_regret_longrun}
\end{figure}

\subsection{Hard Constraint Enforcement}

Next, we investigate the behaviour of $\scolts$ over the same instance, supplied with the data $\asafe = 0$. The natural point of comparison to $\scolts$ is the \textsc{safe-LTS} algorithm \cite{moradipari2021safe}, which operates in $O(\socp \log t)$ computation per round.\footnote{We do not implement other prior methods for SLBs, mainly because \textsc{safe-LTS} has previously been seen to have similar behaviour, and be about $2d = 18$ times faster than these methods. Of course, we also did not implement \textsc{doss} as a comparison for the soft constriant enforcement methods since it is impractical to execute for $d = m = 9$.}

\begin{figure}[t]
    \centering
    \includegraphics[width=0.45\linewidth]{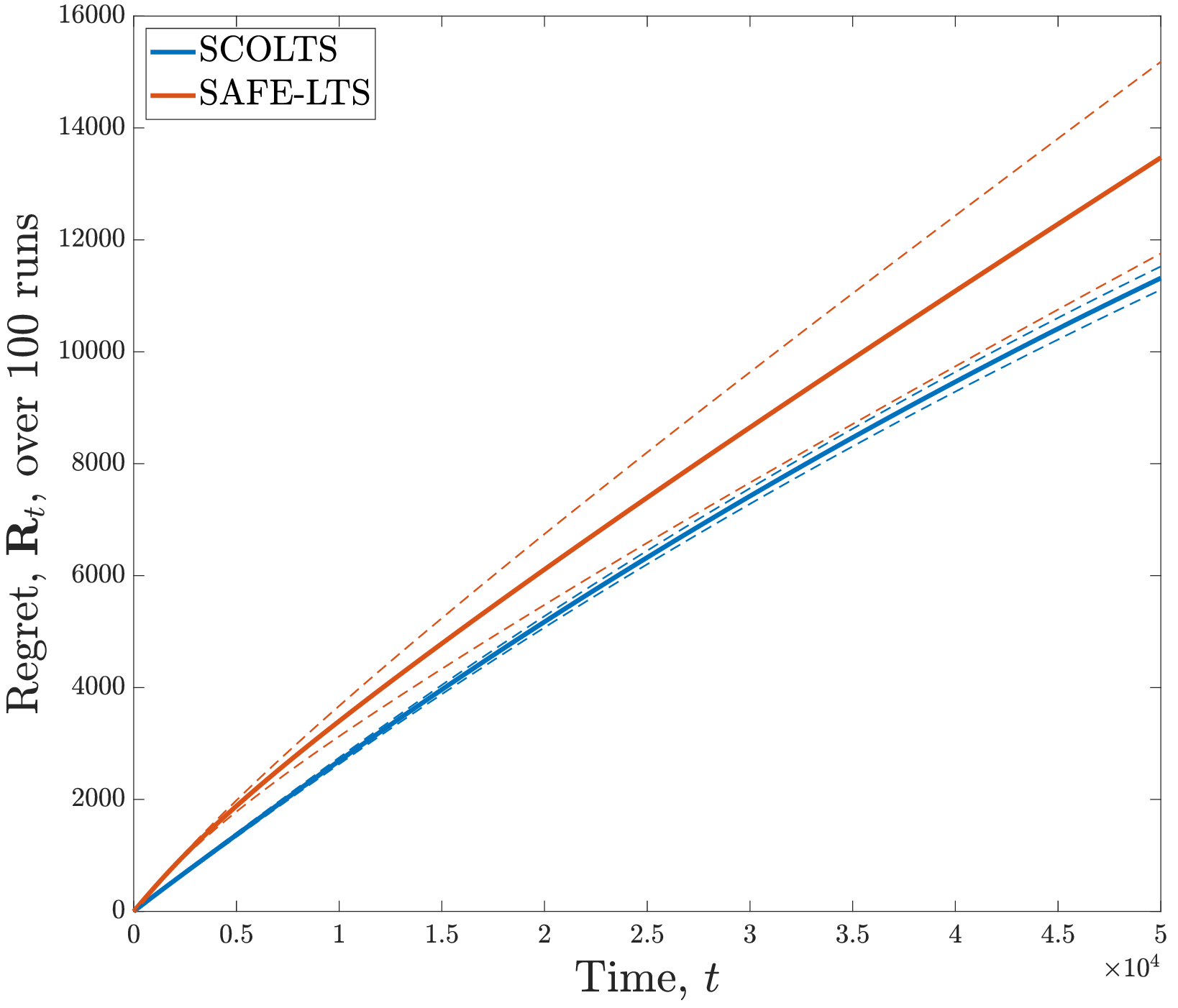}
    \caption{\footnotesize Regret Behaviour of $\scolts$ and \textsc{safe-LTS} on the same instance as previous figures (one-sigma error curves). We note that $\scolts$ offers a mild improvement in regret over \textsc{safe-LTS}. However, this comes with a $5\times$ reduction in net computational time per round, which is the main advantage of $\scolts$.}\vspace{-.5\baselineskip}
    \label{fig:scolts_regret}
\end{figure}

Concretely, we again drive this method with $\nu_{0.5}$ as before. For $\textsc{safe-LTS},$ we sample a perturbed objective vector with the same noise scale, and otherwise optimise over the second order conic constraints as detailed in \S\ref{sec:scolts_bounds}. In both cases, we used the library methods \texttt{linprog} and \texttt{coneprog} provided by MATLAB to implement these methods.\footnote{Of course, $\rcolts/\ecolts$ were also implemented using \texttt{linprog}.} Note that these methods are specifically tailored to linear and conic programming respectively. As before, we repeat runs of length $T = 5 \cdot 10^4$ for a total of $100$ runs.

\emph{Strong Safety Behaviour.} We note that in all of our runs, we did not observe any constraint violation from either $\scolts$ or $\textsc{safe-LTS}$, despite the fact that we executed these methods with $\delta = 0.1$. This suggests both that in practice, the parameter $\delta$ can be relaxed (which would yield mild improvements in regret), and in any case verifies the strong safety properties of these methods.

\emph{Comparison of Regret.} We show the regret traces over the $100$ runs in Figure~\ref{fig:scolts_regret}. We observe that $\scolts$ has a slightly improved regret performance relative to \textsc{safe-LTS}, which may be attributed to the selection of stronger exploratory directions through solving the perturbed program.

\emph{Computational Speedup.} In wall-clock terms, each iteration of \textsc{safe-LTS} is about $5.2 \times$ slower than that of $\scolts$ on this $9$ dimensional instance with $9$ unknown constraints (over $5 \cdot 10^6$ total iterations, $\scolts$ took about $0.22$ms per iteration, while \textsc{safe-lts} took about $1.16$ms), a significant computational advantage even in this modest parameter setup. %

\emph{High Level Conclusions.} The main takeaway from this set of experiments is that $\scolts$ offers tangible benefits in computational time relative to \textsc{safe-LTS} (and a fortiori, to other pessimism-optimism based frequentist hard constraint enforcement methods), while even obtaining a slight improvement in the regret behaviour. This demonstrates the utility of $\scolts$ over these prior methodologies, and suggests that it is the natural approach that should be used in practice.

\subsubsection{Investingating Behaviour with Increasing \texorpdfstring{$m$}{Number of Constraints}}\label{appx:scolts_varying_m}

\begin{figure}[tb]
    \centering
        \includegraphics[width = 0.4\linewidth]{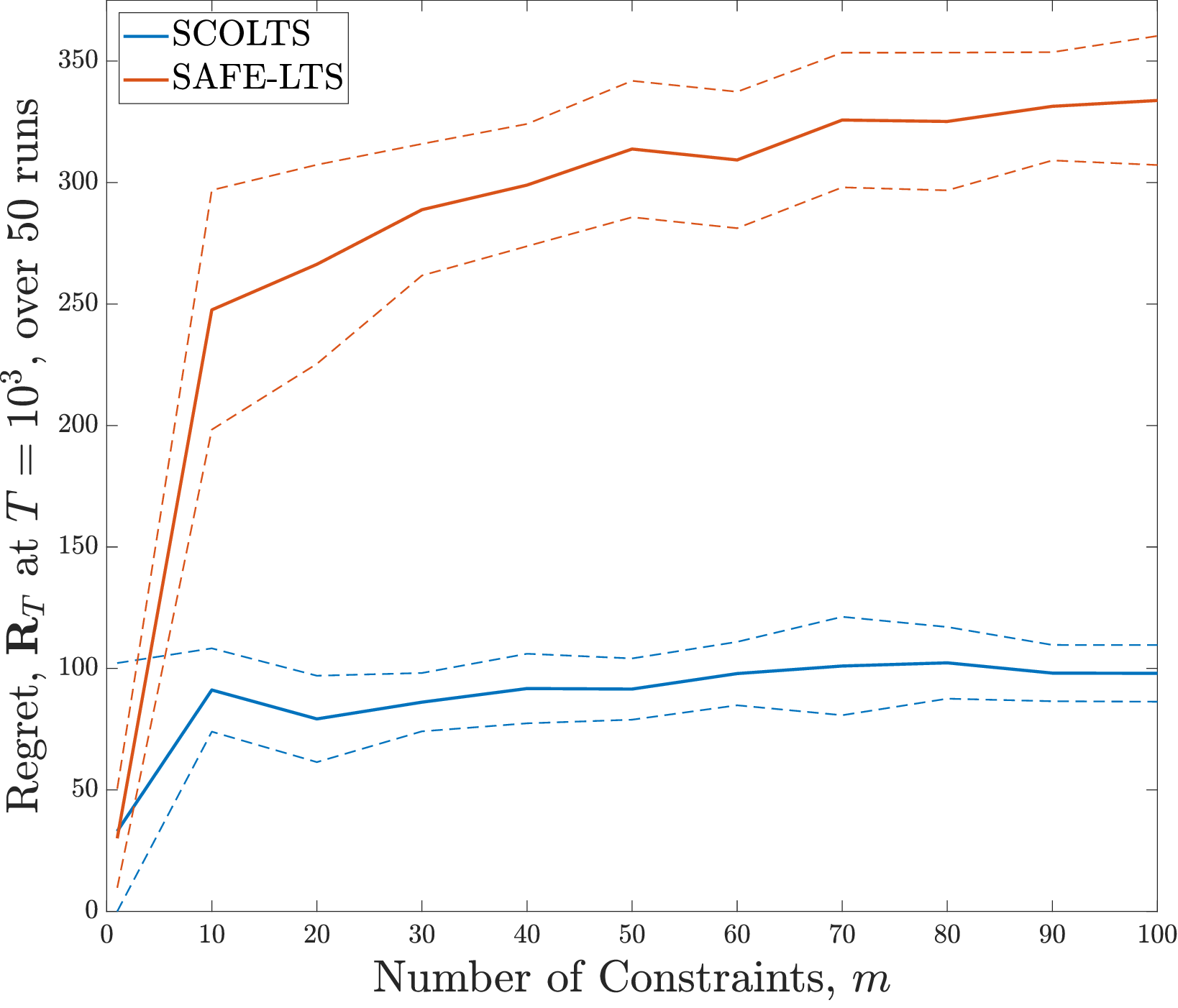}~\hspace{0.05\linewidth}\includegraphics[width = 0.4\linewidth]{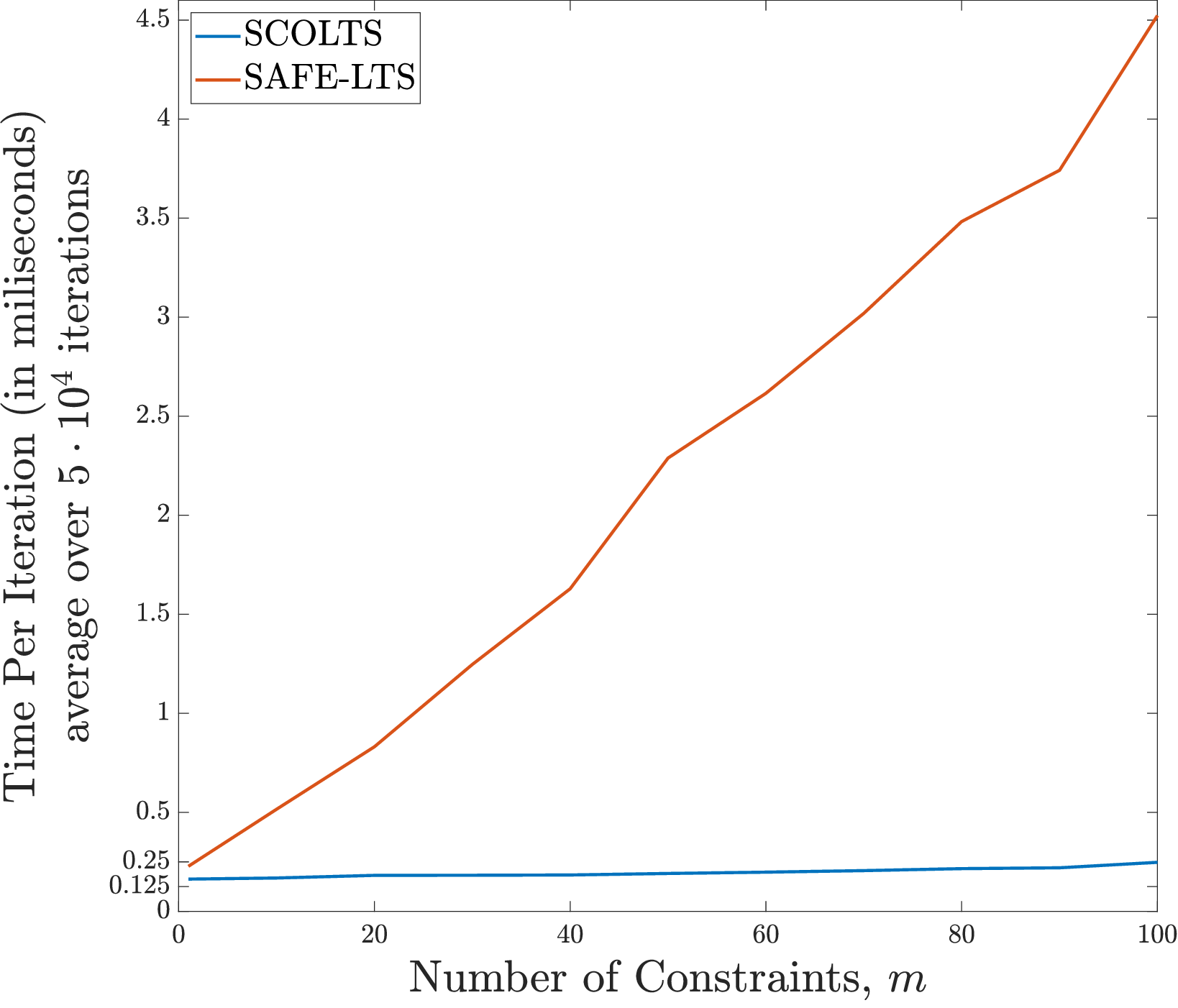}\vspace{-.5\baselineskip}
        \caption{\footnotesize Comparisons of the regret (left, one-sigma error curves) and computational costs (right) of $\scolts$ and \textsc{safe-lts} in the $d = 2$ instance as $m$ varies. This is the same setting as Figure~\ref{fig:scolts_main}, right, but presented separately rather than as a ratio. The left plots the regrets at time $T = 10^3$ over $50$, and the right plots the wall-clock time per iteration on our resources in milliseconds. $\scolts$ needs $0.14-0.25$ milliseconds per iteration, while $\textsc{safe-lts}$ needs $>4.5$ at $m = 100$. At the same time, for $m \ge 10,$ the regret of $\scolts$ is about $3\times$ smaller.}\vspace{-\baselineskip} \label{fig:scolts_vary_m_separate_plots}
\end{figure}

Of course, the computational problem of optimising $m$ SOC constraints becomes harder as $m$ grows, and so we expect that the computational advantage of $\scolts$ over \textsc{safe-lts} would grow with $m$.\footnote{Note that it may be possible to mitigate this somewhat by instead imposing the convex constraint $\max_i (\hat\Phi_t a - \alpha)^i + \|a\|_{V_t}^{-1} \le 0$ to exploit that the same matrix $V_t^{-1}$ appears in all constraints. However, the gradient computation of this map still grows with $m$, so the overall picture is unclear. Of course, imposing only $m$ linear constraints is bound to be faster.} To investigate this hypothesis more closely, we turn to a slightly different setup.\vspace{0.3\baselineskip}\\
\emph{Setup.}  We set $d = 2, \theta_* = (1,0),\mathcal{A}= [-1/\sqrt{d}, 1/\sqrt{d}]^d$. For $m \ge 3,$ we impose $m$ unknown constraints such that the feasible region forms a regular $m$-gon with one vertex at $(0.2/\sqrt{2},0)$. This allows us to systematically increase $m$ (to very high values) without incurring significant computational costs. We investigate the behaviour of $\scolts$ and $\ecolts$ on this setup with the coupled noise design as in the previous section ($\gamma = 0.5$) for $m \in \{10, 20, \cdots, 100\}$. We also execute this for $m = 1,$ where a single constraint passing through the same vertex is enforced. In all cases, we set $\asafe = 0$, which is always feasible.\vspace{0.3\baselineskip}\\
\emph{Strong Computatational Speedup.} As seen in Figure~\ref{fig:scolts_vary_m_separate_plots}, $\scolts$ has a strong computational advatange, which further grows with $m$. In particular, at $m = 1,$ $\scolts$ is about $1.3\times$ faster to execute than \textsc{safe-lts}, while for $m = 100,$ this advantage grows to $18\times$. \vspace{0.3\baselineskip}\\
\emph{Improved Regret Performance.}\footnote{Note: for the regret ratio in Figure~\ref{fig:scolts_main}, we perform $100$ separate runs with both methods, and compute the ratio of regret for the two methods in each. That figure reports the mean over this data - in this case, the expected mean is $\sim 1.5$ at $m = 1$, but with wide confidence intervals (CIs). For $m\ge 10,$ the lower confidence bounds all exceed $2$. At $m = 1,$ the mean regret of \textsc{safe-lts} is about $0.91\times$ that of $\scolts$, with strongly overlapping CIs.}   Further, instead of the small gain seen in the previous section, in this problem $\scolts$ has a strong statistical advantage relative to \textsc{safe-lts} for even moderate $m$. Indeed, while at $m = 1,$ its regret is about $10\%$ larger than that of \textsc{safe-lts}, for larger $m$, its regret is many times \emph{smaller}. In particular, for $m \ge 10,$ we found that the regret of $\scolts$ is roughly $3\times$ smaller (ranging between $2.7\times$ and $3.4\times$.).\vspace{0.3\baselineskip}\\
\emph{Takeaways.} This investigation further bolsters the strong advantage of $\scolts$ over \textsc{safe-lts}. Note that alternative confidence-set based hard enforcement methods are at least $2d$ times slowed than \textsc{safe-lts}, meaning that the computational advantage of $\scolts$ is even stronger relative to these methods. For large $m$, this appears to be accompanied by a large statistical advantage, making this the natural method in applications of SLBs.

\subsection{Simulation Study on the Behaviour of the Decoupled Noise}\label{appx:simulations_decoupled}

Finally, we investigate the behaviour of the $\colts$ framework under the decoupled noise design, wherein, instead of setting $H = -\mathbf{1}_m \eta,$ we draw $\eta$, and each row of $H,$ independently from $\nu_{\gamma}$. The main impetus behind this, of course, is that this decoupled design is a natural choice to execute $\colts$, although it is contraindicated by the analysis tools available to us.

\emph{Behaviour of Event Rates with $\gamma$.} To begin with, Figure~\ref{fig:decoupled_rates} shows the global optimism, local optimism, and unsaturation rates with this decoupled noise for the same instance as previously studied. Observe first that the decoupled noise design does experience a slight decrease in each of these rates compared to those seen in Figure~\ref{fig:rcolts_pilot_rates_with_gamma}. However, this effect is relatively mild, and in particular, we can see that the unsaturation rate is already up to nearly one at our previously selected value of $\gamma = 0.5.$ This suggests that the decoupled noise would do nearly as well as the coupled noise in this case.

\begin{figure}[htb]
    \centering
    \includegraphics[width=0.4\linewidth]{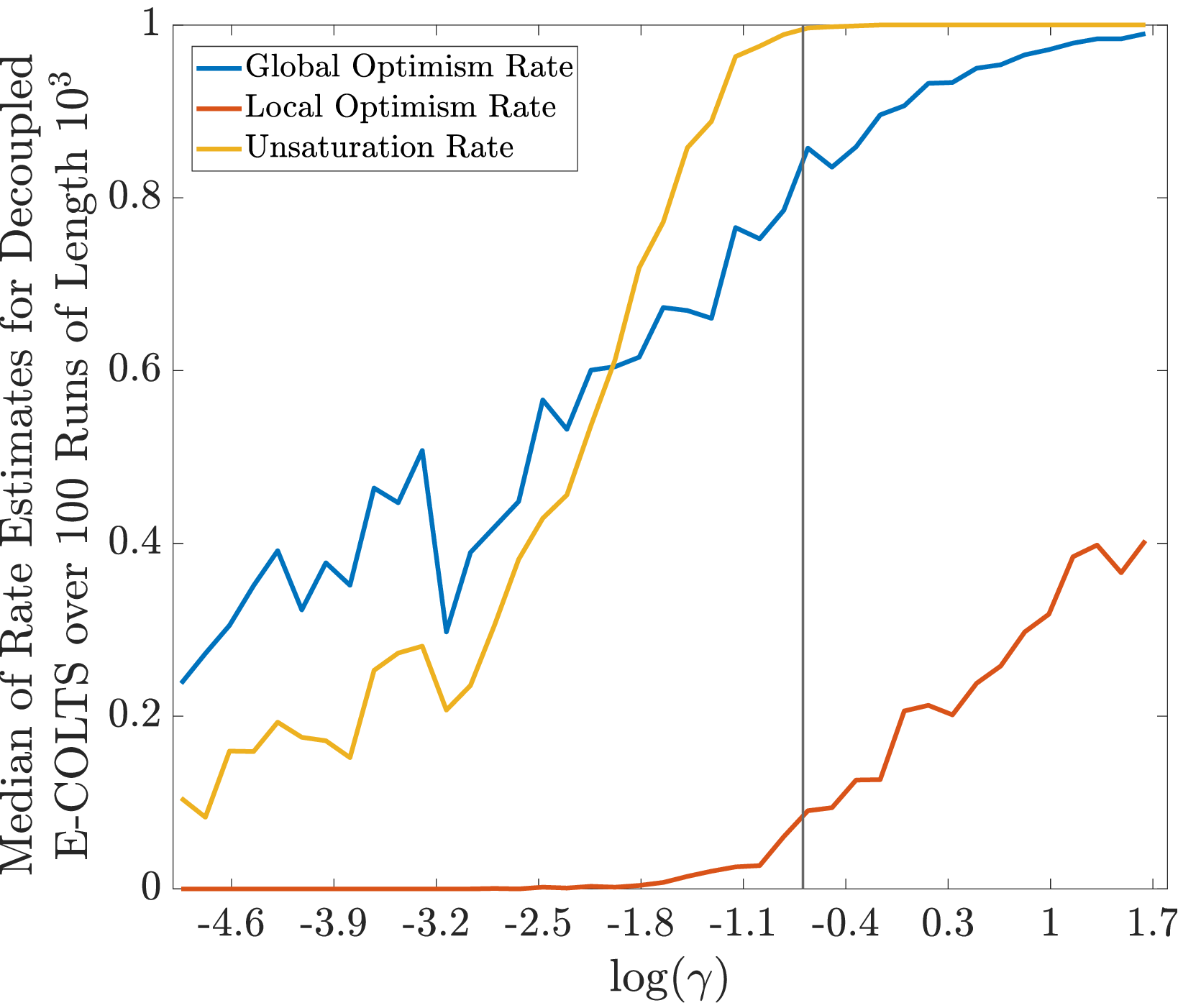}\vspace{-.5\baselineskip}
    \caption{\footnotesize Behaviour of the Global Optimism, Local Optimism, and Unsaturation Rates with $\gamma$ for the Decoupled Noise in the setting of Figure~\ref{fig:rcolts_pilot_rates_with_gamma}. Observe that while these rates decay somewhat with respect to the coupled noise, they are still strong, and especially for large $\gamma$ are nearly as good as with the coupled noise.}\vspace{-.5\baselineskip}
    \label{fig:decoupled_rates}
\end{figure}

\emph{Behaviour of Regret and Risk.} To further investigate the above claim, we execute $\ecolts$ without exploration (or equivalently, $\rcolts$ with $r = 0$) driven with this decoupled noise over the longer horizon $T = 5\cdot 10^4$. The resulting regret and risks are plotted in Figure~\ref{fig:decoupled_reg_and_risk}, along with the same for $\ecolts$ with coupled noise. Observe that the decoupled noise sees a significant loss of about $3\times $ in regret, but sees a gain of about $1.5\times $ in risk. Heuristically, we may think of the decoupled noise as behaving as if the noise is coupled but "shrunk", so that the behaviour of the risk is improved, but the behaviour of the regret worsens. 

Practically speaking, our recommendation remains to use the coupled noise design, in that it attains higher rates of explanatory events, and carries theoretical guarantees. Nevertheless, establishing that $\eff_T$ and $\saf_T$ do scale sublinearly with the decoupled noise design, as is evident from Figure~\ref{fig:decoupled_reg_and_risk}, is an interesting open problem.

\begin{figure}[t]
    \centering
    \includegraphics[width=0.4\linewidth]{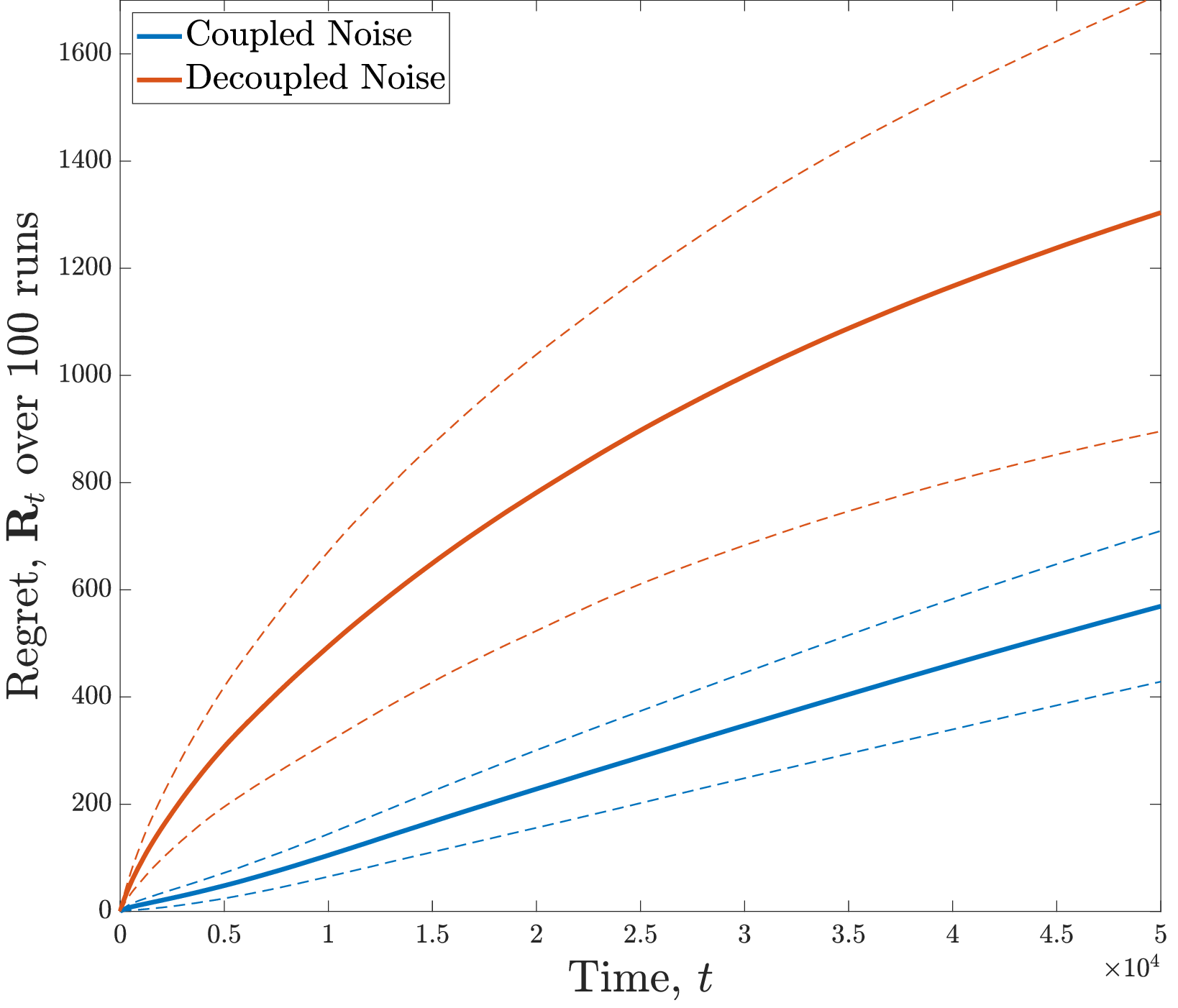}~\hspace{0.04\linewidth}~\includegraphics[width = 0.4\linewidth]{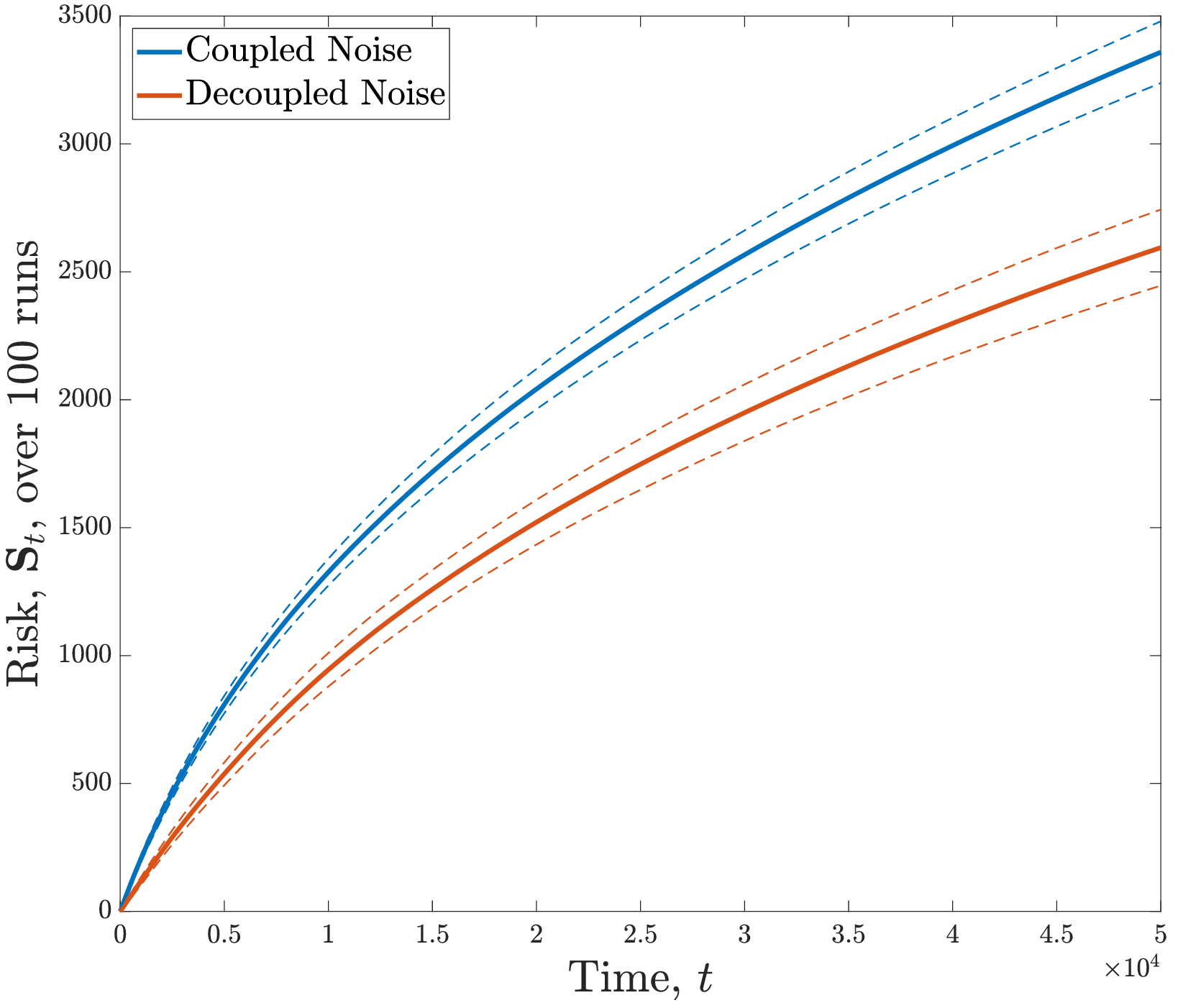}\vspace{-.5\baselineskip}
    \caption{\footnotesize Behaviour of regret (left) and risk (right) for $\ecolts$ executed with the decoupled noise compared with $\ecolts$ executed with the coupled noise (one-sigma error bars). Observe that the regret behaviour sharply deteriorates, while the risk behaviour slightly improves for the decoupled noise design. Heuristically, this suggests that the decoupled noise behaves `like' the coupled noise, but with a smaller value of $\gamma$.}%
    \label{fig:decoupled_reg_and_risk}
\end{figure}

\subsubsection{Investigation of Rates with Increasing \texorpdfstring{$m$}{Number of Constraints}}

Of course, the main obstruction with the use of the decoupled noise in \S\ref{sec:coupled_noise_design} was to do with many constraints. Indeed, it should be clear that under this decoupled noise, the local optimism rate must decay exponentially with $m$, since if any row of $\tPhi_t$ is perturbed so that $a_*$ violates its constraints, local optimism would fail (and this would occcur with a constant chance, no matter the estimates). 

To probe whether this indeed occurs, we simulate the behaviour of $\ecolts$ with the coupled and decoupled noise designs on a simplified setup. 

\emph{Setup.} We again take the $d = 2$ polygonal constraints investigated in \S\ref{appx:scolts_varying_m}. We investigate the behaviour of $\ecolts$ with both the coupled and decoupled noise designs on this instance as $m \in \{10, 20, \dots, 100\} \cup \{200, 300, \dots, 1000\},$ thus letting us probe an extremely high number of unknown constraints.

\emph{Observations.} There are two main observations of Figure~\ref{fig:rates_with_m}. Firstly, note that as shown in the main text, the rates of optimism and unsaturation under the coupled noise design are stable, and do not meaningfully vary with $m$ after it has grown at least slightly. 

On the other hand, under the decoupled noise design, the local optimism rate clearly crashes exponentially. The unsaturation rate has a slower but evident decay: roughly, this is as $m^{-1.3}$ for $m \le 100,$ and appears to be exponential for large $m$. However, surprisingly, the \emph{global optimism} rate remains stable (although lower than the same with the coupled design). This shows that there are situations with low-regret where frequent global optimism would be the `correct' explanation for good performance of methods like $\scolts$ or $\ecolts$ (indeed, this is what prompted us to write the optimism based analysis of these methods in \S\ref{appx:optimism_for_scolts}). Note however that \emph{proving} that global optimism is frequent under the decoupled design is an open problem. In fact, with unknown constraints, we do not know of any method to deal with global optimism lower bounds that does not pass through local optimism, since the approach of Abeille \& Lazaric \cite{abeille_lazaric} relies on convexity properties of the value function in terms of the unknown parameters, which fails in this case. 

\begin{figure}[t]
    \centering
    \includegraphics[width=0.45\linewidth]{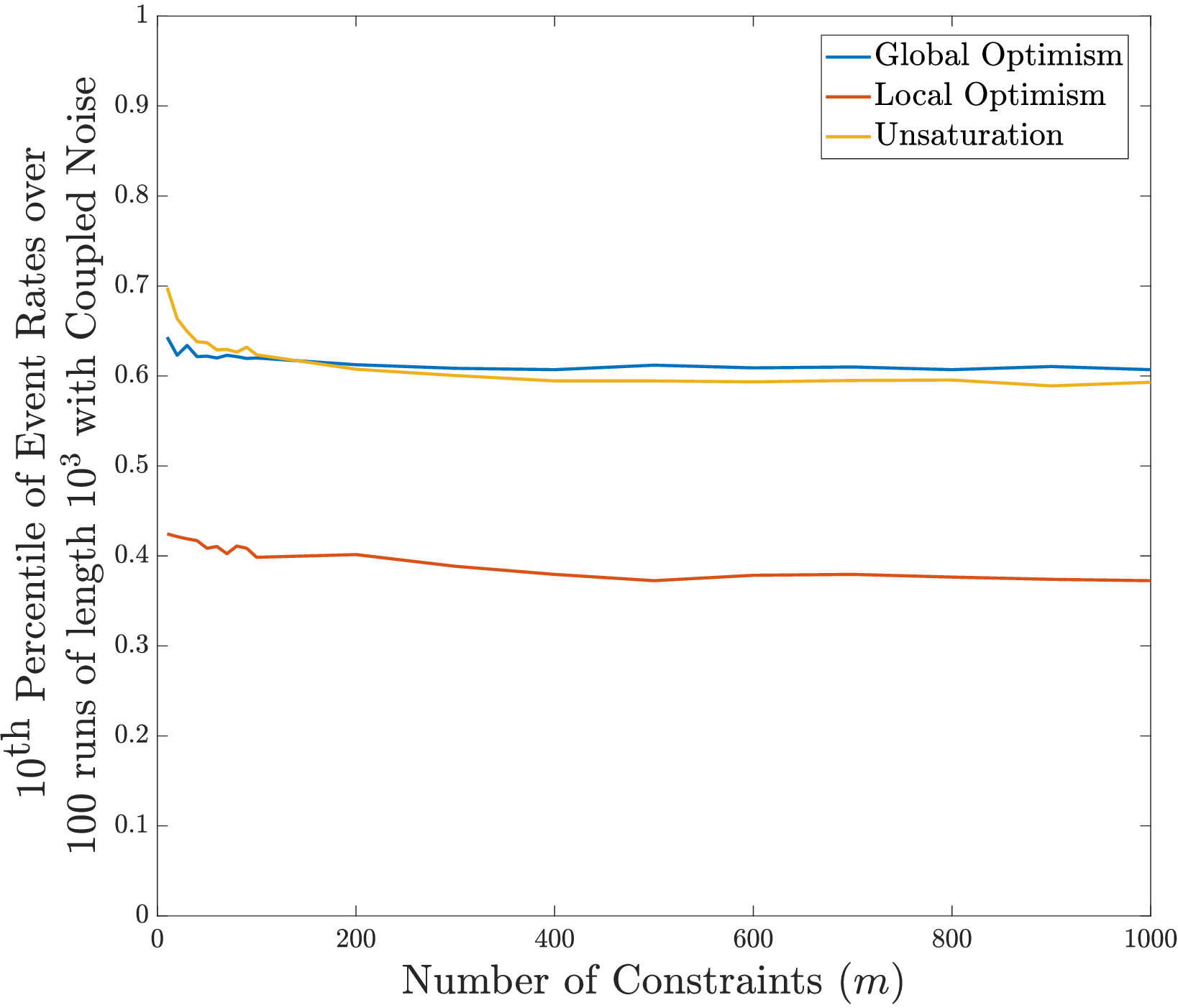}~\includegraphics[width = 0.45\linewidth]{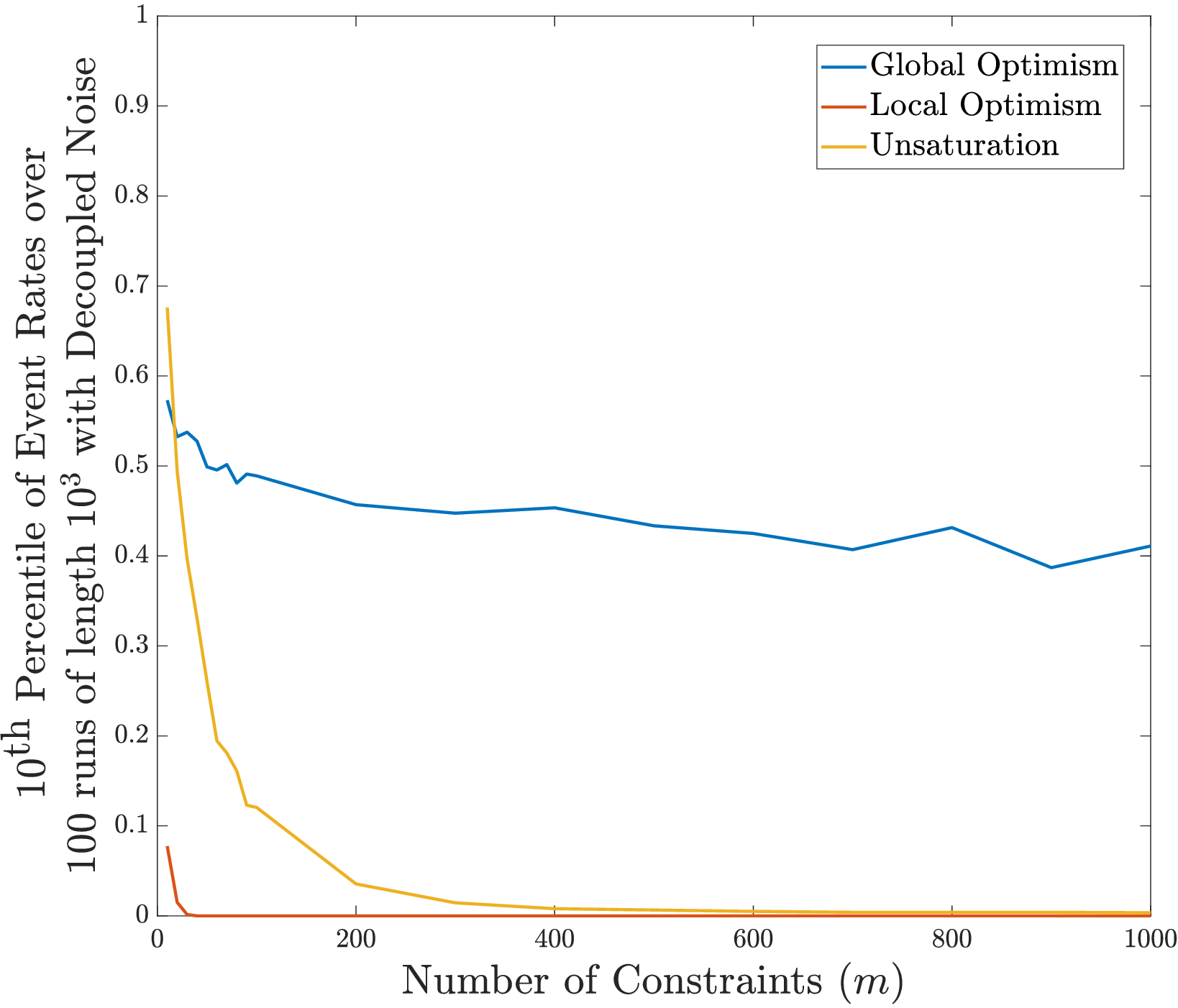}\vspace{-.5\baselineskip}
    \caption{\footnotesize Behaviour of the rates of global and local optimism, and of unsaturation, in the polygonal instances as the number of constraints is increased for the coupled (left) and decoupled (right) noise designs driven by $\mathrm{Unif}(\mathbb{S}^2)$. Observed that the behaviour of these is stable with $m$ for the coupled design, but for the decoupled design, the local optimism and unsaturation rate decay with $m$. Surprisingly, the global optimism rate remains stable even for the decoupled noise design.}
    \label{fig:rates_with_m}
\end{figure}

\end{document}